%% file: main.tex
\documentclass[draftclsnofoot, onecolumn]{IEEEtran}

\usepackage{algorithm,algorithmicx,algpseudocode} 
\usepackage{amsmath,amsfonts,amssymb,amsthm} 
\usepackage{pifont,stmaryrd} 
\usepackage{graphicx,epstopdf,caption,subcaption} 
\usepackage{threeparttable,multirow} 
\usepackage[colorlinks]{hyperref} 
\usepackage[normalem]{ulem} 
\usepackage[dvipsnames]{xcolor} 
\usepackage{cite} 
\usepackage{blindtext} 
\usepackage{wrapfig}
\IEEEoverridecommandlockouts

\input{_defs}

\theoremstyle{plain}
\newtheorem{lemma}{Lemma}
\newtheorem{theorem}{Theorem}

\theoremstyle{definition}

\theoremstyle{remark}

\makeatletter
\newcommand*{\rom}[1]{\expandafter\@slowromancap\romannumeral #1@}
\makeatother

\graphicspath{{./}} 

\begin{document}

\title{A Linearly Convergent Algorithm for Distributed Principal Component Analysis}

\author{Arpita Gang and Waheed U. Bajwa
%
\thanks{Preliminary versions of some of the results reported in this paper were presented at the 2019 IEEE International Conference on Acoustics, Speech and Signal Processing (ICASSP), Brighton, United Kingdom, May 12-17 2019,~\cite{gang.raja.bajwa.2019}. Arpita Gang and WUB are with the Department of Electrical and Computer Engineering, Rutgers University--New Brunswick, NJ 08854 (Emails: {\tt \{arpita.gang,~waheed.bajwa\}@rutgers.edu}).}%
%
\thanks{This work was supported in part by the National Science Foundation under Awards CCF-1453073, CCF-1907658, and OAC-1940074, by the Army Research Office under Awards W911NF-17-1-0546 and W911NF-21-1-0301, and by the DARPA Lagrange Program under ONR/NIWC Contract N660011824020.}}

\maketitle
\begin{abstract}
Principal Component Analysis (PCA) is the workhorse tool for dimensionality reduction in this era of big data. While often overlooked, the purpose of PCA is not only to reduce data dimensionality, but also to yield features that are uncorrelated. Furthermore, the ever-increasing volume of data in the modern world often requires storage of data samples across multiple machines, which precludes the use of centralized PCA algorithms. This paper focuses on the dual objective of PCA, namely, dimensionality reduction and decorrelation of features, but in a distributed setting. This requires estimating the eigenvectors of the data covariance matrix, as opposed to only estimating the subspace spanned by the eigenvectors, when data is distributed across a network of machines. Although a few distributed solutions to the PCA problem have been proposed recently, convergence guarantees and/or communications overhead of these solutions remain a concern. With an eye towards communications efficiency, this paper introduces a feedforward neural network-based one time-scale distributed PCA algorithm termed Distributed Sanger's Algorithm (DSA) that estimates the eigenvectors of the data covariance matrix when data is distributed across an undirected and arbitrarily connected network of machines. Furthermore, the proposed algorithm is shown to converge linearly to a neighborhood of the true solution. Numerical results are also provided to demonstrate the efficacy of the proposed solution.
\end{abstract}


\input{Introduction}

\input{ProblemDescription}

\input{ProposedAlgorithm}

\input{ConvergenceAnalysisDSA}
\input{ExperimentalResults}

\section{Conclusion}\label{sec:conc}
In this paper, we proposed and analyzed a new distributed Principal Component Analysis (PCA) algorithm that, as opposed to distributed subspace learning methods, facilitates both dimensionality reduction and data decorrelation in a distributed setup. Our main contribution in this regard was a detailed convergence analysis to prove that the proposed distributed method linearly converges to a neighborhood of the eigenvectors of the global covariance matrix. We also provided numerical results to demonstrate the communications efficiency and overall effectiveness of the proposed algorithm. 

In terms of future work, an obvious extension would be a distributed algorithm that enables \textit{exact} convergence to the PCA solution at a linear rate. 
Note that the use of a diminishing step size $\alpha$ along with the analysis in this paper already guarantees that DSA can converge exactly to the PCA solution. However, this exact convergence guarantee comes at the expense of a slow convergence rate. We instead expect to combine ideas from this work as well as ideas such as \emph{gradient tracking} from the literature on distributed optimization~\cite{extra, next, qu.li.2018} to develop a linearly convergent, exact algorithm for distributed PCA in the future. Another possible future direction involves developing an algorithm for distributed PCA that does not require the top $K$ eigenvalues to be distinct. We also leave the case of multiple eigenvector estimation from distributed, streaming data, as in \cite{raja.bajwa.2020}, for future work.
\begin{appendices} 
\section{Statements and Proofs of Auxiliary Lemmas}\label{app:aux_lemma}
\subsection{Statement and Proof of Lemma~\ref{lemma:bounded_iterate_rayleigh}}\label{app:lemma1}
\begin{lemma}\label{lemma:bounded_iterate_rayleigh}
   	Assume $\|\bx_{k}^{(0)}\| = 1,~\forall k$. If the step size is bounded above as $\alpha \leq \frac{1}{3\lambda_1{(2K-1)}}$, where $\lambda_1$ is the largest eigenvalue of $\bC$ and $K$ is the number of eigenvectors to be estimated, then
   \begin{equation}
   	\forall t, \quad \|\bx_{k}^{(t)}\| < \sqrt{3} \quad \text{and} \quad (\bx_k^{(t)})^T\bC\bx_k^{(t)} < \frac{1}{\alpha}.
   \end{equation}
\end{lemma}
\begin{proof}
From \eqref{eq:centralized_sangerk}, we know the iterate for $k^{th}$ eigenvector estimate is
\begin{eqnarray*}
	\bx_{k}^{(t+1)} &=& \bx_{k}^{(t)} + \alpha \bigg(\bC\bx_{k}^{(t)} - (\bx_{k}^{(t)})^T\bC\bx_{k}^t\bx_{k}^{(t)} - \sum_{p=1}^{k-1}\bq_{p}\bq_{p}^T\bC\bx_{k}^{(t)}\bigg)\\
	&=& \bx_{k}^{(t)} + \alpha \big(\bC\bx_{k}^{(t)} - (\bx_{k}^{(t)})^T\bC\bx_{k}^t\bx_{k}^{(t)} - \sum_{p=1}^{k-1}\lambda_p\bq_{p}\bq_{p}^T\bx_{k}^{(t)}\big)\\
	&=& \bx_{k}^{(t)} + \alpha \big(\tilde{\bC}_k\bx_{k}^{(t)} - (\bx_{k}^{(t)})^T\bC\bx_{k}^{(t)}\bx_{k}^{(t)} \big),
\end{eqnarray*}
where $\tilde{\bC}_k = \bC - \sum_{p=1}^{k-1}\lambda_p\bq_{p}\bq_{p}^T$. Notice that $\tilde{\bC}_k^2 = \bC^2 - \sum_{p=1}^{k-1}\lambda_p^2\bq_p\bq_p^T$.
Hence,
\begin{align}\nonumber
	\|\bx_k^{(t+1)}\|^2	&= \|\bx_{k}^{(t)} + \alpha \big( \tilde{\bC}_k\bx_{k}^{(t)} - (\bx_{k}^{(t)})^T\bC\bx_{k}^{(t)}\bx_{k}^{(t)} \big)\|^2\\ \nonumber
	&= \|\bx_{k}^{(t)}\|^2 + \alpha^2 \| \tilde{\bC}_k\bx_{k}^{(t)} - (\bx_{k}^{(t)})^T\bC\bx_{k}^{(t)}\bx_{k}^{(t)}\|^2 + 2\alpha(\bx_{k}^{(t)})^T( \tilde{\bC}_k\bx_{k}^{(t)} - (\bx_{k}^{(t)})^T\bC\bx_{k}^{(t)}\bx_{k}^{(t)})\\ \nonumber
	&= \|\bx_{k}^{(t)}\|^2 + \alpha^2\big((\bx_{k}^{(t)})^T\tilde{\bC}_k^2\bx_{k}^{(t)} + ((\bx_{k}^{(t)})^T\bC\bx_{k}^{(t)})^2\|\bx_{k}^{(t)}\|^2 - 2(\bx_{k}^{(t)})^T\bC\bx_{k}^{(t)}(\bx_{k}^{(t)})^T\tilde{\bC}_k\bx_{k}^{(t)}\big)\\ \nonumber
	& + 2\alpha \big((\bx_{k}^{(t)})^T\tilde{\bC}_k\bx_{k}^{(t)} - (\bx_{k}^{(t)})^T\bC\bx_{k}^{(t)}\|\bx_{k}^{(t)}\|^2\big)\\ \nonumber
	&=  \|\bx_{k}^{(t)}\|^2 + \alpha^2\big((\bx_{k}^{(t)})^T( \bC^2 - \sum_{p=1}^{k-1}\lambda_p^2\bq_p\bq_p^T)\bx_{k}^{(t)} + ((\bx_{k}^{(t)})^T\bC\bx_{k}^{(t)})^2\|\bx_{k}^{(t)}\|^2 - 2(\bx_{k}^{(t)})^T\bC\bx_{k}^{(t)}\times\\\nonumber
	&(\bx_{k}^{(t)})^T(\bC - \sum_{p=1}^{k-1}\lambda_p\bq_{p}\bq_{p}^T)\bx_{k}^{(t)}\big) + 2\alpha \big((\bx_{k}^{(t)})^T( \bC - \sum_{p=1}^{k-1}\lambda_p\bq_{p}\bq_{p}^T)\bx_{k}^{(t)} - (\bx_{k}^{(t)})^T\bC\bx_{k}^{(t)}\|\bx_{k}^{(t)}\|^2\big)\\ \nonumber
	&=  \|\bx_{k}^{(t)}\|^2 + \alpha^2\big((\bx_{k}^{(t)})^T\bC^2\bx_{k}^{(t)} - \sum_{p=1}^{k-1}\lambda_p^2(\bq_p^T\bx_{k}^{(t)})^2 + ((\bx_{k}^{(t)})^T\bC\bx_{k}^{(t)})^2(\|\bx_{k}^{(t)}\|^2 - 2)  \\ \label{eq:lemma1_eq1}
	& + 2(\bx_{k}^{(t)})^T\bC\bx_{k}^{(t)}\sum_{p=1}^{k-1}\lambda_p(\bq_{p}^T\bx_{k}^{(t)})^2\big) +   2\alpha \big((\bx_{k}^{(t)})^T\bC\bx_{k}^{(t)}(1-\|\bx_{k}^{(t)}\|^2) - \sum_{p=1}^{k-1}\lambda_p(\bq_{p}^T\bx_{k}^{(t)})^2 \big).
\end{align}
We now split our analysis into three cases based on the range of values of $\|\bx_{k}^{(t)}\|^2$.

\textbf{Case \rom{1}}: Let $\|\bx_{k}^{(t)}\|^2 \leq 1$. Then we see from~\eqref{eq:lemma1_eq1} that
\begin{eqnarray*}
	\|\bx_k^{(t+1)}\|^2 &\leq&  1  + \alpha^2(\lambda_1^2 + 2\lambda_1\sum_{p=1}^{k-1}\lambda_p) + 2\alpha\lambda_1 \leq  1  + \alpha^2(\lambda_1^2 + 2\lambda_1\sum_{p=1}^{k-1}\lambda_1) + 2\alpha\lambda_1 \\
	&\leq&  1  + \alpha^2\lambda_1^2(2K-1) + 2\alpha\lambda_1\sqrt{2K-1} =  (1  + \alpha\lambda_1\sqrt{2K-1})^2 \\
	&\leq& 2(1 + \alpha^2\lambda_1^2(2K-1)) \leq 2(1 + \frac{1}{9(2K-1)}) \leq 2(1 + \frac{1}{9}) <  3.
\end{eqnarray*}
\textbf{Case \rom{2}}: Now suppose $1 < \|\bx_k^{(t)}\|^2 \leq 2$. Then from~\eqref{eq:lemma1_eq1} we have
\begin{eqnarray*}
	\|\bx_k^{(t+1)}\|^2 &\leq&  2  + \alpha^2(2\lambda_1^2 + 2\lambda_1\sum_{p=1}^{k-1}2\lambda_p) \leq  2  + \alpha^2(2\lambda_1^2 + 2\lambda_1\sum_{p=1}^{k-1}2\lambda_1)  \\
	&\leq&  2(1 + \frac{1}{9(2K-1)}) \leq 2(1 + \frac{1}{9}) < 3, \quad \text{using similar steps as Case \rom{1}}.
\end{eqnarray*}
\textbf{Case \rom{3}}: Finally suppose $2 < \|\bx_k^{(t)}\|^2 < 3$. Then from~\eqref{eq:lemma1_eq1} we get
\begin{eqnarray*}
	\|\bx_k^{(t+1)}\|^2 &<&  3  + \alpha^2\big((\bx_{k}^{(t)})^T\bC^2\bx_{k}^{(t)} - \sum_{p=1}^{k-1}\lambda_p^2(\bq_p^T\bx_{k}^{(t)})^2 + ((\bx_{k}^{(t)})^T\bC\bx_{k}^{(t)})^2(\|\bx_{k}^{(t)}\|^2 - 2)  \\
	&& + 2(\bx_{k}^{(t)})^T\bC\bx_{k}^{(t)}\sum_{p=1}^{k-1}\lambda_p(\bq_{p}^T\bx_{k}^{(t)})^2\big) +   2\alpha \big((\bx_{k}^{(t)})^T\bC\bx_{k}^{(t)}(1-\|\bx_{k}^{(t)}\|^2) - \sum_{p=1}^{k-1}\lambda_p(\bq_{p}^T\bx_{k}^{(t)})^2 \big).
\end{eqnarray*}
To show that $\|\bx_k^{(t+1)}\|^2 < 3$, we have to show
\begin{align}\nonumber
&\alpha^2\big((\bx_{k}^{(t)})^T\bC^2\bx_{k}^{(t)} - \sum_{p=1}^{k-1}\lambda_p^2(\bq_p^T\bx_{k}^{(t)})^2 + ((\bx_{k}^{(t)})^T\bC\bx_{k}^{(t)})^2(\|\bx_{k}^{(t)}\|^2 - 2)  \\ \nonumber
 & 2(\bx_{k}^{(t)})^T\bC\bx_{k}^{(t)}\sum_{p=1}^{k-1}\lambda_p(\bq_{p}^T\bx_{k}^{(t)})^2\big) +   2\alpha \big((\bx_{k}^{(t)})^T\bC\bx_{k}^{(t)}(1-\|\bx_{k}^{(t)}\|^2) - \sum_{p=1}^{k-1}\lambda_p(\bq_{p}^T\bx_{k}^{(t)})^2 \big) \leq 0\\ \label{eq:lemma1_alphabound}
 \Leftrightarrow \quad & \alpha \leq \frac{2(\bx_{k}^{(t)})^T\bC\bx_{k}^{(t)}(\|\bx_{k}^{(t)}\|^2 - 1) + 2\sum_{p=1}^{k-1}\lambda_p(\bq_{p}^T\bx_{k}^{(t)})^2}{(\bx_{k}^{(t)})^T\bC^2\bx_{k}^{(t)} - \sum_{p=1}^{k-1}\lambda_p^2(\bq_p^T\bx_{k}^{(t)})^2 + ((\bx_{k}^{(t)})^T\bC\bx_{k}^{(t)})^2(\|\bx_{k}^{(t)}\|^2 - 2)  +  2(\bx_{k}^{(t)})^T\bC\bx_{k}^{(t)}\sum_{p=1}^{k-1}\lambda_p(\bq_{p}^T\bx_{k}^{(t)})^2}.
\end{align}
We now find a lower bound of the right hand side of~\eqref{eq:lemma1_alphabound}. Note that
\begin{align}\label{eq:alpha_num}
&2(\bx_{k}^{(t)})^T\bC\bx_{k}^{(t)}(\|\bx_{k}^{(t)}\|^2 - 1) + 2\sum_{p=1}^{k-1}\lambda_p(\bq_{p}^T\bx_{k}^{(t)})^2 \geq 2(\bx_{k}^{(t)})^T\bC\bx_{k}^{(t)}(\|\bx_{k}^{(t)}\|^2 - 1)\\ \nonumber
\text{and} \quad
&(\bx_{k}^{(t)})^T\bC^2\bx_{k}^{(t)} - \sum_{p=1}^{k-1}\lambda_p^2(\bq_p^T\bx_{k}^{(t)})^2 + ((\bx_{k}^{(t)})^T\bC\bx_{k}^{(t)})^2(\|\bx_{k}^{(t)}\|^2 - 2) + 2(\bx_{k}^{(t)})^T\bC\bx_{k}^{(t)}\sum_{p=1}^{k-1}\lambda_p(\bq_{p}^T\bx_{k}^{(t)})^2 \\ \label{eq:alpha_den}
&\leq (\bx_{k}^{(t)})^T\bC^2\bx_{k}^{(t)} + ((\bx_{k}^{(t)})^T\bC\bx_{k}^{(t)})^2(\|\bx_{k}^{(t)}\|^2 - 2)  +  2(\bx_{k}^{(t)})^T\bC\bx_{k}^{(t)}\sum_{p=1}^{k-1}\lambda_p(\bq_{p}^T\bx_{k}^{(t)})^2.
\end{align}
Now, $\frac{(\bx_{k}^{(t)})^T\bC\bx_{k}^{(t)}}{(\bx_{k}^{(t)})^T\bC^2\bx_{k}^{(t)}}$ is a generalized Rayleigh quotient whose maximum and minimum values are the largest and smallest eigenvalues of the generalized eigenvalue problem $\bC\by = \lambda\bC^2\by$. Since the eigenvectors of $\bC$ and $\bC^2$ are the same, the largest and smallest eigenvalues of the generalized problems are $\frac{1}{\lambda_{d}}$ and $\frac{1}{\lambda_{1}}$, respectively, where $\lambda_{1}$ and $\lambda_{d}$ are the largest and smallest eigenvalues of $\bC$. Thus,  $(\bx_{k}^{(t)})^T\bC^2\bx_{k}^{(t)} \leq \lambda_1(\bx_{k}^{(t)})^T\bC\bx_{k}^{(t)}$. Also, since $\bq_{p}^T\bx_{k}^{(t)} \leq \|\bq\|\|\bx_k^{(t)}\|$, we have the right hand side of~\eqref{eq:alpha_den}
\begin{eqnarray*}
&& \lambda_1(\bx_{k}^{(t)})^T\bC\bx_{k}^{(t)} + ((\bx_{k}^{(t)})^T\bC\bx_{k}^{(t)})^2(\|\bx_{k}^{(t)}\|^2 - 2)  +  2(\bx_{k}^{(t)})^T\bC\bx_{k}^{(t)}\sum_{p=1}^{k-1}\lambda_p\|\bx_{k}^{(t)}\|^2 \\
&=&(\bx_{k}^{(t)})^T\bC\bx_{k}^{(t)}(\lambda_1 + (\bx_{k}^{(t)})^T\bC\bx_{k}^{(t)}(\|\bx_{k}^{(t)}\|^2 - 2)  +  2\sum_{p=1}^{k-1}\lambda_p\|\bx_{k}^{(t)}\|^2)\\
&\leq&(\bx_{k}^{(t)})^T\bC\bx_{k}^{(t)}(\lambda_1 + \lambda_1\|\bx_k^{(t)}\|^2(\|\bx_{k}^{(t)}\|^2 - 2)  +  2\sum_{p=1}^{k-1}\lambda_1\|\bx_{k}^{(t)}\|^2)\\
&=&\lambda_1(\bx_{k}^{(t)})^T\bC\bx_{k}^{(t)}(1 + \|\bx_k^{(t)}\|^4 - 2\|\bx_k^{(t)}\|^2  +  2(k-1)\|\bx_{k}^{(t)}\|^2)\\
&=&\lambda_1(\bx_{k}^{(t)})^T\bC\bx_{k}^{(t)}((\|\bx_k^{(t)}\|^2 - 1)^2 +  2(k-1)\|\bx_{k}^{(t)}\|^2)\\
&=&\lambda_1(\bx_{k}^{(t)})^T\bC\bx_{k}^{(t)}(\|\bx_k^{(t)}\|^2 - 1)((\|\bx_k^{(t)}\|^2 - 1) +  2(k-1)\frac{\|\bx_{k}^{(t)}\|^2}{(\|\bx_k^{(t)}\|^2 - 1)})\\
&<&\lambda_1(\bx_{k}^{(t)})^T\bC\bx_{k}^{(t)}(\|\bx_k^{(t)}\|^2 - 1)((3 - 1) +  2(k-1)2), \quad \text{since} \frac{\|\bx_{k}^{(t)}\|^2}{(\|\bx_k^{(t)}\|^2 - 1)} < 2  \\
&=&2\lambda_1(\bx_{k}^{(t)})^T\bC\bx_{k}^{(t)}(\|\bx_k^{(t)}\|^2 - 1)(2k-1) \leq 2\lambda_1(\bx_{k}^{(t)})^T\bC\bx_{k}^{(t)}(\|\bx_k^{(t)}\|^2 - 1)(2K -1).
\end{eqnarray*}
Hence, we have that the right hand side of~\eqref{eq:lemma1_alphabound} exceeds
\begin{eqnarray*}
  \frac{2(\bx_{k}^{(t)})^T\bC\bx_{k}^{(t)}(\|\bx_{k}^{(t)}\|^2 - 1)}{2\lambda_1(\bx_{k}^{(t)})^T\bC\bx_{k}^{(t)}(\|\bx_k^{(t)}\|^2 - 1)(2K -1)} = \frac{1}{\lambda_1(2K -1)} > \frac{1}{3\lambda_1(2K -1)}.
\end{eqnarray*}
Thus, if $\alpha \leq \frac{1}{3\lambda_1(2K -1)}$, then $\|\bx_{k}^{(t)}\|^2 < 3$. 

Next,
\begin{equation} \label{eq:rayleigh_upperboundk}
	0 \leq (\bx_k^{(t)})^T\bC\bx_k^{(t)} \leq \lambda_1\|\bx_k^{(t)}\|^2 < 3\lambda_1 \leq 3(2K-1)\lambda_1 \leq \frac{1}{\alpha}.
\end{equation}
Hence, $(\bx_k^{(t)})^T\bC\bx_k^{(t)} < \frac{1}{\alpha}$. 
\end{proof}

\subsection{Statement and Proof of Lemma~\ref{lemma:lowerbound_rayleigh}}\label{app:lemma2}
\begin{lemma} \label{lemma:lowerbound_rayleigh}
	Suppose $\bq_k^T\bx_k^{(0)} = z_{k,k}^{(0)} \neq 0$ and $(\bx_k^{(t)})^T\bC\bx_k^{(t)} < \frac{1}{\alpha}$, then
	\begin{equation*}
		(\bx_k^{(t)})^T\bC\bx_k^{(t)} > \min \{(1 - 3\alpha\lambda_1)^2\lambda_m, (\tilde{\bx}_k^{(0)})^T\bC\tilde{\bx}_k^{(0)}\}, \quad \forall t.
	\end{equation*}
\end{lemma}
\begin{proof}
We know $ 0  \leq (\bx_k^{(t)})^T\bC\bx_k^{(t)} \leq \lambda_1\|\bx_k^{(t)}\|^2 < 3\lambda_1~\text{using Lemma~\ref{lemma:bounded_iterate_rayleigh}} $.
Let $\lambda_m, m > K$ be the smallest non-zero eigenvalue of $\bC$. Now, if $\lambda_m \leq (\bx_k^{(t)})^T\bC\bx_k^{(t)} < 3\lambda_1$, then
\begin{align}\nonumber
    ({\bx}_k^{(t+1)})^T\bC{\bx}_k^{(t+1)} &= \sum_{l=1}^{d}\lambda_l(z_{k,l}^{(t+1)})^2 \\ \nonumber
    &= \sum_{l=1}^{k-1}\lambda_l(z_{k,l}^{(t+1)})^2 + \sum_{l=k}^{d}\lambda_l(z_{k,l}^{(t+1)})^2 \\ \nonumber
    &= \sum_{l=1}^{k-1}\lambda_l(1 - \alpha({\bx}_{k}^{(t)})^T\bC{\bx}_{k}^{(t)})^2(z_{k,l}^{(t)})^2 + \sum_{l=k}^{d} \lambda_l(1 + \alpha(\lambda_l - ({\bx}_{k}^{(t)})^T\bC{\bx}_{k}^{(t)}))^2({z}_{k,l}^{(t)})^2\\ \nonumber
    &\geq (1 - \alpha({\bx}_{k}^{(t)})^T\bC{\bx}_{k}^{(t)})^2\sum_{l=1}^{k-1}\lambda_l(z_{k,l}^{(t)})^2 + (1 + \alpha(\lambda_m - ({\bx}_{k}^{(t)})^T\bC{\bx}_{k}^{(t)}))^2\sum_{l=k}^{d} \lambda_l({z}_{k,l}^{(t)})^2\\ \nonumber
    &> (1 - \alpha({\bx}_{k}^{(t)})^T\bC{\bx}_{k}^{(t)})^2\sum_{l=1}^{k-1}\lambda_l(z_{k,l}^{(t)})^2 + (1 - \alpha({\bx}_{k}^{(t)})^T\bC{\bx}_{k}^{(t)}))^2\sum_{l=k}^{d} \lambda_l({z}_{k,l}^{(t)})^2\\ \nonumber
    &= (1 - \alpha({\bx}_{k}^{(t)})^T\bC{\bx}_{k}^{(t)})^2\sum_{l=1}^{d}\lambda_l(z_{k,l}^{(t)})^2 \\ \label{eq:rayleigh_lowerbound1}
    &> (1 - 3\alpha\lambda_1)^2({\bx}_k^{(t)})^T\bC{\bx}_k^{(t)} \geq  (1 - 3\alpha\lambda_1)^2\lambda_m.
\end{align}
Also, from \eqref{eq:expansion_centralized}, we have $\bx_{k}^{(t)} = \sum_{l=1}^{d}z_{k,l}^{(t)}\bq_l = \sum_{l=1}^{k-1}z_{k,l}^{(t)}\bq_l + \sum_{l=k}^{d}z_{k,l}^{(t)}\bq_l.$
Let $\sum_{l=1}^{k-1}z_{k,l}^{(t)}\bq_l = \bx_k^{'(t)}$ and $\sum_{l=k}^{d}z_{k,l}^{(t)}\bq_l = \tilde{\bx}_k^{(t)}$. Thus, $(\bx_k^{(t)})^T\bC\bx_k^{(t)} = (\tilde{\bx}_k^{(t)})^T\bC\tilde{\bx}_k^{(t)} + (\bx_k^{'(t)})^T\bC\bx_k^{'(t)}$. 

Now, if  $(\tilde{\bx}_k^{(t)})^T\bC\tilde{\bx}_k^{(t)} \leq (\bx_k^{(t)})^T\bC\bx_k^{(t)} <  \lambda_m$ then
\begin{align}\nonumber
    ({\bx}_k^{(t+1)})^T\bC{\bx}_k^{(t+1)} \geq (\tilde{\bx}_k^{(t+1)})^T\bC\tilde{\bx}_k^{(t+1)} &= \sum_{l=k}^{d}\lambda_l(z_{k,l}^{(t+1)})^2 \\ \nonumber
    &= \sum_{l=k}^{d} \lambda_l(1 + \alpha(\lambda_l - ({\bx}_{k}^{(t)})^T\bC{\bx}_{k}^{(t)}))^2({z}_{k,l}^{(t)})^2\\ \label{eq:rayleigh_lowerbound2}
    &\geq (1 + \alpha(\lambda_m - ({\bx}_{k}^{(t)})^T\bC{\bx}_{k}^{(t)}))^2\sum_{l=k}^{d} \lambda_l({z}_{k,l}^{(t)})^2 > (\tilde{\bx}_k^{(t)})^T\bC\tilde{\bx}_k^{(t)}.
\end{align}
Combining \eqref{eq:rayleigh_lowerbound1} and \eqref{eq:rayleigh_lowerbound2}, we have 
\begin{align}\label{eq:rayleigh_lowerboundk}
    (\bx_k^{(t)})^T\bC\bx_k^{(t)} > \min \{(1 - 3\alpha\lambda_1)^2\lambda_m, (\tilde{\bx}_k^{(0)})^T\bC\tilde{\bx}_k^{(0)}\}.
\end{align} 
\end{proof}
\subsection{Statement and Proof of Lemma~\ref{lemma:bounded_rayleighk}}\label{app:lemma6}
\begin{lemma}\label{lemma:bounded_rayleighk}
	Assume $\|\bx_{i,k}^{(0)}\| = 1$. If the step size is bounded above as $\alpha \leq \frac{w_{ii}}{3\lambda_1{(2K-1)}}$, where $\lambda_1$ is the largest eigenvalue of $\bC$ and $K$ is the number of eigenvectors to be estimated, then
	\begin{equation}
		\|\bx_{i,k}^{(t)}\| < \sqrt{3} \quad \text{and} \quad (\bx_{i,k}^{(t)})^T\bC_i\bx_k^{(t)} < \frac{1}{\alpha}, \quad \forall k,t.
	\end{equation}
\end{lemma}
\begin{proof}
We have
\begin{eqnarray}
\bx_{i,k}^{(t+1)} &=& \sum_{j\in \cN_i \cup \{i\}}w_{ij}\bx_{j,k}^{(t)} + \alpha \big(\bC_i\bx_{i,k}^{(t)} - (\bx_{i,k}^{(t)})^T\bC_i\bx_{i,k}^t\bx_{i,k}^{(t)} - \sum_{p=1}^{k-1}\bx_{i,p}^{(t)}(\bx_{i,p}^{(t)})^T\bC_i\bx_{i,k}^{(t)}\big).
\end{eqnarray}
Hence,
\begin{eqnarray*}
	\|\bx_{i,k}^{(t+1)}\| &\leq& \|w_{ii}\bx_{i,k}^{(t)} + \alpha \big( {\bC}_i\bx_{i,k}^{(t)} - (\bx_{i,k}^{(t)})^T\bC_i\bx_{i,k}^{(t)}\bx_{i,k}^{(t)} \big)\| +\alpha\sum_{p=1}^{k-1}\|\bx_{i,p}^{(t)}(\bx_{i,p}^{(t)})^T\bC_i\bx_{i,k}^{(t)}\|+ \sum_{j\neq i} \|w_{ij}\bx_{j,k}^{(t)}\|\\
	&\leq& \|w_{ii}\bx_{i,k}^{(t)} + \alpha \big( {\bC}_i\bx_{i,k}^{(t)} - (\bx_{i,k}^{(t)})^T\bC_i\bx_{i,k}^{(t)}\bx_{i,k}^{(t)} \big)\| +\alpha\sum_{p=1}^{k-1}\lambda_1\|(\bx_{i,p}^{(t)})\|\|\bx_{i,k}^{(t)}\|\|\bx_{i,p}^{(t)}\|+ \sum_{j\neq i}w_{ij} \|\bx_{j,k}^{(t)}\|\\
	&=& \|w_{ii}\bx_{i,k}^{(t)} + \alpha \big( {\bC}_i\bx_{i,k}^{(t)} - (\bx_{i,k}^{(t)})^T\bC_i\bx_{i,k}^{(t)}\bx_{i,k}^{(t)} \big)\| +\alpha\sum_{p=1}^{k-1}\lambda_1\|(\bx_{i,p}^{(t)})\|^2\|\bx_{i,k}^{(t)}\| +  \sum_{j\neq i}w_{ij}\|\bx_{j,k}^{(t)}\|\\
	&\leq& \|w_{ii}\bx_{i,k}^{(t)} + \alpha \big( {\bC}_i\bx_{i,k}^{(t)} - (\bx_{i,k}^{(t)})^T\bC_i\bx_{i,k}^{(t)}\bx_{i,k}^{(t)} \big)\| +3\alpha\lambda_1\sum_{p=1}^{k-1}\|\bx_{i,k}^{(t)}\| +  \sum_{j\neq i}w_{ij}\|\bx_{j,k}^{(t)}\|\\
	&=& \|w_{ii}\bx_{i,k}^{(t)} + \alpha \big( {\bC}_i\bx_{i,k}^{(t)} - (\bx_{i,k}^{(t)})^T\bC_i\bx_{i,k}^{(t)}\bx_{i,k}^{(t)} \big)\| +3(k-1)\alpha\lambda_1\|\bx_{i,k}^{(t)}\| + \sum_{j\neq i}w_{ij}\|\bx_{j,k}^{(t)}\|.
\end{eqnarray*}
Now,
\begin{align*}
	&\|w_{ii}\bx_{i,k}^{(t)} + \alpha (\bC_i\bx_{i,k}^{(t)} - ((\bx_{i,k}^{(t)})^T\bC_i\bx_{i,k}^{(t)})\bx_{i,k}^{(t)})\|^2 \\
	&= w_{ii}^2\|\bx_{i,k}^{(t)}\|^2 + \alpha^2\|\bC_i\bx_{i,k}^{(t)} - ((\bx_{i,k}^{(t)})^T\bC_i\bx_{i,k}^{(t)})\bx_{i,k}^{(t)}\|^2 + 2\alpha w_{ii}(\bx_{i,k}^{(t)})^T(\bC_i\bx_{i,k}^{(t)} - (\bx_{i,k}^{(t)})^T\bC_i\bx_{i,k}^{(t)}\bx_{i,k}^{(t)})\\
	&= w_{ii}^2\|\bx_{i,k}^{(t)}\|^2  + 2\alpha w_{ii}(\bx_{i,k}^{((t))})^T\bC_i\bx_{i,k}^{(t)} (1 - \|\bx_{i,k}^{(t)}\|^2) +\alpha^2(\bx_{i,k}^{(t)})^T\bC_i^2\bx_{i,k}^{(t)} + \alpha^2((\bx_{i,k}^{(t)})^T\bC_i\bx_{i,k}^{(t)})^2(\|\bx_{i,k}^{(t)}\|^2 - 2).
\end{align*}
\textbf{Case \rom{1}}: Let us assume $\|\bx_{i,k}^t\|^2 \leq 1, \forall i$. Then, we have
\begin{eqnarray*}
 \|w_{ii}\bx_{i,k}^{(t)} + \alpha \big( {\bC}_i\bx_{i,k}^{(t)} - (\bx_{i,k}^{(t)})^T\bC_i\bx_{i,k}^{(t)}\bx_{i,k}^{(t)} \big)\|^2 &\leq& (w_{ii} + \alpha\lambda_1)^2 \leq \big(w_{ii} + \frac{w_{ii}}{3(2K-1)}\big)^2.
 \end{eqnarray*}
Thus,
\begin{eqnarray*}
\|\bx_{i,k}^{t+1}\| &\leq& w_{ii}\big(1+ \frac{1}{3(2K-1)}\big) +  \frac{3(k-1)}{3(2K-1)} + (1-w_{ii})\\
&<& \frac{1}{3(2K-1)} + \frac{k-1}{2K-1} + 1 = \frac{k- 0.67}{2(K-0.5)} + 1 \\
&\leq&  \frac{K- 0.67}{2(K-0.5)} + 1 < 1.5 < \sqrt{3}.
\end{eqnarray*}
\textbf{Case \rom{2}}: Now, suppose $1 \leq \|\bx_{i,k}^t\|^2 < 2, \forall i$. Then, we get
\begin{align*}
	\|w_{ii}\bx_{i,k}^{(t)} + \alpha \big( {\bC}_i\bx_{i,k}^{(t)} - (\bx_{i,k}^{(t)})^T\bC_i\bx_{i,k}^{(t)}\bx_{i,k}^{(t)} \big)\|^2 \leq  2w_{ii}^2  + 2\alpha^2\lambda_1^2 < 2(w_{ii} + \alpha\lambda_1)^2.
\end{align*}
Thus, if we need $\|\bx_{i,k}^{(t+1)}\| \leq \sqrt{3}$, the following condition should be met:
\begin{align*}
	\|\bx_{i,k}^{t+1}\| &\leq \sqrt{2}w_{ii}(1+ \alpha \lambda_1) +  3(k-1)\alpha\lambda_1\sqrt{2} + (1-w_{ii})\sqrt{2} \leq \sqrt{3}\\
	\Leftrightarrow \quad & \sqrt{2}+ \sqrt{2}w_{ii}\alpha \lambda_1 +  3(k-1)\alpha\lambda_1\sqrt{2} \leq \sqrt{3} \\
	\Leftrightarrow \quad & \sqrt{2}\alpha \lambda_1 +  3(k-1)\alpha\lambda_1\sqrt{2} \leq \sqrt{3} -\sqrt{2}\\
	\Leftrightarrow \quad & \sqrt{2}\alpha \lambda_1(3k-2) \leq \sqrt{3} -\sqrt{2} \Leftrightarrow \quad \sqrt{2}\alpha \lambda_1(3K-2) \leq \sqrt{3} -\sqrt{2}\\
	\Leftrightarrow \quad &\alpha \leq \frac{\sqrt{3} -\sqrt{2}}{\sqrt{2}\lambda_1(3K-2)} = \frac{\sqrt{1.5} -1 }{ \lambda_1(3K-2)} =  \frac{0.225}{\lambda_1(3K-2)}.
\end{align*}
Since $\frac{0.225}{\lambda_13(2K-1)} <\frac{0.225}{\lambda_1(3K-2)} $, if $\alpha \leq \frac{0.225}{3\lambda_1(2K-1)}, \quad \text{then} \quad \|\bx_{i,k}^{(t+1)}\| \leq \sqrt{3}$.

\textbf{Case \rom{3}}: Finally, suppose $2 \leq \|\bx_{i,k}^{(t)}\|^2 \leq 3, \forall i$. We then have the following: $\sum_{j\neq i}w_{ij}\|\bx_{j,k}^{t}\| \leq \sum_{j\neq i}w_{ij}\sqrt{3}= (1-w_{ii})\sqrt{3}$. 

Now, if we desire $\|\bx_{i,k}^{(t+1)}\| \leq \sqrt{3}$, then we need
\begin{eqnarray*}
&& \|w_{ii}\bx_{i,k}^{(t)} + \alpha \big( {\bC}_i\bx_{i,k}^{(t)} - (\bx_{i,k}^{(t)})^T\bC_i\bx_{i,k}^{(t)}\bx_{i,k}^{(t)} \big)\| + 3(k-1)\alpha\lambda_1\|\bx_{i,k}^{(t)}\| + \sum_{j\neq i}w_{ij}\sqrt{3} \leq \sqrt{3}\\
 \Leftrightarrow \quad &&\|w_{ii}\bx_{i,k}^{(t)} + \alpha \big( {\bC}_i\bx_{i,k}^{(t)} - (\bx_{i,k}^{(t)})^T\bC_i\bx_{i,k}^{(t)}\bx_{i,k}^{(t)} \big)\| + 3(k-1)\alpha\lambda_1\|\bx_{i,k}^{(t)}\| + (1-w_{ii})\sqrt{3} \leq \sqrt{3}\\
 \Leftrightarrow \quad &&\|w_{ii}\bx_{i,k}^{(t)} + \alpha \big( {\bC}_i\bx_{i,k}^{(t)} - (\bx_{i,k}^{(t)})^T\bC_i\bx_{i,k}^{(t)}\bx_{i,k}^{(t)} \big)\|  \leq \sqrt{3} - 3(k-1)\alpha\lambda_1\|\bx_{i,k}^{(t)}\| - (1-w_{ii})\sqrt{3}\\
 \Leftrightarrow \quad&& \|w_{ii}\bx_{i,k}^{(t)} + \alpha \big( {\bC}_i\bx_{i,k}^{(t)} - (\bx_{i,k}^{(t)})^T\bC_i\bx_{i,k}^{(t)}\bx_{i,k}^{(t)} \big)\|^2  \leq 3w_{ii}^2 - 6\sqrt{3}(k-1)w_{ii}\alpha\lambda_1\|\bx_{i,k}^{(t)}\| + 9\alpha^2\lambda_1^2(k-1)^2\|\bx_{i,k}^{(t)}\|^2.
\end{eqnarray*}
Therefore, we need
\begin{align} \nonumber
& w_{ii}^2\|\bx_{i,k}^{(t)}\|^2  + 2\alpha w_{ii}(\bx_{i,k}^{((t))})^T\bC_i\bx_{i,k}^{(t)} (1 - \|\bx_{i,k}^{(t)}\|^2) +
\alpha^2(\bx_{i,k}^{(t)})^T\bC_i^2\bx_{i,k}^{(t)} + \alpha^2((\bx_{i,k}^{(t)})^T\bC_i\bx_{i,k}^{(t)})^2(\|\bx_{i,k}^{(t)}\|^2 - 2) \\ \nonumber
&\leq 3w_{ii}^2 - 6\sqrt{3}(k-1)w_{ii}\alpha\lambda_1\|\bx_{i,k}^{(t)}\| + 9\alpha^2\lambda_1^2(k-1)^2\|\bx_{i,k}^{(t)}\|^2\\ \nonumber
\Leftrightarrow \quad &3w_{ii}^2  + 2\alpha w_{ii}(\bx_{i,k}^{((t))})^T\bC_i\bx_{i,k}^{(t)} (1 - \|\bx_{i,k}^{(t)}\|^2) +
\alpha^2(\bx_{i,k}^{(t)})^T\bC_i^2\bx_{i,k}^{(t)} + \alpha^2((\bx_{i,k}^{(t)})^T\bC_i\bx_{i,k}^{(t)})^2(\|\bx_{i,k}^{(t)}\|^2 - 2) \\ \nonumber
&\leq 3w_{ii}^2 - 6\sqrt{3}(k-1)w_{ii}\alpha\lambda_1\|\bx_{i,k}^{(t)}\| + 9\alpha^2\lambda_1^2(k-1)^2\|\bx_{i,k}^{(t)}\|^2\\ \nonumber
\Leftrightarrow \quad & 2\alpha w_{ii}(\bx_{i,k}^{((t))})^T\bC_i\bx_{i,k}^{(t)} (1 - \|\bx_{i,k}^{(t)}\|^2) +
\alpha^2(\bx_{i,k}^{(t)})^T\bC_i^2\bx_{i,k}^{(t)} + \alpha^2((\bx_{i,k}^{(t)})^T\bC_i\bx_{i,k}^{(t)})^2(\|\bx_{i,k}^{(t)}\|^2 - 2) \\ \nonumber
&\leq  - 6\sqrt{3}(k-1)w_{ii}\alpha\lambda_1\|\bx_{i,k}^{(t)}\| + 9\alpha^2\lambda_1^2(k-1)^2\|\bx_{i,k}^{(t)}\|^2\\ \nonumber
\Leftrightarrow \quad & \alpha^2(\bx_{i,k}^{(t)})^T\bC_i^2\bx_{i,k}^{(t)} + \alpha^2((\bx_{i,k}^{(t)})^T\bC_i\bx_{i,k}^{(t)})^2(\|\bx_{i,k}^{(t)}\|^2 - 2) - 9\alpha^2\lambda_1^2(k-1)^2\|\bx_{i,k}^{(t)}\|^2\\ \nonumber
&\leq 2\alpha w_{ii}(\bx_{i,k}^{((t))})^T\bC_i\bx_{i,k}^{(t)} (\|\bx_{i,k}^{(t)}\|^2 - 1) - 6\sqrt{3}(k-1)w_{ii}\alpha\lambda_1\|\bx_{i,k}^{(t)}\| \\ \label{eq:lemma6_alphabound}
\Leftrightarrow \quad &\alpha \leq \frac{2 w_{ii}(\bx_{i,k}^{((t))})^T\bC_i\bx_{i,k}^{(t)} (\|\bx_{i,k}^{(t)}\|^2 - 1) - 6\sqrt{3}(k-1)w_{ii}\lambda_1\|\bx_{i,k}^{(t)}\|}{(\bx_{i,k}^{(t)})^T\bC_i^2\bx_{i,k}^{(t)} + ((\bx_{i,k}^{(t)})^T\bC_i\bx_{i,k}^{(t)})^2(\|\bx_{i,k}^{(t)}\|^2 - 2) - 9\lambda_1^2(k-1)^2\|\bx_{i,k}^{(t)}\|^2}.
\end{align}
We now find the lower bound of the right-hand side of~\eqref{eq:lemma6_alphabound}. Note that
\begin{align}\nonumber
&(\bx_{i,k}^{(t)})^T\bC_i^2\bx_{i,k}^{(t)} + ((\bx_{i,k}^{(t)})^T\bC_i\bx_{i,k}^{(t)})^2(\|\bx_{i,k}^{(t)}\|^2 - 2) - 9\lambda_1^2(k-1)^2\|\bx_{i,k}^{(t)}\|^2\\ \nonumber
&\leq \lambda_1(\bx_{i,k}^{(t)})^T\bC_i\bx_{i,k}^{(t)} + ((\bx_{i,k}^{(t)})^T\bC_i\bx_{i,k}^{(t)})^2(\|\bx_{i,k}^{(t)}\|^2 - 2) - 9\lambda_1^2(k-1)^2\|\bx_{i,k}^{(t)}\|^2, \quad \text{since} \quad (\bx_{i,k}^{(t)})^T\bC_i^2\bx_{i,k}^{(t)} \leq \lambda_1(\bx_{i,k}^{(t)})^T\bC_i\bx_{i,k}^{(t)}\\ \nonumber
&\leq \lambda_1(\bx_{i,k}^{(t)})^T\bC_i\bx_{i,k}^{(t)} + \lambda_1(\bx_{i,k}^{(t)})^T\bC_i\bx_{i,k}^{(t)}\|\bx_{i,k}^{(t)}\|^2(\|\bx_{i,k}^{(t)}\|^2 - 2) - 9\lambda_1^2(k-1)^2\|\bx_{i,k}^{(t)}\|^2\\ \nonumber
&= \lambda_1(\bx_{i,k}^{(t)})^T\bC_i\bx_{i,k}^{(t)}(\|\bx_{i,k}^{(t)}\|^2 - 1)^2 - 9\lambda_1^2(k-1)^2\|\bx_{i,k}^{(t)}\|^2\\ \nonumber
&\leq \lambda_1(k-1)(\bx_{i,k}^{(t)})^T\bC_i\bx_{i,k}^{(t)}(\|\bx_{i,k}^{(t)}\|^2 - 1)^2 - 9\lambda_1^2(k-1)^2\|\bx_{i,k}^{(t)}\|^2\\ \nonumber
&< \lambda_1(k-1)(\|\bx_{i,k}^{(t)}\|^2 - 1)\Big((\bx_{i,k}^{(t)})^T\bC_i\bx_{i,k}^{(t)}(\|\bx_{i,k}^{(t)}\|^2 - 1) - 9\lambda_1(k-1)\Big) \quad \text{since}\quad \frac{\|\bx_{i,k}^{(t)}\|^2}{\|\bx_{i,k}^{(t)}\|^2 - 1}> 1
\end{align}
and, 
\begin{align}\nonumber
2w_{ii}(\bx_{i,k}^{((t))})^T\bC_i\bx_{i,k}^{(t)} (\|\bx_{i,k}^{(t)}\|^2 - 1) - 6\sqrt{3}(k-1)w_{ii}\lambda_1\|\bx_{i,k}^{(t)}\|
&\geq 2 w_{ii}(\bx_{i,k}^{((t))})^T\bC_i\bx_{i,k}^{(t)} (\|\bx_{i,k}^{(t)}\|^2 - 1) - 18(k-1)w_{ii}\lambda_1\\ \nonumber
&= 2w_{ii}((\bx_{i,k}^{((t))})^T\bC_i\bx_{i,k}^{(t)} (\|\bx_{i,k}^{(t)}\|^2 - 1) - 9(k-1)\lambda_1).
\end{align}
Thus, we have that the right hand side of~\eqref{eq:lemma6_alphabound} exceeds
\begin{eqnarray*}
&&\frac{2w_{ii}((\bx_{i,k}^{((t))})^T\bC_i\bx_{i,k}^{(t)} (\|\bx_{i,k}^{(t)}\|^2 - 1) - 9(k-1)\lambda_1)}{\lambda_1(k-1)(\|\bx_{i,k}^{(t)}\|^2 - 1)\Big((\bx_{i,k}^{(t)})^T\bC_i\bx_{i,k}^{(t)}(\|\bx_{i,k}^{(t)}\|^2 - 1) - 9\lambda_1(k-1)\Big)} = \frac{2w_{ii}}{\lambda_1(k-1)} >   \frac{w_{ii}}{3\lambda_1(2K-1)}.
\end{eqnarray*}
This proves if $\alpha \leq \min\{\frac{w_{ii}}{3\lambda_1(2K-1)}, \frac{0.225}{3\lambda_1(2K-1)}\} \text{ then } \|\bx_{i,k}^{(t+1)}\| \leq \sqrt{3}$.
\end{proof}

\subsection{Statement and Proof of Lemma~\ref{lemma:bounded_sanger}}\label{app:lemma7}
\begin{lemma}\label{lemma:bounded_sanger}
	The norm of Sanger's direction $\bcH_i(\bx_{i,k}^{(t)})$ is bounded as
	\begin{align}
	    \|\bcH_i(\bx_{i,k}^{(t)})\|^2 \leq 3\lambda_{i,1}^2(3k-2)(3k+1), \forall k=1,\ldots,K.
	\end{align}
\end{lemma}
\begin{proof}
We know
\begin{align*}
\bcH_i(\bx_{i,k}^{(t)}) &= \bC_i\bx_{i,k}^{(t)} - (\bx_{i,k}^{(t)})^T\bC_i\bx_{i,k}^{(t)}\bx_{i,k}^{(t)} - \sum_{p=1}^{k-1}\bx_{i,p}^{(t)}(\bx_{i,p}^{(t)})^T\bC_i\bx_{i,k}^{(t)}\\
&= (\bI - \sum_{p=1}^{k-1}\bx_{i,p}^{(t)}(\bx_{i,p}^{(t)})^T) \bC_i\bx_{i,k}^{(t)} - (\bx_{i,k}^{(t)})^T\bC_i\bx_{i,k}^{(t)}\bx_{i,k}^{(t)} = \tilde{\bC}_i^{(t)}\bx_{i,k}^{(t)} - (\bx_{i,k}^{(t)})^T\bC_i\bx_{i,k}^{(t)}\bx_{i,k}^{(t)}\\
\|\bcH_i(\bx_{i,k}^{(t)}) \|^2 &= \| \tilde{\bC}_i^{(t)}\bx_{i,k}^{(t)} - (\bx_{i,k}^{(t)})^T\bC_i\bx_{i,k}^{(t)}\bx_{i,k}^{(t)}\|^2\\
&= (\bx_{i,k}^{(t)})^T(\tilde{\bC}_i^{(t)})^T\tilde{\bC}_i^{(t)}\bx_{i,k}^{(t)}+ ((\bx_{i,k}^{(t)})^T\bC_i\bx_{i,k}^{(t)})^2(\|\bx_{i,k}^{(t)}\|^2 - 2) + (\bx_{i,k}^{(t)})^T\bC_i\bx_{i,k}^{(t)} \sum_{p=1}^{k-1}(\bx_{i,k}^{(t)})^T\bx_{i,p}^{(t)}(\bx_{i,p}^{(t)})^T \bC_i\bx_{i,k}^{(t)}.
\end{align*}
Next, notice that $\|\tilde{\bC}_i^{(t)}\| = \|\bC_i - \sum_{p=1}^{k-1}\bx_{i,p}^{(t)}(\bx_{i,p}^{(t)})^T \bC_i\|$. Thus,
\begin{eqnarray*}
\|\tilde{\bC}_i^{(t)}\| &\leq&  \|\bC_i\| + \sum_{p=1}^{k-1}\|\bx_{i,p}^{(t)}(\bx_{i,p}^{(t)})^T\|\| \bC_i\| \leq \lambda_{i,1} + \sum_{p=1}^{k-1}3\lambda_{i,1} = \lambda_{i,1}  + 3(k-1)\lambda_{i,1} = \lambda_{i,1}(3k-2).
\end{eqnarray*}
We, therefore, get
\begin{eqnarray*}
(\bx_{i,k}^{(t)})^T(\tilde{\bC}_i^{(t)})^T\tilde{\bC}_i^{(t)}\bx_{i,k}^{(t)} \leq  \lambda_{\max}((\tilde{\bC}_i^{(t)})^T)(\bx_{i,k}^{(t)})^T\tilde{\bC}_i^{(t)}\bx_{i,k}^{(t)} = \|(\tilde{\bC}_i^{(t)})\|(\bx_{i,k}^{(t)})^T\tilde{\bC}_i^{(t)}\bx_{i,k}^{(t)} \leq \lambda_{i,1}(3k-2)(\bx_{i,k}^{(t)})^T\tilde{\bC}_i^{(t)}\bx_{i,k}^{(t)}.
\end{eqnarray*}
Thus,
\begin{align*}
\|\bcH_i(\bx_{i,k}^{(t)}) \|^2 &\leq \lambda_{i,1}(3k-2)(\bx_{i,k}^{(t)})^T\tilde{\bC}_i^{(t)}\bx_{i,k}^{(t)} + ((\bx_{i,k}^{(t)})^T\bC_i\bx_{i,k}^{(t)})^2(\|\bx_{i,k}^{(t)}\|^2 - 2) + (\bx_{i,k}^{(t)})^T\bC_i\bx_{i,k}^{(t)} \sum_{p=1}^{k-1}(\bx_{i,k}^{(t)})^T\bx_{i,p}^{(t)}(\bx_{i,p}^{(t)})^T \bC_i\bx_{i,k}^{(t)}\\
&\leq \lambda_{i,1}(3k-2)(\bx_{i,k}^{(t)})^T\tilde{\bC}_i^{(t)}\bx_{i,k}^{(t)} + ((\bx_{i,k}^{(t)})^T\bC_i\bx_{i,k}^{(t)})^2(\|\bx_{i,k}^{(t)}\|^2 - 2) + (\bx_{i,k}^{(t)})^T\bC_i\bx_{i,k}^{(t)} \sum_{p=1}^{k-1}\|\bx_{i,k}^{(t)}\|^2\|\bx_{i,p}^{(t)}\|^2 \lambda_{i,1}
\end{align*}
\begin{align*}
\|\bcH_i(\bx_{i,k}^{(t)}) \|^2 &\leq \lambda_{i,1}(3k-2)(\bx_{i,k}^{(t)})^T(\bI - \sum_{p=1}^{k-1}\bx_{i,p}^{(t)}(\bx_{i,p}^{(t)})^T) \bC_i\bx_{i,k}^{(t)} + ((\bx_{i,k}^{(t)})^T\bC_i\bx_{i,k}^{(t)})^2(\|\bx_{i,k}^{(t)}\|^2 - 2) + \\
& (\bx_{i,k}^{(t)})^T\bC_i\bx_{i,k}^{(t)} \sum_{p=1}^{k-1}\|\bx_{i,k}^{(t)}\|^2\|\bx_{i,p}^{(t)}\|^2 \lambda_{i,1}\\
&\leq \lambda_{i,1}(3k-2)(\bx_{i,k}^{(t)})^T \bC_i\bx_{i,k}^{(t)} +\lambda_{i,1}(3k-2)\sum_{p=1}^{k-1}\|\bx_{i,k}^{(t)}\|^2\|\bx_{i,p}^{(t)}\|^2\lambda_{i,1} + \\ &((\bx_{i,k}^{(t)})^T\bC_i\bx_{i,k}^{(t)})^2(\|\bx_{i,k}^{(t)}\|^2 - 2) + (\bx_{i,k}^{(t)})^T\bC_i\bx_{i,k}^{(t)} \sum_{p=1}^{k-1}\|\bx_{i,k}^{(t)}\|^2\|\bx_{i,p}^{(t)}\|^2 \lambda_{i,1}\\
&\leq  \lambda_{i,1}(3k-2)3\lambda_{i,1} + \lambda_{i,1}(3k-2)\sum_{p=1}^{k-1}9\lambda_{i,1} + 9\lambda_{i,1}^2 + \lambda_{i,1}3 \sum_{p=1}^{k-1}9\lambda_{i,1} = 3\lambda_{i,1}^2(3k-2)(3k+1).
\end{align*}
\end{proof}

\section{Statement and Proof of Lemma~\ref{lemma:coeff_decay_lower}}\label{app:lemma3}
\begin{lemma}\label{lemma:coeff_decay_lower}
	Let $\eta = \min \{(1 - 3\alpha\lambda_1)^2\lambda_m, (\tilde{\bx}_k^{(0)})^T\bC\tilde{\bx}_k^{(0)}\}$. Now, suppose  $\eta < (\bx_k^{(t)})^T\bC\bx_k^{(t)}  < \frac{1}{\alpha}$, then the following is true for $\gamma = 1 - \alpha\eta$,  and some constant $a_1 > 0$:
	\begin{equation}
	 \sum_{l=1}^{k-1}(z_{k,l}^{(t+1)})^2 < a_1\gamma^{t+1}.
	\end{equation}
\end{lemma}
\begin{proof}
For $l = 1,\ldots,k-1$, we know from \eqref{eq:eqn_z_lower}
\begin{eqnarray*}
	{z}_{k,l}^{(t+1)} &=& (1 - \alpha({\bx}_{k}^{(t)})^T\bC{\bx}_{k}^{(t)}) {z}_{k,l}^{(t)} \\
	\text{or,}\quad ({z}_{k,l}^{(t+1)})^2 &=& (1 - \alpha({\bx}_{k}^{(t)})^T\bC{\bx}_{k}^{(t)})^2 ({z}_{k,l}^{(t)} )^2.
\end{eqnarray*}
Let $\min \{(1 - 3\alpha\lambda_1)^2\lambda_m, (\tilde{\bx}_k^{(0)})^T\bC\tilde{\bx}_k^{(0)} \} = \eta$. Since $0 <\eta < ({\bx}_{k}^{(t)})^T\bC{\bx}_{k}^{(t)} < \frac{1}{\alpha}$ (from \eqref{eq:rayleigh_upperboundk} and \eqref{eq:rayleigh_lowerboundk}), we have $0 < 1 - \alpha({\bx}_{k}^{(t)})^T\bC{\bx}_{k}^{(t)} < 1 - \alpha\eta  < 1$. 
Therefore,
\begin{eqnarray}
\sum_{l=1}^{k-1}({z}_{k,l}^{(t+1)})^2 < \sum_{l=1}^{k-1}\gamma({z}_{k,l}^{(t)})^2 < \gamma^{t+1}\sum_{l=1}^{k-1}({z}_{k,l}^{(0)})^2 = a_1\gamma^{t+1}, \quad \text{where} \quad \gamma = (1-\alpha\eta)^2.
\end{eqnarray} 
\end{proof}

\section{Statement and Proof of Lemma~\ref{lemma:coeff_decay_upper}}\label{app:lemma4}
\begin{lemma}\label{lemma:coeff_decay_upper}
	Suppose $z_{k,k}^{(0)} \neq 0$ and $(\bx_k^{(t)})^T\bC\bx_k^{(t)}  < \frac{1}{\alpha}$. Then the following is true for $\rho_k = \Big(\frac{1 + \alpha \lambda_{k+1}}{1 + \alpha\lambda_k }\Big)^2 < 1 $ and some constant $a_2 > 0$:
	\begin{equation}
	\sum_{l=k+1}^{d}({z}_{k,l}^{(t+1)})^2 \leq a_2\rho_k^{t+1}.
	\end{equation}
\end{lemma}
\begin{proof}
For $l = k,\ldots, d$ we know from \eqref{eq:eqn_z_upper} that ${z}_{k,l}^{(t+1)} = (1 + \alpha(\lambda_l-({\bx}_{k}^{(t)})^T\bC{\bx}_{k}^{(t)})){z}_{k,l}^{(t)}$. If $ ({\bx}_{k}^{(t)})^T\bC{\bx}_{k}^{(t)} < \frac{1}{\alpha}$, we have $1 + \alpha(\lambda_l - ({\bx}_{k}^{(t)})^T\bC{\bx}_{k}^{(t)}) >  \alpha\lambda_l \geq 0, \forall l = k,\ldots, d$.

Thus, we have for $l = k+1, \cdots d$,
\begin{eqnarray*}
	\Bigg(\frac{{z}_{k,l}^{(t+1)}}{{z}_{k,k}^{(t+1)}}\Bigg)^2 &=& \Bigg(\frac{1 + \alpha (\lambda_l - ({\bx}_{k}^{(t)})^T\bC{\bx}_{k}^{(t)})}{1 +\alpha(\lambda_k - ({\bx}_{k}^{(t)})^T\bC{\bx}_{k}^{(t)})}\Bigg)^2\Bigg(\frac{{z}_{k,l}^{(t)}}{{z}_{k,k}^{(t)}}\Bigg)^2 \\
	&=&  \Bigg(1 - \frac{\alpha (\lambda_k-\lambda_l)}{1 + \alpha (\lambda_k  - ({\bx}_{k}^{(t)})^T\bC{\bx}_{k}^{(t)})}\Bigg)^2\Bigg(\frac{{z}_{k,l}^{(t)}}{{z}_{k,k}^{(t)}}\Bigg)^2 \\
	&\leq&  \Big(1 - \frac{\alpha (\lambda_k-\lambda_l)}{1 +\alpha\lambda_k }\Big)^2\Big(\frac{{z}_{k,l}^{(t)}}{{z}_{k,k}^{(t)}}\Big)^2 \\
	&=&  \Big(\frac{1 + \alpha \lambda_l }{1 + \alpha\lambda_k }\Big)^2\Big(\frac{{z}_{k,l}^{(t)}}{{z}_{k,k}^{(t)}} \Big)^2 \leq  \Big(\frac{1 + \alpha \lambda_{k+1}}{1 + \alpha\lambda_k }\Big)^2\Big(\frac{{z}_{k,l}^{(t)}}{{z}_{k,k}^{(t)}} \Big)^2 \\
	&=& \rho_k\Big(\frac{{z}_{k,l}^{(t)}}{{z}_{k,k}^{(t)}} \Big)^2 , \quad \rho_k  =  \Big(\frac{1 + \alpha \lambda_{k+1}}{1 + \alpha\lambda_k }\Big)^2 < 1.
\end{eqnarray*}
Therefore, for $l=k+1,\ldots,d$, $({z}_{k,l}^{(t+1)})^2 \leq \rho_k^{t+1}\Big(\frac{{z}_{k,l}^{(0)}}{{z}_{k,k}^{(0)}} \Big)^2({z}_{k,k}^{(t+1)})^2$. 
Since $\|{\bx}_{k}^{(t+1)}\|^2 \leq 3$ and $\|{\bx}_k^{(0)}\| = 1$, hence $({z}_{k,k}^{t+1})^2 \leq 3$ and ${z}_{k,l}^{(0)} \leq 1$. Also, because of the assumption ${z}_{k,k}^{(0)} \neq 0$, let us assume $({z}_{k,k}^{(0)})^2 > \tilde{\eta}$. Thus, we can write
\begin{eqnarray}
	\sum_{l=k+1}^{d}({z}_{k,l}^{t+1})^2 &\leq& \rho_k^{t+1}\sum_{l=k+1}^{d}\frac{3}{\tilde{\eta}} = a_2\rho_k^{t+1}.
\end{eqnarray} 
\end{proof}

\section{Statement and Proof of Lemma~\ref{lemma:monotonic_rayleigh_quotientk}}\label{app:lemma5}
\begin{lemma} \label{lemma:monotonic_rayleigh_quotientk}
	Suppose $ (\bx_k^{(t)})^T\bC\bx_k^{(t)}  < \frac{1}{\alpha}$ and $(\bx_k^{(t)})^T\bC\bx_k^{(t)} > \min \{(1 - 3\alpha\lambda_1)^2\lambda_m, (\tilde{\bx}_k^{(0)})^T\bC\tilde{\bx}_k^{(0)}\}$. Then there exists constants $0<\delta, \gamma_1 <1, a_4 > 0$ such that
	\begin{equation}
		|\lambda_k - ({\bx}_{k}^{(t+1)})^T\bC{\bx}_{k}^{(t+1)} | \leq ta_4(\delta^{t+1} + \max\{\delta^t, \gamma_1^t\}).
	\end{equation}
\end{lemma}
\begin{proof}
\begin{align*}
	({\bx}_{k}^{(t+1)})^T\bC{\bx}_{k}^{(t+1)} &= \sum_{l=1}^{k-1}\lambda_l(1 - \alpha(\bx_{k}^{(t)})^T\bC\bx_{k}^{(t)})^2({z}_{k,l}^{(t)})^2 + \sum_{l=k}^{d}\lambda_l(1 + \alpha(\lambda_l - (\bx_{k}^{(t)})^T\bC\bx_{k}^{(t)}))^2({z}_{k,l}^{(t)})^2\\
	 &= \sum_{l=1}^{k-1}\lambda_l(1 + \alpha(\lambda_k - (\bx_{k}^{(t)})^T\bC\bx_{k}^{(t)}))^2({z}_{k,l}^{(t)})^2 + \sum_{l=k}^{d}\lambda_l(1 + \alpha(\lambda_k - (\bx_{k}^{(t)})^T\bC\bx_{k}^{(t)}))^2({z}_{k,l}^{(t)})^2 \\
	& + \sum_{l=1}^{k-1}\lambda_l(1 - \alpha(\bx_{k}^{(t)})^T\bC\bx_{k}^{(t)})^2({z}_{k,l}^{(t)})^2 -  \sum_{l=1}^{k-1}\lambda_l(1 + \alpha(\lambda_k - (\bx_{k}^{(t)})^T\bC\bx_{k}^{(t)}))^2({z}_{k,l}^{(t)})^2 \\
	& + \sum_{l=k+1}^{d}\lambda_l(1 + \alpha(\lambda_l - (\bx_{k}^{(t)})^T\bC\bx_{k}^{(t)}))^2({z}_{k,l}^{(t)})^2 - \sum_{l=k+1}^{d}\lambda_l(1 + \alpha(\lambda_k - (\bx_{k}^{(t)})^T\bC\bx_{k}^{(t)}))^2({z}_{k,l}^{(t)})^2 \\
	&= \sum_{l=1}^{d}\lambda_l(1 + \alpha(\lambda_k - (\bx_{k}^{(t)})^T\bC\bx_{k}^{(t)}))^2({z}_{k,l}^{(t)})^2  + P^{(t)}\\
	&= (1 + \alpha(\lambda_k - (\bx_{k}^{(t)})^T\bC\bx_{k}^{(t)}))^2({\bx}_{k}^{(t)})^T\bC{\bx}_{k}^{(t)} + P^{(t)},
\end{align*}
where
\begin{align*}
	P^{(t)} &= \sum_{l=1}^{k-1}\lambda_l(1 - \alpha(\bx_{k}^{(t)})^T\bC\bx_{k}^{(t)})^2({z}_{k,l}^{(t)})^2 -  \sum_{l=1}^{k-1}\lambda_l(1 + \alpha(\lambda_k - (\bx_{k}^{(t)})^T\bC\bx_{k}^{(t)}))^2({z}_{k,l}^{(t)})^2 \\
	& + \sum_{l=k+1}^{d}\lambda_l(1 + \alpha(\lambda_l - (\bx_{k}^{(t)})^T\bC\bx_{k}^{(t)}))^2({z}_{k,l}^{(t)})^2 - \sum_{l=k+1}^{d}\lambda_l(1 + \alpha(\lambda_k - (\bx_{k}^{(t)})^T\bC\bx_{k}^{(t)}))^2({z}_{k,l}^{(t)})^2.
\end{align*}
Now,
\begin{align*}
	\lambda_k - ({\bx}_{k}^{(t+1)})^T\bC{\bx}_{k}^{(t+1)} &= \lambda_k - (1 + \alpha(\lambda_k - (\bx_{k}^{(t)})^T\bC\bx_{k}^{(t)}))^2({\bx}_{k}^{(t)})^T\bC{\bx}_{k}^{(t)} - P^{(t)}\\
	&= \lambda_k - (1  +\alpha^2(\lambda_k - (\bx_{k}^{(t)})^T\bC\bx_{k}^{(t)})^2 + 2\alpha(\lambda_k - (\bx_{k}^{(t)})^T\bC\bx_{k}^{(t)}))({\bx}_{k}^{(t)})^T\bC{\bx}_{k}^{(t)} - P^{(t)} \\
	&= \lambda_k - ({\bx}_{k}^{(t)})^T\bC{\bx}_{k}^{(t)} - (\alpha^2(\lambda_k - (\bx_{k}^{(t)})^T\bC\bx_{k}^{(t)})^2 +  2\alpha(\lambda_k - (\bx_{k}^{(t)})^T\bC\bx_{k}^{(t)}))({\bx}_{k}^{(t)})^T\bC{\bx}_{k}^{(t)} - P^{(t)} \\
	&= \lambda_k - ({\bx}_{k}^{(t)})^T\bC{\bx}_{k}^{(t)} - (\lambda_k - (\bx_{k}^{(t)})^T\bC\bx_{k}^{(t)})(\alpha^2(\lambda_k - (\bx_{k}^{(t)})^T\bC\bx_{k}^{(t)}) + 2\alpha)({\bx}_{k}^{(t)})^T\bC{\bx}_{k}^{(t)} - P^{(t)} \\
	&= (\lambda_k - ({\bx}_{k}^{(t)})^T\bC{\bx}_{k}^{(t)} )(1-\alpha^2(\lambda_k - (\bx_{k}^{(t)})^T\bC\bx_{k}^{(t)})({\bx}_{k}^{(t)})^T\bC{\bx}_{k}^{(t)} - 2\alpha({\bx}_{k}^{(t)})^T\bC{\bx}_{k}^{(t)})	- P^{(t)}.
\end{align*}
Let us denote $V^{(t)} = |\lambda_k - ({\bx}_{k}^{(t)})^T\bC{\bx}_{k}^{(t)}|$. Then,
\begin{eqnarray*}
	V^{(t+1)} &\leq& V^{(t)}|1-\alpha^2(\lambda_k - (\bx_{k}^{(t)})^T\bC\bx_{k}^{(t)})({\bx}_{k}^{(t)})^T\bC{\bx}_{k}^{(t)} - 2\alpha({\bx}_{k}^{(t)})^T\bC{\bx}_{k}^{(t)}| +|P^{(t)}| \\
	&\leq& V^{(t)}\max \{(1 - \alpha({\bx}_{k}^{(t)})^T\bC{\bx}_{k}^{(t)})^2, \alpha^2\lambda_k({\bx}_{k}^{(t)})^T\bC{\bx}_{k}^{(t)}\} + |P^{(t)}|.
\end{eqnarray*}
Also from~\eqref{eq:rayleigh_upperboundk} and \eqref{eq:rayleigh_lowerboundk}, $	0< \alpha\eta < \alpha({\bx}_{k}^{(t)})^T\bC{\bx}_{k}^{(t)} < 1 \text{ and } \alpha^2\lambda_k({\bx}_{k}^{(t)})^T\bC{\bx}_{k}^{(t)} < \alpha\lambda_k$.
Denote $\delta = \max\{(1-\alpha\eta)^2, \alpha\lambda_k \}$. Since $  \alpha\lambda_k < \alpha\lambda_1 < 1$, hence $0 < \delta < 1$. Thus,
\begin{equation}\label{bound_V}
	V^{(t+1)} \leq \delta V^{(t)}  + |P^{(t)}|.
\end{equation}
Next, we bound $ |P^{(t)}|$ as follows:
\begin{align*}
	|P^{(t)}| &= |\sum_{l=1}^{k-1}\lambda_l((1 - \alpha(\bx_{k}^{(t)})^T\bC\bx_{k}^{(t)})^2 - (1 + \alpha(\lambda_k - (\bx_{k}^{(t)})^T\bC\bx_{k}^{(t)}))^2)({z}_{k,l}^{(t)})^2 \\
	& + \sum_{l=k+1}^{d}\lambda_l((1 + \alpha(\lambda_l - (\bx_{k}^{(t)})^T\bC\bx_{k}^{(t)}))^2 - (1 + \alpha(\lambda_k - (\bx_{k}^{(t)})^T\bC\bx_{k}^{(t)}))^2)({z}_{k,l}^{(t)})^2|\\
	&= |\sum_{l=1}^{k-1}\lambda_l(-\alpha\lambda_k)(2 +\alpha\lambda_k - 2\alpha(\bx_{k}^{(t)})^T\bC\bx_{k}^{(t)})({z}_{k,l}^{(t)})^2 \\
	& + \sum_{l=k+1}^{d}\lambda_l\alpha(\lambda_l-\lambda_k)(2 + \alpha(\lambda_k + \lambda_l) - 2\alpha(\bx_{k}^{(t)})^T\bC\bx_{k}^{(t)})({z}_{k,l}^{(t)})^2|\\
	&\leq \sum_{l=1}^{k-1}\lambda_l|(-\alpha\lambda_k)(2 +\alpha\lambda_k - 2\alpha(\bx_{k}^{(t)})^T\bC\bx_{k}^{(t)})({z}_{k,l}^{(t)})^2| \\
	& + \sum_{l=k+1}^{d}\lambda_l|\alpha(\lambda_l-\lambda_k)(2 + \alpha(\lambda_k + \lambda_l) - 2\alpha(\bx_{k}^{(t)})^T\bC\bx_{k}^{(t)})({z}_{k,l}^{(t)})^2|\\
	&\leq \sum_{l=1}^{k-1}\lambda_l\alpha\lambda_k(2 +\alpha\lambda_k)({z}_{k,l}^{(t)})^2  + \sum_{l=k+1}^{d}\lambda_l\alpha(\lambda_k-\lambda_l)(2 + \alpha(\lambda_k + \lambda_l))({z}_{k,l}^{(t)})^2\\
	&< \sum_{l=1}^{k-1}\lambda_l\alpha\lambda_k(2 +\alpha\lambda_k)({z}_{k,l}^{(t)})^2  + \sum_{l=k+1}^{d}\lambda_l(2\alpha\lambda_k + \alpha^2\lambda_k^2)({z}_{k,l}^{(t)})^2\\
	&< \sum_{l=1}^{k-1}3\lambda_1({z}_{k,l}^{(t)})^2  + \sum_{l=k+1}^{d}3\lambda_1({z}_{k,l}^{(t)})^2, \quad \text{since $\alpha\lambda_k < 1$ and $\lambda_l < \lambda_1$}\\
	&= 3\lambda_1(\sum_{l=1}^{k-1}({z}_{k,l}^{(t)})^2 + \sum_{l=k+1}^{d}({z}_{k,l}^{(t)})^2) < 3\lambda_1(a_1\gamma^{t} + a_2\rho_k^{t}) \quad \text{using Lemma~\ref{lemma:coeff_decay_lower} and ~\ref{lemma:coeff_decay_upper}} \\
	&\leq a_3\gamma_1^t, \quad \text{where} \quad a_3 = \max\{3\lambda_1a_1, 3\lambda_1a_2\} \quad \text{and} \quad \gamma_1 = \max\{\gamma, \rho_k\}.
\end{align*}
So from \eqref{bound_V}, we have $V^{(t+1)} \leq \delta V^{(t)} + a_3\gamma_1^t \leq \delta^{t+1}V^{(0)} + a_3\sum_{r=0}^{t}(\delta\gamma_1^{-1})^r\gamma_1^{t}$.
Since $\gamma_1, \delta < 1$, we have the following two cases:
\begin{enumerate}
	\item $\delta \leq \gamma_1\implies \delta\gamma_1^{-1} \leq 1$. Then, $\sum_{r=0}^{t}(\delta\gamma_1^{-1})^{r}\gamma_1^t \leq \sum_{r=0}^{t}\gamma_1^t = t\gamma_1^t$.
	\item $\delta > \gamma_1 $. Then $\sum_{r=0}^{t}(\delta\gamma_1^{-1})^{r}\gamma_1^t = \gamma_1^t + \delta\gamma_1^{t-1} + \cdots + \delta^t < \delta^t + \cdots + \delta^t = t\delta^t$.
\end{enumerate}
Thus, 
\begin{eqnarray*}
	V^{(t+1)} &\leq& \delta^{t+1} V^{(0)} + ta_3\max\{\delta^t, \gamma_1^t\} \leq ta_4(\delta^{t+1} + \max\{\delta^t, \gamma_1^t\}),
\end{eqnarray*}
where $a_4 = \max\{V^{(0)}, a_3\}$. 
\end{proof}


\section{Statement and Proof of Lemma~\ref{lemma:bounded_mean_deviationk}}\label{app:lemma8}
\begin{lemma}	\label{lemma:bounded_mean_deviationk}
	The deviation of an iterate at a node from the average is bounded from above as
	\begin{equation}
		\|\bx_{i,k}^{(t)} - \bar{\bx}_k^{(t)}\| \leq b_k(\beta^t + \frac{\alpha}{1-\beta}\big), \forall k =1,\ldots,K,
	\end{equation}
	where $\beta$  is the second largest magnitude of the eigenvalues of $\bW$ given as $\beta = \max\{|\lambda_2(\bW)|, |\lambda_M(\bW)|\} < 1$ and $b_k > 0$ is some constant.
\end{lemma}
\begin{proof}
We stack the iterates $\bx_{i,k}^{(t)}$ and $\bcH_i(\bx_{i,k}^{(t)})$ as
\[
\bx_k^{(t) }= \begin{bmatrix}
	\bx_{1,k}^{(t)}\\
	\bx_{2,k}^{(t)}\\
	\vdots\\
	\bx_{M,k}^{(t)}\\
\end{bmatrix} \in \R^{Md}\quad
\bcH(\bx_k^{(t)})  = \begin{bmatrix}
	\bcH_1(\bx_{1,k}^{(t)})\\
	\bcH_2(\bx_{2,k}^{(t)})\\
	\vdots\\
	\bcH_M(\bx_{M,k}^{(t)})\\
\end{bmatrix} \in \R^{Md} \quad
{\bx}_{avg,k}^{(t)} = \begin{bmatrix}
	\bar{\bx}_k^{(t)} \\
	\bar{\bx}_k^{(t)} \\
	\vdots\\
	\bar{\bx}_k^{(t)} \\
\end{bmatrix} \in \R^{Md}.
\]
The next network-wide iterate (as a stacked vector) can then be written as $\bx_k^{(t)} = (\bW\otimes\bI)\bx_k^{(t-1)} + \alpha \bcH(\bx_k^{(t-1)})$,
where $\otimes$ denotes the Kronecker product. The $t^{th}$ iterate can thus be written as
\begin{equation*}
	\bx_k^{(t)} = (\bW^t\otimes\bI)\bx_k^{(0)} + \alpha\sum_{s=0}^{t-1}(\bW^{t-1-s}\otimes\bI)\bcH(\bx_k^{(s)}).
\end{equation*}
Since $\bW = [w_{ij}]$ is a symmetric and doubly stochastic mixing matrix, its largest eigenvalue is $1$ corresponding to the eigenvector $\bone_M$, a column vector of all 1's. It is also the left eigenvector of $\bW$. That is, $\bW\bone_M = \bone_M \text{ and } \bone_M^T \bW = \bone_M^T$.
Also, since the squared norm of Sanger's direction at every node is bounded, it is easy to see $\|\bcH(\bx_k^{(t)})\|^2 = 3M\lambda_{1}^2(3k-2)(3k+1)$.
Now,
\begin{align*}
	&\|\bx_{i,k}^{(t)} - \bar{\bx}_k^{(t)}\| \leq \|\bx_k^{(t)} - {\bx}_{avg,k}^{(t)}  \|  =  \|\bx_k^{(t)}  - \frac{1}{M}((\bone_M\bone_M^{T})\otimes\bI)\bx_{k}^{(t)} \| \\
	&= \| (\bW^t\otimes\bI)\bx_k^{(0)} + \alpha \sum_{s=0}^{t-1}(\bW^{t-1-s}\otimes\bI)\bcH(\bx_k^{(s)}) - \frac{1}{M}((\bone_M\bone_M^{T})\otimes\bI)((\bW^t\otimes\bI) \bx_k^{(0)} + \alpha\sum_{s=0}^{t-1}(\bW^{t-1-s}\otimes\bI)\bcH(\bx_k^{(s)}))\| \\
	&= \| (\bW^t\otimes\bI)\bx_k^{(0)} + \alpha\sum_{s=0}^{t-1}(\bW^{t-1-s}\otimes\bI)\bcH(\bx_k^{(s)}) - \frac{1}{M}((\bone_M\bone_M^{T})\otimes\bI)\bx_k^{(0)} - \alpha\sum_{s=0}^{t-1}(\frac{1}{M}((\bone_M\bone_M^{T})\otimes\bI)\bcH(\bx_k^{(s)}))\| \\
	&= \| ((\bW^t - \frac{1}{M}(\bone_M\bone_M^{T}))\otimes\bI) \bx_k^{(0)} + \alpha\sum_{s=0}^{t-1}((\bW^{t-1-s} - \frac{1}{M}(\bone_M\bone_M^{T}))\otimes\bI)\bcH(\bx_k^{(s)})\|\\
	&\leq \|((\bW^t - \frac{1}{M}(\bone_M\bone_M^{T}))\otimes\bI) \bx_k^{(0)}\| + \alpha\|\sum_{s=0}^{t-1}((\bW^{t-1-s} - \frac{1}{M}(\bone_M\bone_M^{T}))\otimes\bI)\bcH(\bx_k^{(s)})\| \\
	&\leq \|((\bW^t - \frac{1}{M}(\bone_M\bone_M^{T}))\otimes\bI)\|\| \bx_k^{(0)}\| + \alpha\sum_{s=0}^{t-1}\|((\bW^{t-1-s} - \frac{1}{M}(\bone_M\bone_M^{T}))\otimes\bI)\|\|\bcH(\bx_k^{(s)})\| \\
	&= \beta^t\| \bx_k^{(0)}\|  + \alpha\sum_{s=0}^{t-1}\beta^{t-1-s}\|\bcH(\bx_k^{(s)})\| \leq \beta^t\sqrt{3M} + \alpha\sqrt{3M\lambda_{1}^2(3k-2)(3k+1)}\sum_{s=0}^{t-1}\beta^{t-1-s}\\
	&\leq \beta^t\sqrt{3M} + \frac{\alpha\sqrt{3M\lambda_{1}^2(3k-2)(3k+1)}}{1-\beta}\\
	&\leq \sqrt{3M}\lambda_1\sqrt{(3k-2)(3k+1)}\big(\beta^t + \frac{\alpha}{1-\beta}\big) = b_k\big(\beta^t + \frac{\alpha}{1-\beta}\big), \quad \text{where} \quad b_k =  \lambda_1\sqrt{3M}\sqrt{(3k-2)(3k+1)}.
\end{align*}
\end{proof}

\section{Statement and Proof of Lemma~\ref{lemma:bound_hk}}\label{app:lemma10}
\begin{lemma}\label{lemma:bound_hk}
Suppose $\|\bx_{i,k}^{(t)}\|^2 \leq 3$ and $\|\bx_{i,k}^{(t)} - \bar{\bx}_k^{(t)}\| \leq b_k(\beta^t + \frac{\alpha}{1-\beta}\big)$, then the following is true $\forall k = 1,\ldots, K$:
\begin{align}
    \|\bh_k^{(t)}\| \leq 3(k+2)\lambda_1b_k(\beta^t + \frac{\alpha}{1-\beta}).
\end{align}
\end{lemma}
\begin{proof}
We have
\begin{align*}
	&\bcH_i(\bx_{i,k}^{(t)}) - \bcH_i(\bar{\bx}_k^{(t)}) \\
	&= \bC_i(\bx_{i,k}^{(t)}-\bar{\bx}_{k}^{(t)}) - (\bx_{i,k}^{(t)})^T\bC_i\bx_{i,k}^{(t)}\bx_{i,k}^{(t)}+(\bar{\bx}_{k}^{(t)})^T\bC_i\bar{\bx}_{k}^{(t)}\bar{\bx}_{k}^{(t)} -  \sum_{p=1}^{k-1}({\bx}_{i,p}^{(t)}({\bx}_{i,p}^{(t)})^T\bC_i{\bx}_{i,k}^{(t)} - {\bx}_{i,p}^{(t)}({\bx}_{i,p}^{(t)})^T\bC_i\bar{\bx}_{k}^{(t)})\\
	&= (\bC_i -  (\bx_{i,k}^{(t)})^T\bC_i\bx_{i,k}^{(t)}\bI)(\bx_{i,k}^{(t)}-\bar{\bx}_{k}^{(t)})  - ((\bx_{i,k}^{(t)} + \bar{\bx}_{k}^{(t)})^T\bC_i(\bx_{i,k}^{(t)} - \bar{\bx}_{k}^{(t)}))\bar{\bx}_{k}^{(t)}
 - \sum_{p=1}^{k-1}{\bx}_{i,p}^{(t)}({\bx}_{i,p}^{(t)})^T\bC_i(\bx_{i,k}^{(t)} - \bar{\bx}_k^{(t)})
\end{align*}
\begin{align*}
	&\|\bcH_i(\bx_{i,k}^{(t)}) - \bcH_i(\bar{\bx}_k^{(t)})\| \\
	&\leq \|\bC_i -  (\bx_{i,k}^{(t)})^T\bC_i\bx_{i,k}^{(t)}\bI\|\|\bx_{i,k}^{(t)}-\bar{\bx}_{k}^{(t)}\| +|(\bx_{i,k}^{(t)} + \bar{\bx}_{k}^{(t)})^T\bC_i(\bx_{i,k}^{(t)} - \bar{\bx}_{k}^{(t)})|\|\bar{\bx}_{k}^{(t)}\|  + \sum_{p=1}^{k-1}\|{\bx}_{i,p}^{(t)}({\bx}_{i,p}^{(t)})^T\bC_i(\bx_{i,k}^{(t)} - \bar{\bx}_k^{(t)})\|\\
	&\leq \|\bC_i -  (\bx_{i,k}^{(t)})^T\bC_i\bx_{i,k}^{(t)}\bI\|\|\bx_{i,k}^{(t)}-\bar{\bx}_{k}^{(t)}\| +\|\bx_{i,k}^{(t)} + \bar{\bx}_{k}^{(t)}\|\|\bC_i\|\|\bx_{i,k}^{(t)} - \bar{\bx}_{k}^{(t)}\|\|\bar{\bx}_{k}^{(t)}\| + 
\sum_{p=1}^{k-1}\|{\bx}_{i,p}^{(t)}({\bx}_{i,p}^{(t)})^T\bC_i\|\|\bx_{i,k}^{(t)} - \bar{\bx}_k^{(t)}\|\\
	&\leq  \|\bC_i -  (\bx_{i,k}^{(t)})^T\bC_i\bx_{i,k}^{(t)}\bI\|\|\bx_{i,k}^{(t)}-\bar{\bx}_{k}^{(t)}\| +\|\bar{\bx}_{k}^{(t)}\|(\|\bx_{i,k}^{(t)}\| + \|\bar{\bx}_{k}^{(t)}\|)\|\bC_i\|\|\bx_{i,k}^{(t)} - \bar{\bx}_{k}^{(t)}\| +
 \sum_{p=1}^{k-1}\|{\bx}_{i,p}^{(t)}({\bx}_{i,p}^{(t)})^T\|\|\bC_i\|\|\bx_{i,k}^{(t)} - \bar{\bx}_k^{(t)}\|\\
	&\leq 3\lambda_1\|\bx_{i,k}^{(t)}-\bar{\bx}_{k}^{(t)}\| +  \sqrt{3}(2\sqrt{3})\lambda_{1}\|\bx_{i,k}^{(t)}-\bar{\bx}_{k}^{(t)}\| + \sum_{p=1}^{k-1}3\lambda_{1}\|\bx_{i,k}^{(t)} - \bar{\bx}_k^{(t)}\|\\
	&= 3(k+2)\lambda_1\|\bx_{i,k}^{(t)}-\bar{\bx}_{k}^{(t)}\|  \leq 3(k+2)\lambda_1b_k(\beta^t + \frac{\alpha}{1-\beta}).
\end{align*}
Thus,
\begin{align}\label{eq:norm_hk}
	\|\bh_k^{(t)}\| &\leq \frac{1}{M}\sum_{i=1}^{M}\|\bcH_i(\bx_{i,k}^{(t)}) - \bcH_i(\bar{\bx}_k^{(t)})\| \leq 3(k+2)\lambda_1b_k(\beta^t + \frac{\alpha}{1-\beta}).
\end{align} 
\end{proof}


\end{appendices}


\end{document}

%% file: _defs.tex
\newcommand{\ba}        {\mathbf{a}}
\newcommand{\bA}        {\mathbf{A}}

\newcommand{\bC}        {\mathbf{C}}

\newcommand{\E}         {\mathbb{E}}

\newcommand{\cE}        {\mathcal{E}}

\newcommand{\cG}        {\mathcal{G}}

\newcommand{\bh}        {\mathbf{h}}

\newcommand{\cH}        {\mathcal{H}}
\newcommand{\bcH}       {\boldsymbol{\cH}}

\newcommand{\bI}        {\mathbf{I}}

\newcommand{\cN}        {\mathcal{N}}

\newcommand{\cO}        {\mathcal{O}}

\newcommand{\bq}        {\mathbf{q}}
\newcommand{\bQ}        {\mathbf{Q}}

\newcommand{\R}         {\mathbb{R}}

\newcommand{\cU}        {\mathcal{U}}
\newcommand{\bcU}       {\boldsymbol{\cU}}

\newcommand{\cV}        {\mathcal{V}}

\newcommand{\bW}        {\mathbf{W}}

\newcommand{\bx}        {\mathbf{x}}
\newcommand{\bX}        {\mathbf{X}}

\newcommand{\by}        {\mathbf{y}}
\newcommand{\bY}        {\mathbf{Y}}

\newcommand{\bphi}      {\boldsymbol{\phi}}

\newcommand{\bpsi}      {\boldsymbol{\psi}}

\newcommand{\bSigma}    {\boldsymbol{\Sigma}}

\newcommand{\bzero}     {\mathbf{0}}
\newcommand{\bone}      {\mathbf{1}}

\DeclareMathOperator*{\argmin}  {arg\,min}

\newcommand{\tT}        {\mathrm{T}}



%% file: Introduction.tex
\section{Introduction}\label{sec:intro}
The modern era of machine learning involves leveraging \textit{massive} amounts of \textit{high-dimensional} data, which can have large computational and storage costs. To combat the complexities arising because of the high dimensions of data, dimensionality reduction and feature learning techniques play a pivotal and necessary role in information processing. The most common and widely used technique for this task is Principal Component Analysis (PCA)~\cite{Hotelling.1933} which, in the simplest of terms, transforms data into uncorrelated features that aid conversion of data from a high-dimensional space to a low-dimensional space while retaining maximum information. Simultaneously, the enormity of the amount of available data makes it difficult to manage it at a single location. There are multiple and an increasing number of scenarios where data is distributed across different locations, either due to storage constraints or by its inherent nature like in the Internet-of-Things~\cite{NoklebyRajaEtAl.PI20}. This aspect of the modern-world data have led researchers to explore distributed algorithms, which can process information across different locations/machines~\cite{BajwaCevherEtAl.ISPM20}. These aforementioned issues have motivated us to study and develop algorithms for distributed PCA that are efficient in terms of computations and communications among multiple machines, and that can also be proven to converge at a fast rate.

When the data is available at a single location, one of the goals of PCA is to find a $K$-dimensional subspace, given by the column space of a matrix $\bX \in \R^{d\times K}$, such that the zero-mean data samples $\by \in \R^{d}~(d \gg K)$ retain maximum information when projected onto $\bX$. In other words, when reconstructed as $\bX\bX^{\tT}\by$ (subject to $\bX^{\tT}\bX = \bI$), the data samples should have minimum reconstruction error. It can be shown that this minimal error solution is given by the projection of data onto the subspace spanned by the eigenvectors of data covariance matrix. This implies that for dimensionality reduction, learning any basis of that subspace is sufficient. This is referred to as the \emph{subspace learning} problem. But while simple dimension reduction does not necessarily need uncorrelated features, most downstream machine learning tasks like classification, pattern matching, regression, etc., work more efficiently when the data features are uncorrelated. In the case of image coding, e.g., PCA is known as the Karhunen--Loeve transform~\cite{Loeve.1963}, wherein images are compressed by decorrelating neighbouring pixels. With this goal in mind, one needs to aim to find the specific directions that not only have maximum variance, but that also lead to uncorrelated features when data is projected onto those directions. These specific directions are given by the eigenvectors (also called the principal components) of the data covariance matrix, and not just any set of orthogonal basis vectors spanning the same space. Mathematically, along with minimum reconstruction error, the other goal of PCA is to ensure the condition that the off-diagonal entries of $\E[\bX^{\tT}\by\by^{\tT}\bX]$ are zero (i.e., the data gets decorrelated), while finding the eigenspace of the covariance matrix $\E[\by\by^{\tT}]$.

As explained above, the true and complete purpose of PCA is served when the search for the optimal solution ends with the specific set of eigenvectors of the data covariance matrix, and not just with the subspace it spans. It is known that getting the principal components from any other basis of the subspace would only require performing singular value decomposition (SVD) of the obtained subspace. Although true, the SVD operation has a high computational complexity, which makes it an expensive step for big data. The traditional solutions for PCA were developed to overcome the cost of SVD and hence reverting back to it defeats the whole purpose.

Thus, even though the problem of dimensionality reduction of data has many optimal solutions (corresponding to all the sets of basis vectors spanning the $K$-dimensional space), our goal is to find only the ones that give the eigenvectors as the basis. In terms of optimization geometry of the PCA problem in which one tries to minimize the mean-squared reconstruction error under an orthogonality constraint, it is a non-convex strict-saddle function. In a strict-saddle function, all the stationary points except the local minima are strict saddles wherein the Hessians have at least one negative eigenvalue that helps in escaping these saddle points~\cite{DixitBajwa.arxiv20,DixitBajwa.arxiv21}. Also, in the case of PCA the local minima are the same as the global minima. These geometric aspects make PCA, despite being non-convex, a ``nice and solvable" problem whose optimal solution can be reached efficiently. However, note that the set of global minima contains, along with the set of eigenvectors as basis, all other possible bases that are rotated with respect to the eigenvectors. And our goal is not to find just any of the global minima but to look into a very particular subset of it, where the basis is not rotated. 

A very popular tool that has been used to learn features of data, and hence compress it, is autoencoders. It was shown in~\cite{baldi.hornik.1989} that the globally optimum weights of an autoencoder for minimum reconstruction error are the principal components of the covariance matrix of the input data. In~\cite{oja}, Oja described how using the Hebbian rule for updating the weights of a linear neural network would extract the first principal component from the input data. Several other Hebbian-based rules like Rubner's model, APEX model~\cite{PCNN.1996}, Generalized Hebbian Algorithm (GHA)~\cite{sanger}, etc., were proposed to extend this idea of training a neural network for finding the first eigenvector to extract the first $K$ principal components (eigenvectors) of the input covariance matrix. Given the parallelization potential, a feedforward linear neural network-based solution for PCA seems to be very attractive.

The other aspect of modern day data is, as mentioned earlier, its massive size. Collating the huge amount of raw data is usually prohibitive due to communications overhead and/or privacy concerns. These reasons have encouraged researchers over the last couple of decades to develop algorithms that can solve various problems for non-collocated data. The algorithms developed to deal with such scenarios can be broadly classified into two categories: (1) the setups where a central entity/server is required to co-ordinate among the various data centers to yield the final result, and (2) the setup where the data is scattered over an arbitrary network of interconnected data centers with no availability of a central co-ordinating node. The authors in \cite{Yang.Gang.Bajwa.2020,NoklebyRajaEtAl.PI20} talk about these different setups and the algorithms developed for each of them in more detail. The second scenario is more generic and usually algorithms developed for such setups can be easily modified to be applied to the first scenario. The terms \textit{distributed} and \textit{decentralized} are used interchangeably for both the setups in the literature. In this paper, we focus on the latter scenario of arbitrarily connected networks and henceforth call it \textit{distributed} setup. Hence, here our goal is to solve the PCA problem in the distributed manner when data is scattered over a network of interconnected nodes such that all the nodes in the network eventually agree with each other and converge to the true principal components of the distributed data.

\subsection{Relation to Prior Work}\label{subsec:ExistingWork}
PCA was developed to find simpler models of smaller dimensions that can approximately fit some data. Some seminal work was done by Pearson~\cite{Pearson.1901}, who aimed at fitting a line to a set of points, and by Hotelling~\cite{Hotelling.1933}, where a method for the classical PCA problem of decorrelating the features of a given set of data points (observations) by finding the principal components was proposed. Later, some fast iterative methods like the power method, Lanczos algorithm, and orthogonal iterations~\cite{Golub} were proposed, which were proved to converge to the eigenvectors at a linear rate in the case of symmetric matrices. Many other iterative methods have been proposed over the last few decades that are based on the well-known Hebbian learning rule~\cite{Hebb.1949} like Oja's method~\cite{oja}, generalized Hebbian algorithm~\cite{sanger}, APEX~\cite{PCNN.1996}, etc. The analysis for Oja's algorithm has been provided in~\cite{yi.tan2005}, which shows that in the deterministic setting the convergence to the first eigenvector is guaranteed at a linear rate for some conditions on the step size and initial estimates. The work in~\cite{lv.yi.tan.2007} extended the analysis to the generalized Hebbian case for convergence to the first $K$ principal eigenvectors for a specific choice of step size. 

While ways to solve PCA in the centralized case when data is available at a single location have been around for nearly a century, distributed solutions are very recent. Within our distributed setup where we assume a network of arbitrarily connected nodes with no central server, the data distribution can be broadly classified into two types: (1) distribution by features, and (2) distribution by samples. The PCA algorithms for these two different kinds of distribution are significantly different. While both are completely distributed, the first kind~\cite{mcsherry,scaglione.krim.2008,d-oja,NoklebyBajwa.ConfGlobalSIP13} involves estimating only one (or a subset) of feature(s) at each node. In this paper, we focus on finding the eigenvectors when the distribution is by samples, which requires estimation of the whole set of eigenvectors at each node of the network. For this type of distribution, power method was adapted for the distributed setup as a subroutine in~\cite{cksvd.allerton.2013,depm,cksvd} to extract the first principal component of the global covariance matrix. Such methods make use of an explicit consensus loop~\cite{consensus} after each power iteration to ensure that the nodes (approximately) agree with each other. While a novel approach that reaches the required solution at the nodes accurately (albeit with a small error due to the consensus iterations), the two time-scale aspect makes it a relatively slow algorithm in terms of communications efficiency. Furthermore, finding multiple principal components with these approaches would require a sequential approach where subsequent components are determined by using a covariance matrix residue that is left after projection on estimates of the higher-order components. In contrast, the work in~\cite{raja.bajwa.2020} focuses on finding the top eigenvector in the distributed setup for the streaming data case. A detailed review of various distributed PCA algorithms that exist for different setups is provided in~\cite{wu2018review}. 

Next, note that some distributed optimization-based algorithms for non-convex problems are being studied only since recently and those dealing with constrained problems are even fewer. In~\cite{proxpda}, it is shown that an \textit{unconstrained} non-convex problem converges to a stationary solution at a sublinear rate. The methods proposed in~\cite{bianchi.jakubowicz.2013, wai.scaglione.lafond.2016} deal with non-convex objective functions in a distributed setup when the constraint set is convex and~\cite{next} works with convex approximations of non-convex problems. Thus, none of these methods are directly applicable to the distributed PCA problem in our setup. 

Finally, we proposed an efficient distributed PCA solution in~\cite{gang.raja.bajwa.2019} for a distributed network when the data is split sample-wise among the interconnected nodes. In this paper, we extend the preliminary work in~\cite{gang.raja.bajwa.2019} and provide a detailed mathematical analysis of the proposed algorithm along with exact convergence rates and extensive numerical experiments.

\subsection{Our Contributions}\label{subsec:contribution}
The main contributions of this paper are (1) a novel algorithm for distributed PCA, (2) theoretical guarantees for the proposed distributed algorithm with a linear convergence rate to a small neighborhood of the true PCA solution, and (3) experimental results to further demonstrate the efficacy of the proposed algorithm.

Our focus in this paper is to solve the distributed PCA problem so as to find a solution that not only enables dimensionality reduction, but that also provides uncorrelated features of data distributed over a network. That is, our goal is to estimate the true eigenvectors, not just any subspace spanned by them, of the covariance matrix of the data that is distributed across an arbitrarily connected network. Also, we focus on providing a solution that is efficient in terms of communications between the interconnected nodes of an arbitary network. To that end, we propose a distributed algorithm that is based on the generalized Hebbian algorithm (GHA) proposed by Sanger~\cite{sanger}, wherein the nodes perform local computations along with information exchange with their directly connected neighbors, similar to the idea followed in the distributed gradient descent (DGD) approach in~\cite{dgd}. The local computations do not involve the calculation of any gradient, but we instead use a ``psuedo gradient," which we henceforth call \textit{Sanger's direction}. In our proposed solution, termed the \textit{Distributed Sanger's Algorithm (DSA)}, we have also done away with the need of explicit consensus iterations for making the nodes agree with each other, thereby making it a one time-scale solution that is more communications efficient. Theoretical guarantees are also provided for our proposed distributed PCA algorithm when using a constant step size. The analysis shows that, when using a constant step size $\alpha$, the DSA solution reaches within a $\cO(\alpha)$-neighborhood of the optimal solution at a linear rate when the error metric is the angles between the estimated vectors and the true eigenvectors\footnote{Our results can also be extrapolated to guarantee exact convergence with decaying step size, albeit at a slower than linear rate.}. We also provide experimental results and comparisons with centralized orthogonal iteration~\cite{Golub}, centralized GHA~\cite{sanger}, a sequential distributed power method-based approach and distributed projected gradient descent. The results support our claims and analysis.

To the best of our knowledge, this is the first solution for distributed PCA that uses a Hebbian update, achieves network agreement without the use of explicit consensus iterations, and still provably reaches the globally optimum solution (within an error margin) at all nodes at a linear rate.

\subsection{Notation and Organization}
The following notation is used in this paper. Scalars and vectors are denoted by lower-case and lower-case bold letters, respectively, while matrices are denoted by upper-case bold letters. The operator $|\cdot|$ denotes the absolute value of a scalar quantity. The superscript in $\ba^{(t)}$ denotes time (or iteration) index, while $a^t$ denotes the exponentiation operation. The superscript $(\cdot)^\tT$ denotes the transpose operation, $\|\cdot\|_F$ denotes the Frobenius norm of matrices, while both $\|\cdot\|$ and $\|\cdot\|_2$ denote the $\ell_2$-norm of vectors. Given a matrix $\bA$, both $a_{ij}$ and $(\bA)_{ij}$ denote its entry at the $i^{th}$ row and $j^{th}$ column, while $\ba_j$ denotes its $j^{th}$ column. 

The rest of the paper is organized as follows. In Section~\ref{sec:problem}, we describe and mathematically formulate the distributed PCA problem, while Section~\ref{sec:algo} describes the proposed distributed algorithm, which is based on the generalized Hebbian algorithm. In Section~\ref{sec:centralized_sanger}, we derive a general result for a modified generalized Hebbian algorithm that aids in the convergence analysis of the proposed distributed algorithm, while convergence guarantees for the proposed algorithm are provided in Section~\ref{sec:analysis_dsa_general}. We provide numerical results in Section~\ref{sec:experiments} to show efficacy of the proposed method and provide concluding remarks in Section~\ref{sec:conc}. The statements and proofs of some auxiliary lemmas, which are needed for the proofs of the main lemmas that are used within the convergence analysis in Section~\ref{sec:centralized_sanger} and Section~\ref{sec:analysis_dsa_general}, are given in Appendix~\ref{app:aux_lemma}, while Appendices~\ref{app:lemma3}--\ref{app:lemma10} contain the formal statements and proofs of the main lemmas.

%% file: ProblemDescription.tex
\section{Problem Formulation}\label{sec:problem}
Principal Component Analysis (PCA) aims at finding the basis of a low-dimensional space that can decorrelate the features of data points and also retain maximum information. More formally, for a random vector $\by \in \R^d$ with $\E \begin{bmatrix}\by \end{bmatrix}= \bzero$, PCA involves finding the top-$K$ eigenvectors of the covariance matrix $\bSigma:=\E \begin{bmatrix}\by \by^{\tT}\end{bmatrix}$. The zero mean assumption is taken here without loss of generality as the mean can be subtracted in case data is not centered. Mathematically, PCA can be formulated as
\begin{align}\label{eq:pca}
    \bX &= \underset{\bX \in \R^{d\times K}} \argmin \quad \E\begin{bmatrix}\|\by-\bX\bX^{\tT}\by\|_2^2\end{bmatrix} \qquad \text{such that}\qquad \forall l\neq q, \ \Big(\E \begin{bmatrix} \bX^{\tT}\by\by^{\tT}\bX \end{bmatrix}\Big)_{lq} = 0.
\end{align}
The constraint $\Big(\E \begin{bmatrix} \bX^{\tT}\by\by^{\tT}\bX \end{bmatrix}\Big)_{lq} = 0, \forall l\neq q$, ensures that $\bX$ decorrelates the features of $\by$. Now, $\E \begin{bmatrix} \bX^{\tT}\by\by^{\tT}\bX \end{bmatrix} = \bX^{\tT}\E \begin{bmatrix} \by\by^{\tT} \end{bmatrix}\bX = \bX^{\tT}\bSigma\bX$ and it is straightforward to see that this quantity is diagonal only if $\bX$ contains the eigenvectors of $\bSigma$. This explains why the search for a solution of PCA ends with the eigenvectors and not the subspace spanned by them. In practice, we do not have access to $\bSigma$ and so a covariance matrix estimated from the samples of $\by$ is used instead. Specifically, for a dataset with $N$ samples $\{\by_l\}_{l=1}^N$, or equivalently, for a data matrix $\bY:=\begin{bmatrix} \by_1, \by_2,\dots, \by_N \end{bmatrix}$, the sample covariance matrix can be written as $\bC = \frac{1}{N}\bY \bY^{\tT}$ such that $\bSigma:=\E \begin{bmatrix}\bC\end{bmatrix}$. The true solution for PCA is then obtained by finding the eigenvectors of the covariance matrix $\bC$, which are also the left singular vectors of the data matrix $\bY$. The empirical form of \eqref{eq:pca} is thus
\begin{align}\label{eq:pca1}
    \bX &= \underset{\bX \in \R^{d\times K}} \argmin \quad f(\bX) = \underset{\bX \in \R^{d\times K}} \argmin \quad \|\bY-\bX\bX^{\tT}\bY\|_F^2 \qquad \text{such that} \qquad  \forall l\neq q, \ \Big(\bX^{\tT}\bY\bY^{\tT}\bX\Big)_{lq} = 0.
\end{align}
In the literature, however, PCA is usually posed with a `relaxed' orthogonality constraint of $\bX^T\bX = \bI$ instead of $\Big(\bX^{\tT}\bY\bY^{\tT}\bX\Big)_{lq} = 0, \forall l\neq q,$ as follows:
\begin{equation}\label{eq:pca2}
\bX = \underset{\bX \in \R^{d\times K}, \bX^{\tT}\bX = \bI} \argmin f(\bX) = \underset{\bX \in \R^{d\times K}, \bX^{\tT}\bX = \bI} \argmin \|\bY-\bX\bX^{\tT}\bY\|_F^2.
\end{equation}
The optimization formulation in \eqref{eq:pca2} with this constraint will only lead to a subspace spanned by the eigenvectors of $\bC$ as the solution, thus actually making it a Principal Subspace Analysis (PSA) formulation. In other words, although the formulation \eqref{eq:pca2} gives a solution on the Stiefel manifold, the actual PCA formulation \eqref{eq:pca1} requires the solution to be within a very specific subset of that manifold that corresponds to the eigenvectors of $\bC$. The accuracy of the solutions given by the PCA and PSA formulations will be the same when measured in terms of the principal angles between the subspace estimates and the true subspace spanned by the eigenvectors of the covariance matrix. Specifically, if $\bX = \begin{bmatrix} \bx_1,\cdots, \bx_K \end{bmatrix}$ is an estimate of the basis of the space spanned by the eigenvectors $\bQ = \begin{bmatrix} \bq_1,\cdots,\bq_K\end{bmatrix}$, then the principal angles between $\bQ$ and $\bX$ given by either \eqref{eq:pca1} or \eqref{eq:pca2} will be the same. But a more suitable measure of accuracy for any PCA solution should be the angles between $\bx_i$ and $\bq_i$ for all $i=1, \cdots K$, which motivates us to judge the efficacy of any solution with respect to this metric instead of the principal subspace angles.

In the distributed setup considered in this paper, we consider a network of $M$ nodes such that the undirected graph, $\cG:=(\cV, \cE)$, describing the network is connected. Here $\cV=\{1, 2, \dots, M\}$ is the set of nodes and $\cE$ is the set of edges, i.e., $(i,j) \in \cE$ if there is a direct path between $i$ and $j$. The set of neighbors for any node $i$ is denoted by $\cN_i$. Under the setup of samples being distributed over the $M$ nodes, let us assume that the $i^{th}$ node has a data matrix $\bY_i$ containing $N_i$ samples such that $N = \sum_{i=1}^{M}N_i$. Thus each node has access to only a local covariance matrix $\bC_i = \frac{1}{N_i}\bY_i\bY_i^{\tT}$ instead of the global covariance matrix but one can see that $N\bC=\sum_{i=1}^{M}N_i\bC_i$. In this setting, a straightforward approach might be that each node finds its own solution independent of the data at all the other nodes. While this might seem viable, this approach will have major drawbacks. Recall that the sample covariance $\bC$ approximates the population covariance $\bSigma$ at a rate of $\cO(f(N^{-1}))$, where $f$ is some function (depending on the distribution) of the number of samples $N$~\cite{samplecovariance}. 
Since the local data has smaller number of samples than the global data, working with the local covariance matrix $\bC_i$ alone instead of somehow using the whole data will lead to a larger error in estimation of the eigenvectors. Also, since uniform sampling from the underlying data distribution is not guaranteed in distributed setups, the samples at a node may end up being from a narrow part of the entire distribution, thus being more biased away from the true distribution. This invites the need for the nodes to collaborate amongst themselves in a way that all the data is utilized to find estimates of the eigenvectors at each node while ensuring that all the nodes agree with each other. Thus, for a distributed setting, the PCA problem in \eqref{eq:pca} can be rewritten here as
\begin{align}\label{eq:dpca}
\bX &= \underset{\bX \in \R^{d\times K}} \argmin \quad \sum_{i=1}^{M} f_i(\bX) = \underset{\bX \in \R^{d\times K}} \argmin \quad \sum_{i=1}^{M} \|\bY_i-\bX\bX^{\tT}\bY_i\|_F^2 \quad \text{such that} \quad \forall l\neq q, \ \Big(\bX^{\tT}\big(\sum_{i=1}^{M}\bY_i\bY_i^{\tT}\big)\bX\Big)_{lq} = 0.
\end{align}

It is easy to see that $\sum_{i=1}^{M}f_i(\bX) = f(\bX)$. Also, in a distributed setup, each node $i$ maintains its own copy $\bX_i$ of the variable $\bX$ due to the difference in local information (local data) they carry. Thus, all nodes need to agree with each other to ensure the entire network reaches the same true solution. Hence, the true distributed PCA objective is written as
\begin{align}\label{eq:dpca1}
 &\underset{\bX_i \in \R^{d\times K}} \argmin \sum_{i= 1}^{M}\|\bY_i-\bX_i\bX_i^{\tT}\bY_i\|_F^2 \quad \text{such that}\quad \forall j \in \cN_i, \ \bX_i = \bX_j \quad \text{and} \quad \forall l\neq q, \ \Big(\bX_i^{\tT}\big(\sum_{i=1}^{M}\bY_i\bY_i^{\tT}\big)\bX_i\Big)_{lq}= 0.
\end{align}
Note that \eqref{eq:pca}--\eqref{eq:dpca1} are non-convex optimization problems due to the non-convexity of the constraint set. One possible solution to the PCA problem is to instead solve a convex relaxation of the original non-convex function \cite{arora2013stochastic,warmuth2007randomized}. The issue with these solutions is that they require $O(d^2)$ memory and computation, which can be prohibitive in high-dimensional settings. In addition, due to $O(d^2)$ iterate size these solutions are not ideal for distributed settings. Also, these formulations, without any further constraints, will not necessarily give a basis that is the set of dominant eigenvectors. Instead, they might end up giving a rotated basis as explained earlier, thereby not completing the task of decorrelating features. Hence, in this paper we use an algebraic method based on GHA for neural network training, which has $O(dK)$ memory and computation requirements, to solve the distributed PCA problem. Our goal is to converge to the true eigenvectors of the global covariance matrix $\bC$ at every node of the network. As noted earlier in Section~\ref{subsec:ExistingWork}, distributed variants of the power method exist in the literature~\cite{cksvd.allerton.2013,cksvd, depm} that can find the dominant eigenvector but these methods employ two time-scale approaches that involve several consensus averaging rounds for each iteration of the power method. Such two time-scale approaches can be expensive in terms of communications cost. In this paper, we propose a one time-scale method that finds the top $K$ eigenvectors of the global sample covariance matrix $\bC$ at each node through local computations and information exchange with neighbors. The proposed method also converges linearly up to a neighborhood of the true solution when the error metric considered is the angle between the estimates and the true eigenvectors.


%% file: ProposedAlgorithm.tex
\section{The Proposed Algorithm}\label{sec:algo}
In~\cite{sanger}, Sanger proposed a generalized Hebbian algorithm (GHA) to train a neural network and find the eigenvectors of the input autocorrelation matrix (same as the covariance matrix for zero-mean input). The outputs of such a network, when the weights are given by the eigenvectors, are the uncorrelated features of the input data that allow data reconstruction with minimal error, hence serving the true purpose of PCA. The algorithm was originally developed to tackle the centralized PCA problem in the case of streaming data, where a new data sample $\by_t, t = 1,2, \ldots$, arrives at each time instance. 

In this paper we consider a batch setting, but the alignment of GHA with our basic goal of finding the eigenvectors motivates us to leverage it for our distributed setup. The rationale behind the idea of extrapolating the streaming case to a distributed batch setting is simple: since $\E[\by_t \by_t^{\tT}] =\E[\bY_i \bY_i^{\tT}]=\bSigma$, the sample-wise distributed data setting can be seen as a \emph{mini-batch} variant of the streaming data setting. In the context of neural network training, our approach can be viewed as training a network at each node with a mini-batch of samples in a way that all nodes end up with the same trained network whose weights are given by the eigenvectors of the autocorrelation matrix of the entire batch of samples.

The iterate for the GHA as given in~\cite{sanger} has the following update for the matrix of eigenvectors (i.e., the neural network weight matrix) $\bX$ when the $t^{th}$ sample $\by_t$ arrives at the input of the neural network:
\begin{equation} \label{eq:centralized_gha}
\bX^{(t+1)} = \bX^{(t)} + \alpha_t \Big[\by_t\by_t^{\tT}\bX^{(t)} - \bX^{(t)}\boldsymbol{\mathcal{U}}\Big((\bX^{(t)})^{\tT}\by_t\by_t^{\tT}\bX^{(t)}\Big)\Big],
\end{equation}
where $\boldsymbol{\mathcal{U}} :\R^{K\times K} \rightarrow \R^{K\times K}$ is an operator that sets all the elements below the diagonal to zero and $\alpha_t$ is the step size. For $K=1$, and defining $\bSigma_t = \by_t\by_t^{\tT}$, it was shown in~\cite{oja} that the term $(\bX^{(t)})^{\tT}\by_t\by_t^{\tT}\bX^{(t)} = (\bX^{(t)})^{\tT}\bSigma_t\bX^{(t)}$ is the consequence of a power series approximation of Oja's rule in lieu of the explicit normalization used in the case of the power method. 
In the case of $K > 1$, $\boldsymbol{\mathcal{U}}\Big((\bX^{(t)})^{\tT}\by_t\by_t^{\tT}\bX^{(t)}\Big)=\boldsymbol{\mathcal{U}}\Big((\bX^{(t)})^{\tT}\bSigma_t\bX^{(t)}\Big)$ helps combine Oja's algorithm with the Gram--Schmidt orthogonalization step as follows:
\begin{align}\nonumber
    \bX^{(t)}\boldsymbol{\mathcal{U}}\Big((\bX^{(t)})^{\tT}\bSigma_t\bX^{(t)}\Big) &= \bX^{(t)}\boldsymbol{\mathcal{U}}\Bigg(\begin{bmatrix}(\bx_1^{(t)})^{\tT}\\\vdots\\(\bx_K^{(t)})^{\tT}\end{bmatrix}\bSigma_t\begin{bmatrix}\bx_1^{(t)} & \cdots & \bx_K^{(t)}\end{bmatrix}\Bigg)\\\nonumber
    &= \bX^{(t)}\boldsymbol{\mathcal{U}}\Bigg(\begin{bmatrix}(\bx_1^{(t)})^{\tT}\bSigma_t\bx_1^{(t)}&(\bx_1^{(t)})^{\tT}\bSigma_t\bx_2^{(t)}&\ldots&(\bx_1^{(t)})^{\tT}\bSigma_t\bx_K^{(t)}\\
    (\bx_2^{(t)})^{\tT}\bSigma_t\bx_1^{(t)}&(\bx_2^{(t)})^{\tT}\bSigma_t\bx_2^{(t)}&\ldots&(\bx_2^{(t)})^{\tT}\bSigma_t\bx_K^{(t)}\\
    \vdots &\vdots&\ddots&\vdots\\
    (\bx_K^{(t)})^{\tT}\bSigma_t\bx_1^{(t)}&(\bx_K^{(t)})^{\tT}\bSigma_t\bx_2^{(t)}&\ldots&(\bx_K^{(t)})^{\tT}\bSigma_t\bx_K^{(t)}
    \end{bmatrix}\Bigg)\\\nonumber
    &= \bX^{(t)}\Bigg(\begin{bmatrix}(\bx_1^{(t)})^{\tT}\bSigma_t\bx_1^{(t)}&(\bx_1^{(t)})^{\tT}\bSigma_t\bx_2^{(t)}&\ldots&(\bx_1^{(t)})^{\tT}\bSigma_t\bx_K^{(t)}\\
    0&(\bx_2^{(t)})^{\tT}\bSigma_t\bx_2^{(t)}&\ldots&(\bx_2^{(t)})^{\tT}\bSigma_t\bx_K^{(t)}\\
    \vdots &\vdots&\ddots&\vdots\\
    0&0&\ldots&(\bx_K^{(t)})^{\tT}\bSigma_t\bx_K^{(t)}
    \end{bmatrix}\Bigg)\\\label{eq:gs_expansion}
    &=\begin{bmatrix}(\bx_1^{(t)})^{\tT}\bSigma_t\bx_1^{(t)}\bx_1^{(t)}&\sum_{p=1}^{2}(\bx_p^{(t)})^{\tT}\bSigma_t\bx_2^{(t)}\bx_p^{(t)}&\ldots&\sum_{p=1}^{K}(\bx_p^{(t)})^{\tT}\bSigma_t\bx_K^{(t)}\bx_p^{(t)}\end{bmatrix}.
\end{align}
Thus, for any $k=1,\ldots,K$, the term involving $\bcU(\cdot)$ in~\eqref{eq:centralized_gha} includes an implicit normalization term $(\bx_k^{(t)})^{\tT}\bSigma_t\bx_k^{(t)}\bx_k^{(t)}$ as well an orthogonalization term $\sum_{p=1}^{k-1}(\bx_p^{(t)})^{\tT}\bSigma_t\bx_k^{(t)}\bx_p^{(t)}$, which---analogous to the Gram--Schmidt orthogonalization procedure---forces the estimate $\bx_k^{(t)}$ to be orthogonal to all the estimates $\bx_p^{(t)}, p=1,\ldots,k-1$. Another important thing to note about the GHA algorithm is that, in order to estimate the dominant $K$ eigenvectors, it only requires the corresponding top $K$ eigenvalues to be distinct (and nonzero). In other words, it does not require the covariance matrix to be non-singular.

In the deterministic setting, where we have the full-batch instead of new samples every instance, this iterate changes to
\begin{align}\label{eq:centralized_gha1}
   \bX^{(t+1)} &= \bX^{(t)} + \alpha_t \Big[\bC\bX^{(t)} - \bX^{(t)}\boldsymbol{\mathcal{U}}\Big((\bX^{(t)})^{\tT}\bC\bX^{(t)}\Big)\Big] = \bX^{(t)} + \alpha_t\bcH(\bC, \bX^{(t)}).
\end{align} 
Here, we term $\bcH:\R^{d\times d}\times \R^{d\times K}\rightarrow \R^{d\times K}, \bcH(\bC, \bX^t):=\Big(\bC\bX^{t} - \bX^{t}\boldsymbol{\mathcal{U}}\Big((\bX^{t})^{\tT}\bC\bX^{t}\Big)\Big)$ as the Sanger direction. An iterate similar to~\eqref{eq:centralized_gha1} has been proven to have global convergence in~\cite{lv.yi.tan.2007} for some very specific choice of the step sizes that are dependent on the iterate itself. Its straightforward extension to the distributed case is not possible as that would lead to different step sizes at different nodes of the network, making it difficult to talk about its convergence guarantees. Hence, to adapt this iterative method to our distributed setup, we use the typical combine and update strategy used quite richly in the literature for distributed algorithms such as~\cite{dgd,extra,cattivelli2010diffusion,kar2013consensus}. The main contributions of such works lie in showing that the resulting distributed algorithms achieve consensus (i.e., all nodes will have the same iterate values eventually) and, in addition, the consensus value is the same as the centralized solution. The convergence guarantees for these methods are mainly restricted to convex and strongly convex problems though. Our distributed version of \eqref{eq:centralized_gha1} for PCA, which is non-convex, is based on similar principles of combine and update. 

Specifically, the node $i$ at iteration $t$ carries a local copy $\bX_i^{(t)}$ of the estimate of the eigenvectors of the global covariance matrix $\bC$. In the combine step, each node $i$ exchanges the iterate values with its immediate neighbors $j \in \cN_i$, where $\cN_i$ denotes the neighborhood of node $i$, and then takes a weighted sum of the iterates received along with its local iterate. Then this sum is updated independently at all nodes using their respective local information. Since node $i$ in the network only has access to its local sample covariance $\bC_i$, the update is in the form of a local Sanger's direction given as
\begin{equation}
    \bcH_i(\bC_i, \bX_i^{(t)}) = \bC_i\bX_i^{(t)} - \bX_i^{(t)}\boldsymbol{\mathcal{U}}\Big((\bX_i^{(t)})^{\tT}\bC_i\bX_i^{(t)}\Big).
\end{equation}

\begin{algorithm}[ht]
	\textbf{Input:} $\bY_1,\bY_2, \dots \bY_M, [w_{ij}], \alpha, K$\\
	\textbf{Initialize:} $\forall i, \bX_i^{(0)} \gets \bX_{\text{init}}: \bX_{\text{init}} \in \R^{d\times K}, \bX_{\text{init}}^{\tT} \bX_{\text{init}} = \bI$
	\begin{algorithmic}
		\For{$t=1,2,\dots$}
		    \State Communicate $\bX_i^{(t-1)}$ from each node $i$ to its neighbors
		    \State Estimate of eigenvectors at node $i$: $\mbox{\quad} \bX_i^{(t)} \gets \sum_{j\in \cN_i \cup \{i\}}w_{ij}\bX_j^{(t-1)} + \alpha \bcH_i(\bX_i^{(t-1)})$
		\EndFor
	\end{algorithmic}
	{\bf Return:} $\bX_i^{(t)}, i = 1,2, \dots, M$
	\caption{Distributed Sanger's Algorithm (DSA)}
	\label{algo:dsa}
\end{algorithm}

The details of the proposed distributed PCA algorithm, called the Distributed Sanger's Algorithm (DSA), are given in Algorithm~\ref{algo:dsa}. The weight matrix $\bW = [w_{ij}]$ in this algorithm is a doubly stochastic matrix conforming to the network topology~\cite{consensus} in the sense that for $i\neq j$, $w_{ij} \neq 0$ when $(i,j) \in \cE$ and $w_{ij} = 0$ otherwise. Also, $\forall i, w_{ii}\neq 0$, i.e., there is a self loop at each node. Note that connectivity of the network, as discussed in Section~\ref{sec:problem}, is a necessary condition for convergence of DSA. The connectivity assumption, in turn, ensures the Markov chain underlying the graph $\cG$ is aperiodic and irreducible, which implies that the second-largest (in magnitude) eigenvalue of $\bW$, $\beta = \max\{|\lambda_2(\bW)|, |\lambda_M(\bW)|\}$, is strictly less than $1$. While DSA shares algorithmic similarities with first-order distributed optimization methods~\cite{dgd,dgd1} in which the combine-and-update strategy is used, our challenge is characterizing its convergence behavior due to the non-convex and constrained nature of the distributed PCA problem. To this end, we first provide a general result in Section~\ref{sec:centralized_sanger} where we prove the convergence of a modified form of GHA. Then we utilize that result, along with some linear algebraic tools and additional lemmas provided in the appendices, to characterize the dynamics of the distributed setup in Section~\ref{sec:analysis_dsa_general} and prove the convergence of the proposed algorithm.

%% file: ConvergenceAnalysisDSA.tex
\section{Convergence Analysis of a Modified GHA}\label{sec:centralized_sanger}
Let $\bX^{(t)} = \begin{bmatrix}\bx_{1}^{(t)} & \bx_2^{(t)} & \cdots & \bx_{K}^{(t)}\end{bmatrix} \in \R^{d\times K}$, $K \leq d$, be an estimate of the $K$-dimensional subspace spanned by the eigenvectors of the covariance matrix $\bC$ after $t$ iterations and $\bq_l, l = 1,\ldots, d$, be the eigenvectors of $\bC$ with corresponding eigenvalues $\lambda_l$. On expanding~\eqref{eq:centralized_gha1} using~\eqref{eq:gs_expansion}, it is clear that the GHA update equation for estimation of the $k^{th}$ eigenvector using a constant step size $\alpha$ is as follows:
\begin{align}\label{eq:eq:gha_centralized_exp}
    	\bx_{k}^{(t+1)} = \bx_k^{(t)}  + \alpha\big(\bC\bx_k^{(t)} - (\bx_k^{(t)})^T\bC\bx_k^{(t)}\bx_k^{(t)} - \sum_{p=1}^{k-1}\bx_p^{(t)}(\bx_p^{(t)})^{\tT}\bC\bx_k^{(t)}\big).
\end{align}
We now slightly modify~\eqref{eq:eq:gha_centralized_exp} by replacing $\bx_p^{(t)}$ for $p<k$ by the true eigenvectors $\bq_p$. We term the resulting update equation \textit{modified GHA} and note that this is not an algorithm in the true sense of the term as it cannot be implemented because of its dependence on the true eigenvectors $\bq_p$. The sole purpose of this modified GHA is to help in our ultimate goal of providing convergence guarantee for the DSA algorithm.
The update equation of the modified GHA for ``estimation'' of the $k^{th}$ eigenvector of $\bC$, $k=1,\ldots,K$, has the form
\begin{equation}\label{eq:centralized_sangerk}
	\bx_{k}^{(t+1)} = \bx_k^{(t)}  + \alpha\big(\bC\bx_k^{(t)} - (\bx_k^{(t)})^T\bC\bx_k^{(t)}\bx_k^{(t)} - \sum_{p=1}^{k-1}\bq_p\bq_p^T\bC\bx_k^{(t)}\big).
\end{equation}
Note that similar to the original GHA, this modified GHA assumes that $\bC$ has $K$ distinct eigenvalues, i.e., $\lambda_1 > \lambda_2 > \ldots > \lambda_K > \lambda_{K+1} \geq \cdots \geq \lambda_d \geq 0$. Now, since $\bq_l, l = 1, \ldots, d$, are the eigenvectors of a real symmetric matrix, they form a basis for $\R^d$ and can be used for expansion of any $\bx_{k}^{(t)}$ as
\begin{equation}\label{eq:expansion_centralized}
	\bx_{k}^{(t)} = \sum_{l=1}^{d}z_{k,l}^{(t)}\bq_l,
\end{equation}
where $z_{k,l}^{(t)}$ is the coefficient corresponding to the eigenvector $\bq_l$ in the expansion of $\bx_k^{(t)}$. Multiplying both sides of \eqref{eq:centralized_sangerk} by $\bq_l^T$ and using the fact that $\bq_l^T\bq_{l'} = 0$ for $l \neq l'$, we get 
\begin{eqnarray*}
	{z}_{k,l}^{(t+1)} &=&{z}_{k,l}^{(t)} + \alpha(\bq_l^T\bC{\bx}_{k}^{(t)} - \bq_l^T(\sum_{p=1}^{k-1}\bq_p\bq_p^T\bC{\bx}_{k}^{(t)})- ({\bx}_{k}^{(t)})^T\bC{\bx}_{k}^{(t)}{z}_{k,l}^{(t)}).
\end{eqnarray*}
This gives
\begin{align} \label{eq:eqn_z_lower}
    {z}_{k,l}^{(t+1)} &= {z}_{k,l}^{(t)} - \alpha({\bx}_{k}^{(t)})^T\bC{\bx}_{k}^{(t)}{z}_{k,l}^{(t)}, \quad \text{for} \quad l= 1, \ldots, k - 1, \\ \label{eq:eqn_z_upper}
    \text{and} \quad {z}_{k,l}^{(t+1)} &= {z}_{k,l}^{(t)} + \alpha(\lambda_l - ({\bx}_{k}^{(t)})^T\bC{\bx}_{k}^{(t)}){z}_{k,l}^{(t)}, \quad \text{for} \quad l= k, \ldots, d.
\end{align}

It has been shown in~\cite{yi.tan2005} that the update equation given by
\begin{align*}
    	\bx_{1}^{(t+1)} = \bx_1^{(t)}  + \alpha\big(\bC\bx_1^{(t)} - (\bx_1^{(t)})^T\bC\bx_1^{(t)}\bx_1^{(t)}\big)
\end{align*}
for $k=1$ converges to $\pm \bq_1$ at a linear rate for a certain condition on the step size $\alpha$. Specifically, it was proven that $(z_{1,1}^{(t)})^2 \rightarrow 1 \quad \text{and} \quad \sum_{l=2}^d(z_{1,l}^{(t)})^2 \leq b_1\rho_1^t$,
where $b_1 > 0$ is some constant and $\rho_1 = \big(\frac{1+\alpha\lambda_2}{1+\alpha\lambda_1}\big)^{2} < 1$. Here, we extend the proof to a general $k$ and show that the update equation given in the form of~\eqref{eq:centralized_sangerk} for any $k=1, \ldots, K, K < d$, converges to the $k^{th}$ dominant eigenvector.
\begin{theorem} \label{theorem:convergence_centralized}
    Suppose $\alpha \leq \frac{1}{3\lambda_1{(2K-1)}}$, where $\lambda_1$ is the largest eigenvalue of $\bC$ and $K$ is the number of eigenvectors to be estimated, $\bq_k^T\bx_k^{(0)} \neq 0$, and $\|\bx_k^{(0)}\| = 1$ for all $k$. Then the modified GHA iterate for $\bx_k^{(t)}$ given by \eqref{eq:centralized_sangerk} converges at a linear rate to the eigenvector $\pm\bq_k$ corresponding to the $k^{th}$ largest eigenvalue $\lambda_k$ of the covariance matrix $\bC$.
\end{theorem}
\begin{proof}
The convergence of $\bx_k^{(t)}$ to $\bq_k$ requires convergence of the lower-order coefficients $z_{k,1}^{(t)}, \ldots, z_{k,k-1}^{(t)}$ and the higher-order coefficients $z_{k,k+1}^{(t)}, \ldots, z_{k,d}^{(t)}$ to 0 and convergence of $z_{k,k}^{(t)}$ to $\pm 1$. 
Now,
\begin{align}\nonumber
    |\lambda_k - ({\bx}_{k}^{(t)})^T\bC{\bx}_{k}^{(t)}| &= |\lambda_k - \sum_{l=1}^{d}\lambda_l({z}_{k,l}^{(t)})^2| = |\lambda_k - \lambda_k({z}_{k,k}^{(t)})^2 - \sum_{l=1}^{k-1}\lambda_l({z}_{k,l}^{(t)})^2 - \sum_{l=k+1}^{d}\lambda_l({z}_{k,l}^{(t)})^2|\\\nonumber
	&\geq |\lambda_k - \lambda_k({z}_{k,k}^{(t)})^2| - |\sum_{l=1}^{k-1}\lambda_l({z}_{k,l}^{(t)})^2| - |\sum_{l=k+1}^{d}\lambda_l({z}_{k,l}^{(t)})^2|\\
	\text{or,}\quad \lambda_k|1 - ({z}_{k,k}^{(t)})^2| &\leq |\sum_{l=1}^{k-1}\lambda_l({z}_{k,l}^{(t)})^2| + |\sum_{l=k+1}^{d}\lambda_l({z}_{k,l}^{(t)})^2| + |\lambda_k - ({\bx}_{k}^{(t)})^T\bC{\bx}_{k}^{(t)}|.
\end{align}
Thus, convergence of the lower-order and the higher-order coefficients to 0 along with convergence of the term $|\lambda_k - ({\bx}_{k}^{(t)})^T\bC{\bx}_{k}^{(t)}|$ will also imply the convergence of $z_{k,k}^{(t)}$ to $\pm 1$. To this end, Lemma~\ref{lemma:coeff_decay_lower} in the appendix proves linear convergence of the lower-order coefficients $z_{k,1}^{(t)}, \ldots, z_{k,k-1}^{(t)}$ to 0 by showing $\sum_{l=1}^{k-1}(z_{k,l}^{(t+1)})^2 < a_1\gamma^{t+1}$ for some constants $a_1 > 0, \gamma < 1$. Furthermore, Lemma~\ref{lemma:coeff_decay_upper} in the appendix shows that $\sum_{l=k+1}^{d}({z}_{k,l}^{(t+1)})^2 \leq a_2\rho_k^{t+1}$, where $a_1, a_2 > 0$ and $\gamma, \rho_k < 1$, thereby proving linear convergence of the higher-order coefficients to 0. Finally, Lemma \ref{lemma:monotonic_rayleigh_quotientk} in the appendix shows that $|\lambda_k - ({\bx}_{k}^{(t)})^T\bC{\bx}_{k}^{(t)}| \leq ta_4(\delta^{t+1} + \max\{\delta^t, \gamma_1^t\})$, where $a_4 > 0$ and $\delta, \gamma_1 < 1$. The formal statements and proofs of Lemma~\ref{lemma:coeff_decay_lower}, Lemma~\ref{lemma:coeff_decay_upper} and Lemma \ref{lemma:monotonic_rayleigh_quotientk} are given in Appendix~\ref{app:lemma3}, Appendix~\ref{app:lemma4} and Appendix~\ref{app:lemma5}, respectively.


Thus,
\begin{eqnarray*}
	\lambda_k|1 - ({z}_{k,k}^{(t)})^2| &\leq&|\sum_{l=1}^{k-1}\lambda_l({z}_{k,l}^{(t)})^2| + |\sum_{l=k+1}^{d}\lambda_l({z}_{k,l}^{(t)})^2| +ta_4(\delta^{t+1} + \max\{\delta^t, \gamma_1^t\}) \\
	&=&\sum_{l=1}^{k-1}\lambda_l({z}_{k,l}^{(t)})^2 + \sum_{l=k+1}^{d}\lambda_l({z}_{k,l}^{(t)})^2 + ta_4(\delta^{t+1} + \max\{\delta^t, \gamma_1^t\})\\
&<& \lambda_1 (\sum_{l=1}^{k-1}({z}_{k,l}^{(t)})^2 + \sum_{l=k+1}^{d}({z}_{k,l}^{(t)})^2) +ta_4(\delta^{t+1} + \max\{\delta^t, \gamma_1^t\})\\
&<& \lambda_1(a_1\gamma^t + a_2\rho_k^{t}) + ta_4(\delta^{t+1} + \max\{\delta^t, \gamma_1^t\}).
\end{eqnarray*}
Clearly, $\lim\limits_{t\rightarrow\infty} |1 - ({z}_{k,k}^{(t)})^2| = 0$. Therefore, Theorem~\ref{theorem:convergence_centralized} shows that with an update equation of the form~\eqref{eq:centralized_sangerk}, the iterates $\bx_{k}^{(t)}$ converge linearly to eigenvectors $\bq_k$ of the covariance matrix $\bC$.
\end{proof}

\section{Convergence Analysis of Distributed Sanger's Algorithm (DSA)}\label{sec:analysis_dsa_general}
With the analysis of the modified GHA in hand, let us proceed to analyze the proposed DSA algorithm. The iterate of DSA at node $i$ for the dominant $K$-dimensional eigenspace estimate ($K\leq d$) is given as
\begin{equation}\label{eq:dsa}
\bX_i^{(t+1)} = \sum_{j\in \cN_i \cup \{i\}}w_{ij}\bX_j^{(t)} + \alpha \bcH_i(\bC_i, \bX_i^{(t)}) = \sum_{j\in \cN_i \cup \{i\}}w_{ij}\bX_j^{(t)} + \alpha\Big(\bC_i\bX_i^{(t)} - \bX_i^{(t)}{\bcU}((\bX_i^{(t)})^{\tT}\bC_i\bX_i^{(t)})\Big),
\end{equation}
where $\bX_i^{(t)} = \begin{bmatrix}\bx_{i,1}^{(t)} & \bx_{i,2}^{(t)} & \cdots & \bx_{i,K}^{(t)}\end{bmatrix} \in \mathbb{R}^{d\times K}$ is an estimate of the $K$-dimensional subspace of the global covariance matrix $\bC$ at the $i^{th}$ node after $t$ iterations, $\bcH_i(\bC_i, \bX_i^{(t)})$ is local Sanger's direction, and $w_{ij} \geq 0$ is a weight that node $i$ assigns to $\bX_j^{(t)}$ based on the connectivity between nodes $i$ and $j$ as mentioned before.
The Sanger's direction and the update equation for an estimate of the $k^{th}$ eigenvector is thus given as
\begin{align}\label{eq:sanger_dir}
	\bcH_i(\bC_i, \bx_{i,k}^{(t)}) &= \bC_i\bx_{i,k}^{(t)} - (\bx_{i,k}^{(t)})^T\bC_i\bx_{i,k}^{(t)}\bx_{i,k}^{(t)} - \sum_{p=1}^{k-1}(\bx_{i,p}^{(t)})^T\bC_i\bx_{i,k}^{(t)}\bx_{i,p}^{(t)}\\ \label{eq:dsak}
	\text{and,} \quad \bx_{i,k}^{(t+1)} &= \sum_{j\in \cN_i \cup \{i\}}w_{ij}\bx_{j,k}^{(t)} + \alpha \big(\bC_i\bx_{i,k}^{(t)} - (\bx_{i,k}^{(t)})^T\bC_i\bx_{i,k}^t\bx_{i,k}^{(t)} - \sum_{p=1}^{k-1}\bx_{i,p}^{(t)}(\bx_{i,p}^{(t)})^T\bC_i\bx_{i,k}^{(t)}\big).
\end{align}
Now, let the average of $\bx_{1,k}^{(t)}, \bx_{2,k}^{(t)},\ldots,\bx_{M,k}^{(t)}$ after $t^{th}$ iteration be denoted as $\bar{\bx}_k^{(t)} = \frac{1}{M}\sum_{i=1}^{M}\bx_{i,k}^{(t)}$ and given by taking average of \eqref{eq:dsak} over all the nodes $i=1,\ldots,M$ as
\begin{eqnarray*}
	\bar{\bx}_k^{(t+1)} &=& \bar{\bx}_k^{(t)} +\frac{\alpha}{M}\sum_{i=1}^{M}\bcH_i(\bx_{i,k}^{(t)})\\
	&=&  \bar{\bx}_k^{(t)} + \frac{\alpha}{M}\sum_{i=1}^{M}\bcH_i(\bar{\bx}_{k}^{(t)}) + \frac{\alpha}{M}\sum_{i=1}^{M}\bcH_i(\bx_{i,k}^{(t)})-\frac{\alpha}{M}\sum_{i=1}^{M}\bcH_i(\bar{\bx}_{k}^{(t)}) = \bar{\bx}_k^{(t)} + \frac{\alpha}{M}\sum_{i=1}^{M}\bcH_i(\bar{\bx}_{k}^{(t)}) + \alpha\bh_k^{(t)}\\
	&=& \bar{\bx}_k^{(t)} + \frac{\alpha}{M}\big(\bC\bar{\bx}_{k}^{(t)} - (\bar{\bx}_{k}^{(t)})^T\bC\bar{\bx}_{k}^{(t)}\bar{\bx}_{k}^{(t)} - \sum_{i=1}^{M}\sum_{p=1}^{k-1}{\bx}_{i,p}^{(t)}({\bx}_{i,p}^{(t)})^T\bC_i\bar{\bx}_{k}^{(t)}\big) + \alpha\bh_k^{(t)},
\end{eqnarray*}
where $\bh_k^{(t)} = \frac{1}{M}\sum_{i=1}^{M}(\bcH_i(\bx_{i,k}^{(t)})-\bcH_i(\bar{\bx}_{k}^{(t)}))$. We present analysis of the DSA algorithm by first proving convergence of the average $\bar{\bx}_k^{(t)}$ to a neighborhood of the eigenvector $\bq_k$ of the global covariance matrix $\bC$ while using a constant step size. Then with the help of Lemma~\ref{lemma:bounded_mean_deviationk} in Appendix~\ref{app:lemma8}, which proves that the deviation of the iterates $\bx_{i,k}^{(t)}$ at each node from the average $\bar{\bx}_k^{(t)}$ is upper bounded, we prove that the iterates at each node also converge to a neighborhood of the true solution. It is noteworthy that the analysis of DSA does not require additional constraints on eigenvalues of $\bC_i$, i.e., similar to GHA, we only require the top $K$ eigenvalues of $\bC$ to be distinct and non-zero.

The complete proof of convergence of DSA is done by induction. First, we show the convergence of $\bx_{i,1}^{(t)}$ to a $\cO(\alpha)$ neighborhood of $\bq_1$ and then analyze the rest of the eigenvector estimates $\bx_{i,k}^{(t)}, k = 2,\ldots, K,$ by assuming that the higher-order estimates have converged.

\textbf{\textit{Case I for Induction -- $k=1$}:}
The iterate for the dominant eigenvector is
\begin{equation}\label{eq:dsa1}
\bx_{i,1}^{(t+1)} = \sum_{j\in \cN_i \cup \{i\}}w_{ij}\bx_{j,1}^{(t)} + \alpha (\bC_i\bx_{i,1}^{(t)} - ((\bx_{i,1}^{(t)})^T\bC_i\bx_{i,1}^{(t)})\bx_{i,1}^{(t)}).
\end{equation}
\begin{theorem}\label{theorem:gha1}
Suppose $\alpha \leq \frac{\min_i w_{ii}}{3\lambda_1{(2K-1)}}$, where $\lambda_1$ is the largest eigenvalue of $\bC$ and $K$ is the number of eigenvectors to  be estimated, $\bq_1^T\bx_{i,1}^{(0)} \neq 0$, and $\|\bx_{i,1}^{(0)}\| = 1$. Then the DSA iterate for $\bx_{i,1}^{(t)}$ given by~\eqref{eq:dsa1} converges at a linear rate to an $\cO(\alpha)$ neighborhood of the eigenvector $\pm\bq_1$ corresponding to the largest eigenvalue $\lambda_1$ of the global covariance matrix $\bC$ at every node of the network.
\end{theorem}
\begin{proof}
We know that
\begin{equation}\label{eq:error_distributed1}
\|\bx_{i,1}^{(t)} - \bx_{1}^{*}\| \leq \|\bx_{i,1}^{(t)} - \bar{\bx}_1^{(t)}\| + \|\bar{\bx}_1^{(t)} - \bx_{1}^{*}\|, \quad \text{where $\bx_{1}^* = \pm \bq_1$}.
\end{equation}
The term $\|\bx_{i,1}^{(t)} - \bar{\bx}_1^{(t)}\|$ is a measure of consensus in the network and we prove in Lemma \ref{lemma:bounded_mean_deviationk} in Appendix~\ref{app:lemma8} that this difference decreases linearly until it reaches a level of $\cO(\alpha)$. More precisely, 
\begin{align}\label{eq:consensus diff1}
   \|\bx_{i,1}^{(t)} - \bar{\bx}_1^{(t)}\| \leq b_1\big(\beta^t + \frac{\alpha}{1-\beta}\big),
\end{align}
where $\beta = \max\{|\lambda_2(\bW)|, |\lambda_M(\bW)|\}$. In particular, it is well known that for a connected graph $\beta < 1$. Now, the average iterate of DSA for the estimate of the dominant eigenvector $(k=1)$ is 
\begin{align*}
    \bar{\bx}_1^{(t)} &= \bar{\bx}_1^{(t-1)} +  \frac{\alpha}{M} (\bC\bar{\bx}_{1}^{(t-1)} - (\bar{\bx}_{1}^{(t-1)})^T\bC\bar{\bx}_{1}^{(t-1)}\bar{\bx}_{1}^{(t-1)}) + \alpha\bh_1^{(t-1)}.
\end{align*}
Thus,
\begin{align}\nonumber
    \bar{\bx}_1^{(t)} - \bx_{1}^{*} &= \bar{\bx}_1^{(t-1)} +  \frac{\alpha}{M} (\bC\bar{\bx}_{1}^{(t-1)} - (\bar{\bx}_{1}^{(t-1)})^T\bC\bar{\bx}_{1}^{(t-1)}\bar{\bx}_{1}^{(t-1)}) - \bx_1^{*} + \alpha\bh_1^{(t-1)} \\\nonumber
    \text{or,}\quad \|\bar{\bx}_1^{(t)} - \bx_{1}^{*}\| &= \| \bar{\bx}_1^{(t-1)} +  \frac{\alpha}{M} (\bC\bar{\bx}_{1}^{(t-1)} - (\bar{\bx}_{1}^{(t-1)})^T\bC\bar{\bx}_{1}^{(t-1)}\bar{\bx}_{1}^{(t-1)}) - \bx_1^{*} + \alpha\bh_1^{(t-1)}\|\\ \label{eq:error_inequality1}
    \text{or,}\quad\|\bar{\bx}_1^{(t)} - \bx_{1}^{*}\| &\leq \| \bar{\bx}_1^{(t-1)} +  \frac{\alpha}{M} (\bC\bar{\bx}_{1}^{(t-1)} - (\bar{\bx}_{1}^{(t-1)})^T\bC\bar{\bx}_{1}^{(t-1)}\bar{\bx}_{1}^{(t-1)}) - \bx_1^{*}\| + \alpha\|\bh_1^{(t-1)}\|.
\end{align}
We saw in Section~\ref{sec:centralized_sanger} that an iterate of the form 
\begin{align*}
    \bar{\bx}_1^{(t)} = \bar{\bx}_1^{(t-1)} + \frac{\alpha}{M} (\bC\bar{\bx}_{1}^{(t-1)} - (\bar{\bx}_{1}^{(t-1)})^T\bC\bar{\bx}_{1}^{(t-1)}\bar{\bx}_{1}^{(t-1)})
\end{align*}
converges linearly to $\bx_1^* = \pm \bq_1$ for certain conditions on the step size and the initial point. 
Thus, 
 \begin{eqnarray*}
	\|\bar{\bx}_1^{(t)} - \bx_{1}^{*}\| &\leq& \rho_1\|\bar{\bx}_1^{(t-1)} - \bx_{1}^{*}\| + \alpha\|\bh_1^{(t-1)}\|, \quad\text{where}\quad \rho_1 = \frac{1+\frac{\alpha}{M}\lambda_2}{1+\frac{\alpha}{M}\lambda_1}.
\end{eqnarray*}
The term $\bh_1^{(t-1)}$ in the above equation appears due to the distributed nature of the algorithm and can be bounded separately. Specifically, we prove in Lemma~\ref{lemma:bound_hk}, whose formal statement and proof is given in Appendix~\ref{app:lemma10}, that
\begin{align*}
    \|\bh_1^{(t-1)}\| \leq 9\lambda_1b_1\big(\beta^{t-1} + \frac{\alpha}{1-\beta}\big).
\end{align*}
Thus,
 \begin{eqnarray*}
 \|\bar{\bx}_1^{(t)} - \bx_{1}^{*}\| &\leq&  \rho_1\|\bar{\bx}_1^{(t-1)} - \bx_{1}^{*}\|  + 9\alpha\lambda_1b_1\big(\beta^{t-1} + \frac{\alpha}{1-\beta}\big)\\
	&\leq&  \rho_1\Big(\rho_1\|\bar{\bx}_1^{(t-2)} - \bx_{1}^{*}\|  + 9\alpha\lambda_1b_1\beta^{t-2} + 9\alpha\lambda_1b_1\big(\frac{\alpha}{1-\beta}\big)\Big) +  9\alpha\lambda_1b_1\beta^{t-1} + 9\alpha\lambda_1b_1\big(\frac{\alpha}{1-\beta}\big)\\
	&\leq& \rho_1^{t}\|\bar{\bx}_1^{(0)} - \bx_{1}^{*}\| +  9\alpha\lambda_1b_1\sum_{r=0}^{t-1}(\rho_1\beta^{-1})^r\beta^{t-1} + \frac{1}{1-\rho_1}9\alpha\lambda_1b_1\big(\frac{\alpha}{1-\beta}\big).
\end{eqnarray*}
Since $\rho_1, \beta < 1$, we have the following two cases:
\begin{enumerate}
	\item $\rho_1 \leq \beta \implies \rho_1\beta^{-1} \leq 1$. Then, $\sum_{r=0}^{t-1}(\rho_1\beta^{-1})^{r}\beta^{t-1} \leq \sum_{r=0}^{t-1}\beta^{t-1} = t\beta^{t-1}$.
	\item $\rho_1 > \beta$. Then $\sum_{r=0}^{t-1}(\rho_1\beta^{-1})^{r}\beta^{t-1} = \beta^{t-1} + \rho_1\beta^{t-2} + \cdots + \rho_1^{t-1} < \rho_1^{t-1} + \cdots + \rho_1^{t-1} = t\rho_1^{t-1}$.
\end{enumerate}
Therefore,
\begin{equation}\label{eq:error_average1}
\|\bar{\bx}_1^{(t)} - \bx_{1}^{*}\| \leq \rho_1^{t}\|\bar{\bx}_1^{(0)} - \bx_{1}^{*}\| +  c_1t\max \{\rho_1^{t-1}, \beta^{t-1}\} + \frac{c_1}{1-\rho_1}\big(\frac{\alpha}{1-\beta}\big), \quad \text{where} \quad c_1 = 9\alpha\lambda_1b_1.
\end{equation}
Consequently, from \eqref{eq:consensus diff1} and~\eqref{eq:error_average1}, we get
\begin{eqnarray*}
	\|\bx_{i,1}^{(t)} - \bx_1^{*}\| &\leq& b_1(\beta^t + \frac{\alpha}{1-\beta}) + \rho_1^{t}\|\bar{\bx}_1^{(0)} - \bx_{1}^{*}\| +  c_1t\max \{\rho_1^{t-1}, \beta^{t-1}\} + \frac{c_1}{1-\rho_1}\big(\frac{\alpha}{1-\beta}\big)\\
	&=& \rho_1^{t}\|\bar{\bx}_1^{(0)} - \bx_{1}^{*}\| + b_1\beta^t +  c_1t\max \{\rho_1^{t-1}, \beta^{t-1}\} + (\frac{c_1}{1-\rho_1} + b_1)\big(\frac{\alpha}{1-\beta}\big).
\end{eqnarray*}
This proves that $\bx_{i,1}^{(t)}$ converges to a neighborhood of $\bx_1^{*} = \bq_1$ or $\bx_1^{*} = -\bq$ at a linear rate.
\end{proof}
\textbf{\textit{Case II for Induction -- $1<k\leq K$}:} For the remainder of the eigenvectors, we proceed with the proof of convergence by induction. Since we have already proven the base case, we can assume there exist constants $c_{i,p} > 0$ and $\theta_{i,p} < 1$ such that
\begin{enumerate}
	\item $\|\bx_{i,p}^{(t)}(\bx_{i,p}^{(t)})^T - \bq_p\bq_p^T\| \leq c_{i,p}(\theta_{i,p}^t + \frac{\alpha}{1-\beta}), \forall p = 1, \ldots, k-1$, and
	\item $\|\bx_{i,p}^{(t)}\|^2 \leq 3, p = 1, \ldots, k-1, i = 1,\ldots M$.
\end{enumerate}
Using the inequality in 1) above, we can write ${\bx}_{i,p}^{(t)}({\bx}_{i,p}^{(t)})^T = \bq_p\bq_p^T + {\bphi}_{i,p}^{(t)}, p = 1,\ldots, k-1$ such that $\|{\bphi}_{i,p}^{(t)}\| \leq {c}_{i,p}({\theta}_{i,p}^t + \frac{\alpha}{1-\beta})$. This therefore implies $\frac{\alpha}{M}\sum_{i=1}^{M}\sum_{p=1}^{k-1}{\bx}_{i,p}^{(t)}({\bx}_{i,p}^{(t)})^T\bC_i\bar{\bx}_{k}^{(t)} =
	\frac{\alpha}{M}\sum_{i=1}^{M}\sum_{p=1}^{k-1}(\bq_{p}\bq_{p}^T + {\bphi}_{i,p}^{(t)})\bC_i\bar{\bx}_{k}^{(t)} = \frac{\alpha}{M}\sum_{p=1}^{k-1}\bq_p\bq_p^T\bC\bar{\bx}_{k}^{(t)} + \alpha\bar{\bpsi}_{k}^{(t)},$
where $\bar{\bpsi}_{k}^{(t)} = \frac{1}{M}\sum_{i=1}^{M}\sum_{p=1}^{k-1}{\bphi}_{i,p}^{(t)}\bC\bar{\bx}_{k}^{(t)}$.

Thus, we have
\begin{eqnarray}\nonumber
	\|\bar{\bpsi}_{k}^{(t)}\| &\leq& \frac{1}{M}\sum_{i=1}^{M}\sum_{p=1}^{k-1}\lambda_1\|{\bphi}_{i,p}^{(t)}\|\|\bar{\bx}_{k}^{(t)}\|
	\leq \frac{1}{M}\sum_{i=1}^{M}\sum_{p=1}^{k-1}\sqrt{3}\lambda_1{c}_{i,p}({\theta}_{i,p}^t + \frac{\alpha}{1-\beta})\\ \label{eq:psi_avg_decay}
	&\leq& \frac{1}{M}\sqrt{3}\lambda_1(k-1)M\bar{c}(\bar{\theta}^t + \frac{\alpha}{1-\beta}) = \sqrt{3}\lambda_1(k-1)\bar{c}(\bar{\theta}^t + \frac{\alpha}{1-\beta}),
\end{eqnarray}
where $\bar{c} = \max_{i,p} \{{c}_{i,p}\}$ and $\bar{\theta} = \max_{i,p} \{{\theta}_{i,p}\} < 1$.

Consequently,
\begin{eqnarray}\nonumber
		\bar{\bx}_k^{(t+1)} &=& \bar{\bx}_k^{(t)} + \frac{\alpha}{M}\big(\bC\bar{\bx}_{k}^{(t)} - (\bar{\bx}_{k}^{(t)})^T\bC\bar{\bx}_{k}^{(t)}\bar{\bx}_{k}^{(t)} - \sum_{i=1}^{M}\sum_{p=1}^{k-1}{\bx}_{i,p}^{(t)}({\bx}_{i,p}^{(t)})^T\bC_i\bar{\bx}_{k}^{(t)}\big) + \alpha\bh_k^{(t)}\\ \nonumber
		&=& \bar{\bx}_k^{(t)} + \frac{\alpha}{M}\big(\bC\bar{\bx}_{k}^{(t)} - (\bar{\bx}_{k}^{(t)})^T\bC\bar{\bx}_{k}^{(t)}\bar{\bx}_{k}^{(t)} - \sum_{i=1}^{M}\sum_{p=1}^{k-1}{\bq}_{p}{\bq}_{p}^T\bC_i\bar{\bx}_{k}^{(t)}\big) + \\ \nonumber
		&& \frac{\alpha}{M}\sum_{i=1}^{M}\sum_{p=1}^{k-1}({\bq}_{p}{\bq}_{p}^T - {\bx}_{i,p}^{(t)}({\bx}_{i,p}^{(t)})^T)\bC_i\bar{\bx}_{k}^{(t)}+ \alpha\bh_k^{(t)}\\ \nonumber
		&=& \bar{\bx}_k^{(t)} + \frac{\alpha}{M}\big(\bC\bar{\bx}_{k}^{(t)} - (\bar{\bx}_{k}^{(t)})^T\bC\bar{\bx}_{k}^{(t)}\bar{\bx}_{k}^{(t)} - \sum_{p=1}^{k-1}{\bq}_{p}{\bq}_{p}^T\bC\bar{\bx}_{k}^{(t)}\big) - \frac{\alpha}{M}\sum_{i=1}^{M}\sum_{p=1}^{k-1}\bphi_{i,p}^{(t)}\bC_i\bar{\bx}_{k}^{(t)}+ \alpha\bh_k^{(t)}\\ \label{eq:avg_dsak}
		&=& \bar{\bx}_k^{(t)} + \frac{\alpha}{M}\big(\bC\bar{\bx}_{k}^{(t)} - (\bar{\bx}_{k}^{(t)})^T\bC\bar{\bx}_{k}^{(t)}\bar{\bx}_{k}^{(t)} - \sum_{p=1}^{k-1}{\bq}_{p}{\bq}_{p}^T\bC\bar{\bx}_{k}^{(t)}\big) - \alpha\bar{\bpsi}_k^{(t)} + \alpha\bh_k^{(t)}.
\end{eqnarray}

We can now proceed with the final theorem that characterizes the convergence behavior of DSA.
\begin{theorem}
Suppose $\alpha \leq \frac{\min_i w_{ii}}{3\lambda_1{(2K-1)}}$, where $\lambda_1$ is the largest eigenvalue of $\bC$ and $K$ is the number of eigenvectors to  be estimated, $\bq_k^T\bx_{i,k}^{(0)} \neq 0$ and $\|\bx_{i,k}^{(0)}\| = 1,  \forall k=2,\dots,K$. Then the DSA iterate for $\bx_{i,k}^{(t)}$ given by~\eqref{eq:dsak} converges at a linear rate to an $\cO(\alpha)$ neighborhood of the eigenvector $\bq_k$ corresponding to the $k^{th}$ largest eigenvalue $\lambda_k$ of the global covariance matrix $\bC$ at each node of the network.
\end{theorem}
\begin{proof}
We know
\begin{equation}\label{eq:error_distributedk}
	\|\bx_{i,k}^{(t)} - \bx_{k}^{*}\| \leq \|\bx_{i,k}^{(t)} - \bar{\bx}_k^{(t)}\| + \|\bar{\bx}_k^{(t)} - \bx_{k}^{*}\|, \quad \text{where $\bx_{k}^* = \pm \bq_k$}.
\end{equation}
Also, from Lemma \ref{lemma:bounded_mean_deviationk} in the appendix we know that
\begin{align*}
   \|\bx_{i,k}^{(t)} - \bar{\bx}_k^{(t)}\| \leq b_k(\beta^t + \frac{\alpha}{1-\beta}\big).
\end{align*}
Now, the average iterate of DSA for estimating the $k^{th}$ eigenvector is
\begin{align*}
    \bar{\bx}_k^{(t)} &= \bar{\bx}_k^{(t-1)} +  \frac{\alpha}{M} (\bC\bar{\bx}_{k}^{(t-1)} - ((\bar{\bx}_{k}^{(t-1)})^T\bC\bar{\bx}_{k}^{(t-1)})\bar{\bx}_{k}^{(t-1)} - \sum_{p=1}^{k-1}{\bq}_{p}{\bq}_{p}^T\bC\bar{\bx}_{k}^{(t-1)}) + \alpha\bh_k^{(t-1)} + \alpha\bar{\bpsi}_k^{(t-1)}\\
    \text{or,} \quad \|\bar{\bx}_k^{(t)} - \bx_{k}^{*}\| &\leq \| \bar{\bx}_k^{(t-1)} +  \frac{\alpha}{M} (\bC\bar{\bx}_{k}^{(t-1)} - ((\bar{\bx}_{k}^{(t-1)})^T\bC\bar{\bx}_{k}^{(t-1)})\bar{\bx}_{k}^{(t-1)} - \sum_{p=1}^{k-1}{\bq}_{p}{\bq}_{p}^T\bC\bar{\bx}_{k}^{(t-1)}) - \bx_k^{*}\| + \\
    &\quad\quad\quad\quad\quad\quad\quad\quad\quad\quad\quad\quad\quad\quad\quad\quad\quad\quad \alpha\|\bh_k^{(t-1)}\| + \alpha\|\bar{\bpsi}_k^{(t-1)}\|.
\end{align*}
We know from the discussion in Section~\ref{sec:centralized_sanger} that for an iterate of the form 
\begin{align*}
    \bar{\bx}_k^{(t)} &= \bar{\bx}_k^{(t-1)} +  \frac{\alpha}{M} (\bC\bar{\bx}_{k}^{(t-1)} - ((\bar{\bx}_{k}^{(t-1)})^T\bC\bar{\bx}_{k}^{(t-1)})\bar{\bx}_{k}^{(t-1)} - \sum_{p=1}^{k-1}{\bq}_{p}{\bq}_{p}^T\bC\bar{\bx}_{k}^{(t-1)}),
\end{align*}
there exists a constant $\rho_k^{'} < 1$ such that $\|\bar{\bx}_k^{(t)} - \bx_{k}^{*}\| \leq \rho_k^{'}\|\bar{\bx}_k^{(t-1)} - \bx_{k}^{*}\|$. Thus,
\begin{align*}
    \|\bar{\bx}_k^{(t)} - \bx_{k}^{*}\| \leq \rho_k^{'}\|\bar{\bx}_k^{(t-1)} - \bx_{k}^{*}\| + \alpha\|\bh_k^{(t-1)}\| + \alpha\|\bar{\bpsi}_k^{(t-1)}\|.
\end{align*}
Now, the term $\|\bh_k^{(t-1)}\|$ was bounded in Lemma~\ref{lemma:bound_hk} in the appendix as
\begin{align}\label{eq:hk_bound1}
    \|\bh_k^{(t-1)}\| \leq 3(k+2)\lambda_1b_k\big(\beta^{t-1} + \frac{\alpha}{1-\beta}\big).
\end{align}
Thus, using~\eqref{eq:psi_avg_decay} and~\eqref{eq:hk_bound1}, we can write
\begin{eqnarray*}
	\|\bar{\bx}_k^{(t)} - \bx_{k}^{*}\| 
	&\leq&  \rho_k^{'}\|\bar{\bx}_k^{(t-1)} - \bx_{k}^{*}\|  + \alpha(3(k+2)\lambda_1b_k(\beta^{t-1} + \frac{\alpha}{1-\beta})) + \alpha(\sqrt{3}\lambda_1(k-1)\bar{c}(\bar{\theta}^{t-1} + \frac{\alpha}{1-\beta}))\\
	&\leq&  \rho_k^{'}\|\bar{\bx}_k^{(t-1)} - \bx_{k}^{*}\|  + c_k\max\{\beta^{t-1}, \bar{\theta}^{t-1}\}  + c_k\frac{\alpha}{1-\beta}, \quad c_k = \max\{\alpha(3(k+2)\lambda_1b_k),\alpha(\sqrt{3}\lambda_1(k-1)\bar{c})\} \\
	&\leq&  \rho_k^{'}\Big(\rho_k^{'}\|\bar{\bx}_k^{(t-2)} - \bx_{k}^{*}\|  + c_k\max\{\beta^{t-2}, \bar{\theta}^{t-2}\}  + c_k\frac{\alpha}{1-\beta}\Big) + c_k\max\{\beta^{t-1}, \bar{\theta}^{t-1}\}  + c_k\frac{\alpha}{1-\beta}\\
	&\leq& \rho_k^{'t}\|\bar{\bx}_k^{(0)} - \bx_{k}^{*}\| +  c_k\sum_{r=0}^{t-1}(\rho_k^{'}\max\{\beta, \bar{\theta}\}^{-1})^r\max\{\beta, \bar{\theta}\}^{t-1} + \frac{c_k}{1-\rho_k^{'}}\big(\frac{\alpha}{1-\beta}\big)\\
	&\leq& \rho_k^{'t}\|\bar{\bx}_k^{(0)} - \bx_{k}^{*}\| +  c_kt\max\{\rho_k^{'t-1}, \beta^{t-1}, \bar{\theta}^{t-1}\} + \frac{c_k}{1-\rho_k^{'}}\big(\frac{\alpha}{1-\beta}\big).
\end{eqnarray*}
Consequently, from \eqref{eq:error_distributedk} and Lemma \ref{lemma:bounded_mean_deviationk} we get
\begin{eqnarray*}
	\|\bx_{i,k}^{(t)} - \bx_k^{*}\| &\leq& b_k\big(\beta^t + \frac{\alpha}{1-\beta}\big) + \rho_k^{'t}\|\bar{\bx}_k^{(0)} - \bx_{k}^{*}\| +  c_kt\max\{\rho_k^{'t-1}, \beta^{t-1}, \bar{\theta}^{t-1}\} + \frac{c_k}{1-\rho_k^{'}}\big(\frac{\alpha}{1-\beta}\big)\\
	&=& \rho_k^{t}\|\bar{\bx}_k^{(0)} - \bx_{k}^{*}\| + b_k\beta^t +  c_k(t-1)\max\{\rho_k^{t-1}, \beta^{t-1}, \bar{\theta}^{t-1}\} + (\frac{c_k}{1-\rho_k} + b_k)\big(\frac{\alpha}{1-\beta}\big).
\end{eqnarray*}
This proves that $\bx_{i,k}^{(t)}$ converges to a neighborhood of $\bx_k^{*} = \bq_k$ or $\bx_k^{*} = -\bq_k$ at a linear rate.
\end{proof}
It is noteworthy that if decaying step sizes $\alpha_t$ are used such that $\alpha_t \rightarrow 0$ as $t \rightarrow \infty$ (instead of constant $\alpha$), the convergence will be exact but not linear. The rate in that case will be dominated by the rate of decay of $\alpha_t$.

%% file: ExperimentalResults.tex
\section{Experimental Results}\label{sec:experiments}
In this section, we provide results that demonstrate the efficacy of the proposed DSA algorithm. The need for collaboration between the nodes of a network is a vital part of any distributed algorithm, as already pointed out in Section~\ref{sec:problem}. We first verify that necessity along with the effect of step size on DSA by performing some experiments. In these experiments, the weight matrix $\bW$ that conforms to the underlying graph topology is generated using the Metropolis constant edge-weight approach~\cite{Boydfastestmixing.2003}. The performance of DSA in comparison to some baseline methods is also evaluated in additional experiments. We provide experimental results for DSA on synthetic and real data and compare the results with centralized generalized Hebbian algorithm (GHA)~\cite{sanger}, centralized orthogonal iteration (OI)~\cite{Golub}, distributed projected gradient descent (DPGD) and sequential distributed power method (SeqDistPM). For both the centralized methods, all the data is assumed to be at a single location with the difference being that GHA uses the Hebbian update whereas OI uses the well-known orthogonal iterations to estimate the top $K$ eigenvectors of the covariance matrix $\bC$. DPGD involves two significant steps per iteration. The first is a distributed gradient descent step at every node $i$ given by $\sum_{j\in \cN_i \cup \{i\}}w_{ij}\bX_j + \alpha\nabla f_i(\bX_i)$ as in~\cite{dgd} using trace maximization $f_i(\bX_i) = \max \text{Trace}(\bX_i^\tT\bC_i\bX_i)$ as the objective. This is followed by a projection step to ensure the orthogonality constraint $\bX_i^{\tT}\bX_i = \bI$. The orthogonalization is accomplished using QR decomposition, an approach that ensures projection onto the Stiefel manifold~\cite{absil.opt.matrixmanifolds.2007} and whose computational complexity is $\cO(K^2d)$, at each node in each iteration. In contrast, SeqDistPM involves implementing the distributed power method~\cite{cksvd.allerton.2013, cksvd} $K$ times, estimating one eigenvector at a time and subtracting its impact on the covariance matrix for the estimation of subsequent eigenvectors. Note that SeqDistPM requires a finite $T_c$ number of consensus iterations per iteration of the power method. Assuming the cost of communicating one $\R^{d\times K}$ matrix across the network from nodes to their neighbors to be one unit, the communication cost of SeqDistPM is $T_c/K$ per iteration of the power method. The error metric used for comparison and reporting of the results is the average of the angles between the estimated and true eigenvectors, i.e., if $\bx_{i,k}$ is the estimate of the $k^{th}$ eigenvector at $i^{th}$ node and $\bq_k$ is the true $k^{th}$ eigenvector then the average error across all nodes is calculated as follows:
\begin{equation}
    E = \frac{1}{MK}\sum_{i=1}^{M}\sum_{k=1}^{K}\bigg(1 - \big(\frac{\bx_{i,k}^T\bq_k}{\|\bx_{i,k}\|}\big)^2\bigg).
\end{equation}
\subsection{Synthetic Data}

\begin{figure}[t]
    \centering
    \begin{subfigure}{.4\textwidth}
     \centering   
     \includegraphics[width=\linewidth]{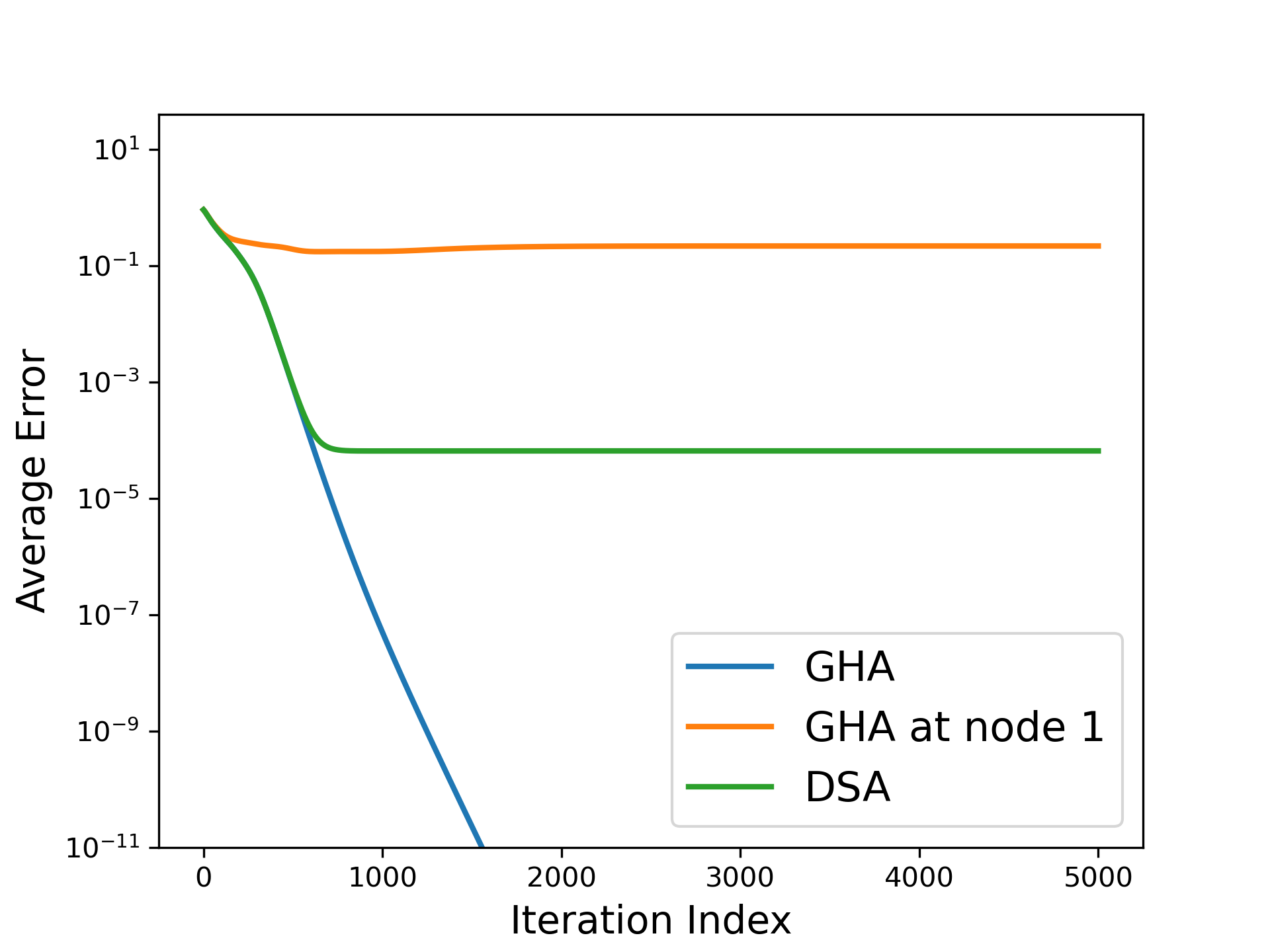}
    \caption{Demonstration of the need for collaboration among the nodes in a network for the PCA problem}
    \label{fig:compare_dsa_gha}
    \end{subfigure}
    \hfil
    \begin{subfigure}{.4\textwidth}
     \centering   
     \includegraphics[width=\linewidth]{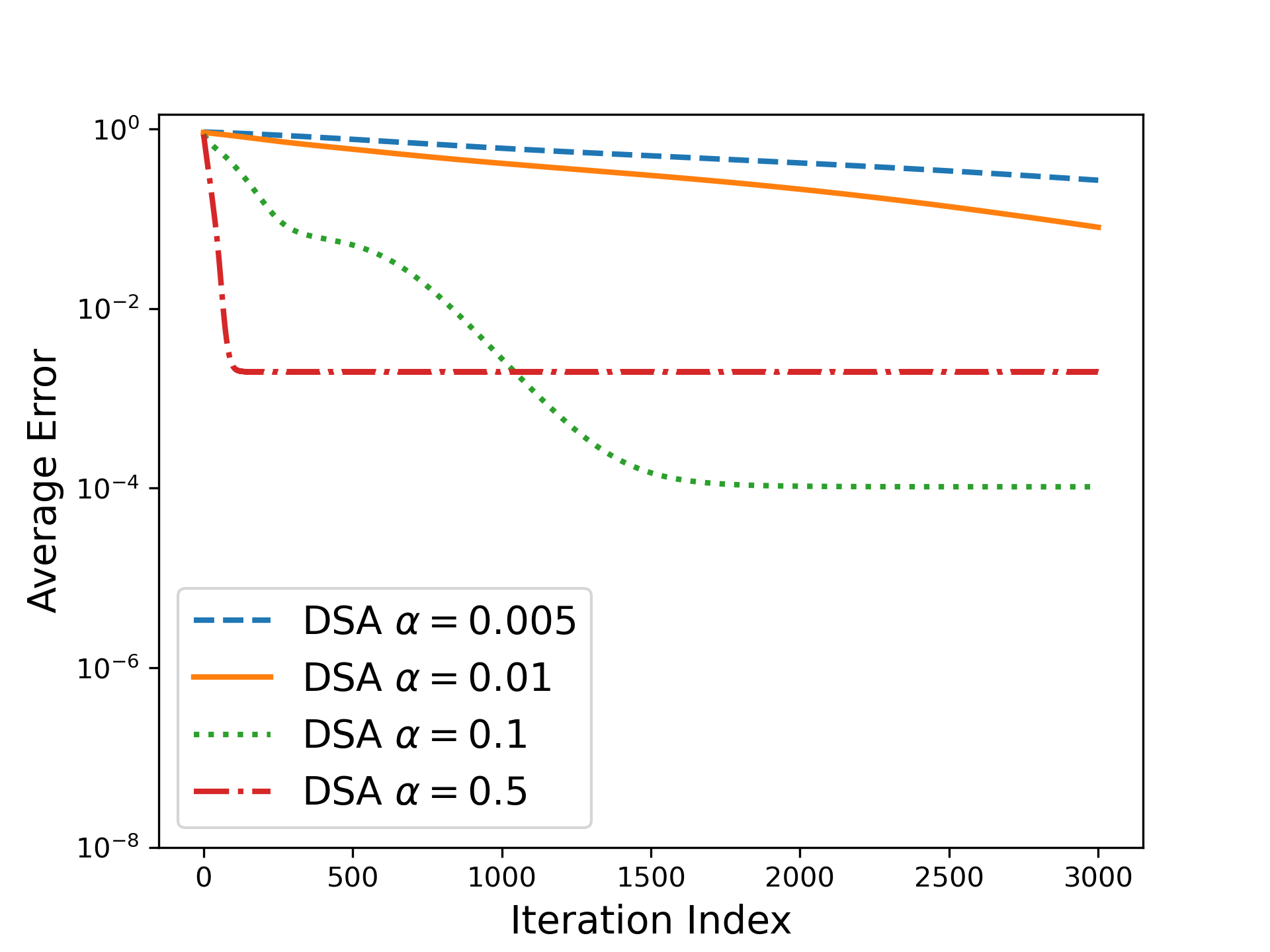}
    \caption{Effect of varying the step size $\alpha$ on the performance of DSA}
    \label{fig:effect_of_alpha}
    \end{subfigure}
    \caption{The role of collaboration in the distributed PCA problem and the effect of changing the step size on the performance of DSA. The distributed setup corresponds to an Erdos--Renyi graph ($p=0.5$) with $M=10$ nodes, while the dimension of data is $d=10$ and the number of estimated eigenvectors is $K=3$.}
	\label{fig:dsa}
\end{figure}
We first show results that emphasize on the need for collaboration among the nodes. To that end, we generate $N = 10,000$ independent and identically distributed (i.i.d.) samples drawn from a multivariate Gaussian distribution with an eigengap $\Delta_K = \frac{\lambda_{K+1}}{\lambda_K} = 0.8$ and dimension $d = 10$. These samples are distributed equally among the $M=10$ nodes of an Erdos--Renyi network (with connectivity probability $p = 0.5$), implying that each node has 1,000 samples. The number of eigenvectors estimated is $K=3$ and a constant step size of $\alpha = 0.1$ is used for this experiment. Figure~\ref{fig:compare_dsa_gha} shows the effect of using the GHA at a node without collaboration with other nodes versus DSA, which in simple terms embodies GHA + collaboration in the network. The blue line indicating GHA in the figure is the result of using all the data in a centralized manner. It is clear that the lack of any communication between nodes increases the error in estimation of the eigenvectors by a significant factor. In Figure~\ref{fig:effect_of_alpha}, we use the same setup and parameters to show the effect of different step sizes on our proposed DSA algorithm. It is evident that if the step size is too low, the convergence becomes significantly slow, while if its high, the final error is larger. Hence, careful choice of the step size is required for DSA, as characterized by its convergence analysis.

\begin{figure}[t]
	\centering
	\begin{subfigure}{.3\textwidth}
		\centering
		\includegraphics[width=\linewidth]{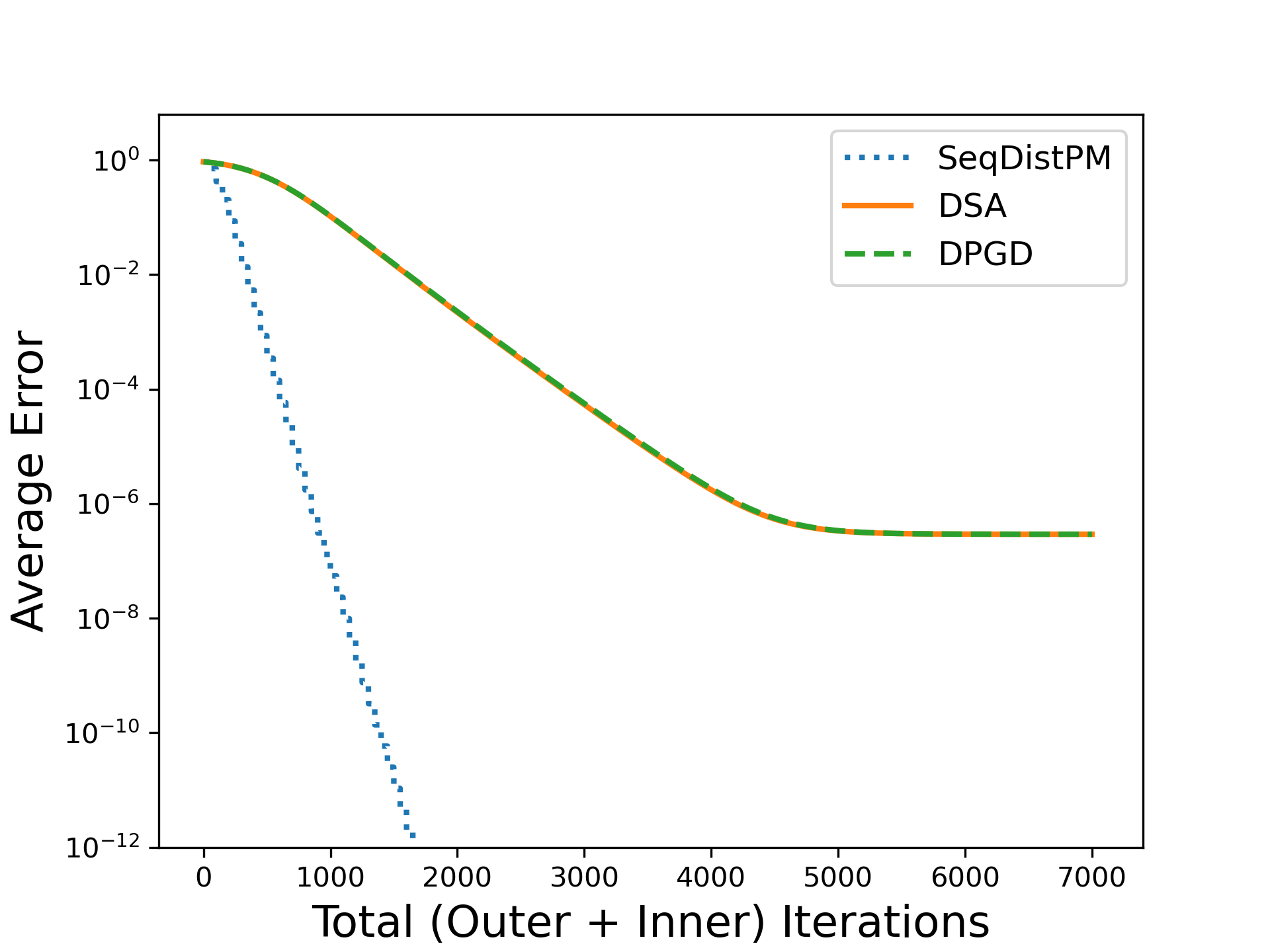}
		\caption{Erdos--Renyi network}
		\label{fig:a1}
	\end{subfigure}
	\hfil
	\begin{subfigure}{.3\textwidth}
		\centering
		\includegraphics[width=\linewidth]{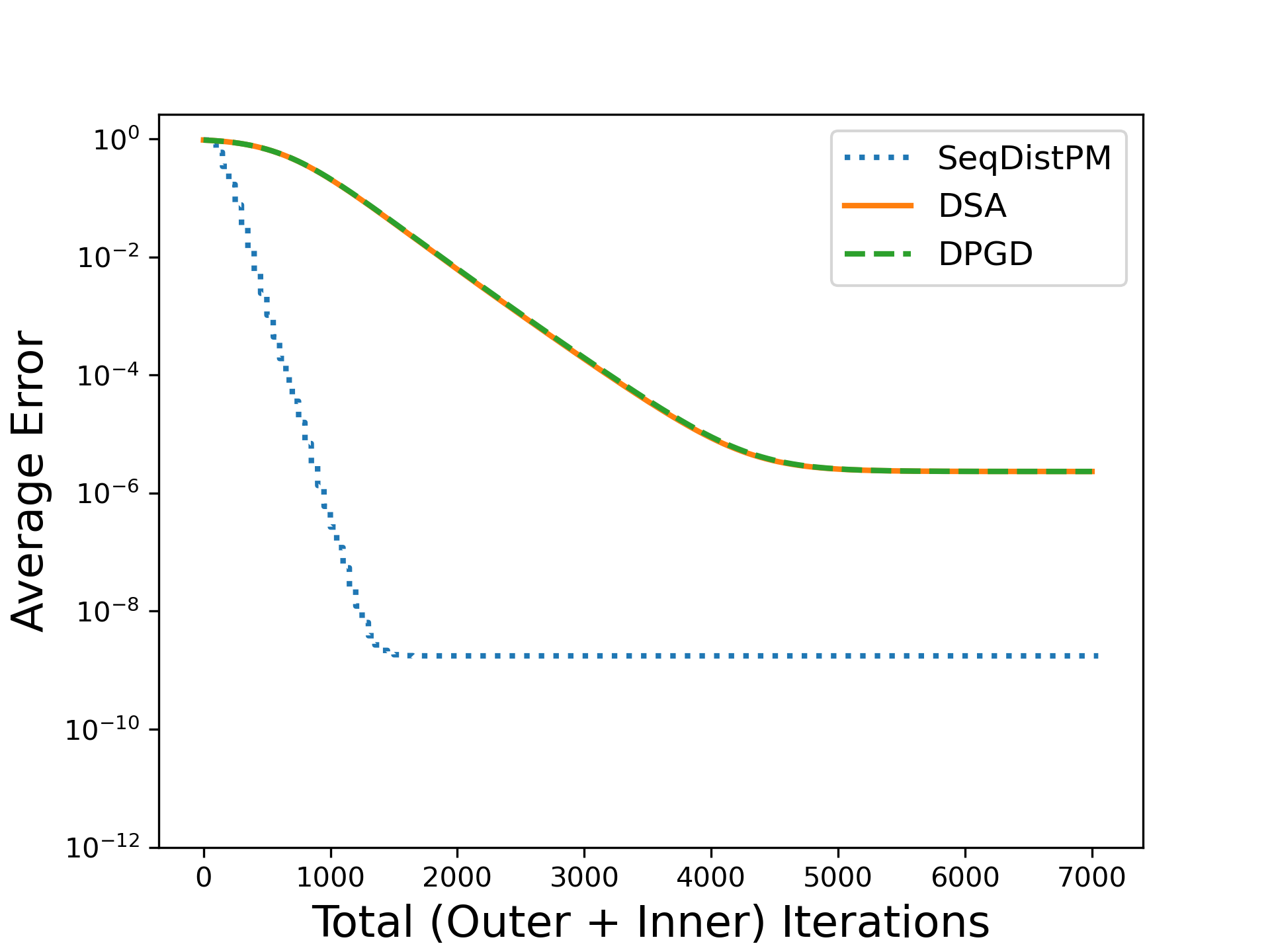}
		\caption{Cyclic network}
		\label{fig:b1}
	\end{subfigure}
	\hfil
	\begin{subfigure}{.3\textwidth}
		\centering
		\includegraphics[width=\linewidth]{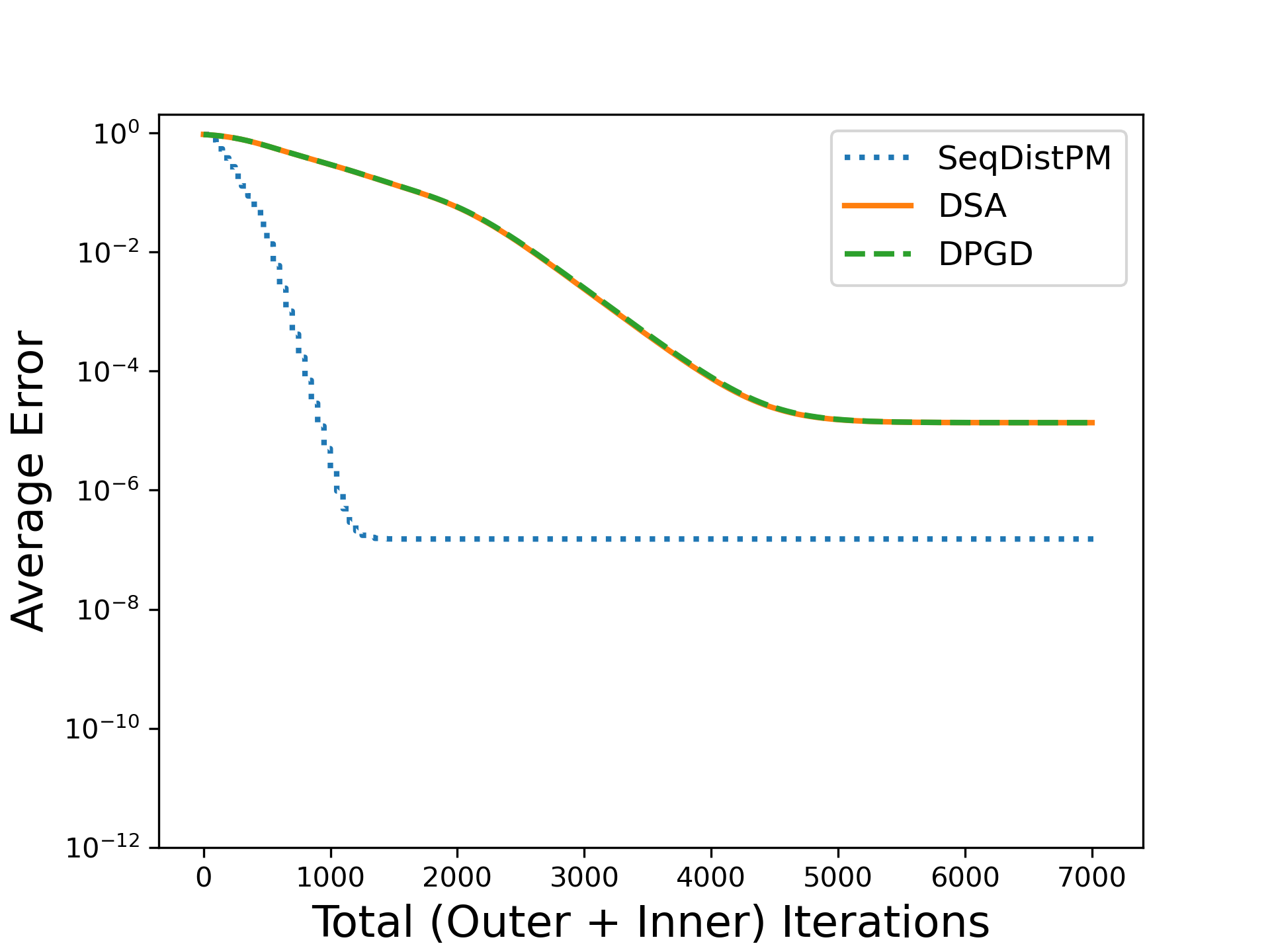}
		\caption{Star network}
		\label{fig:c1}
	\end{subfigure}
	\caption{Comparison between the performances of DSA, DPGD and SeqDistPM for $K=1$ and $\Delta_K = 0.8$ in terms of communications efficiency, i.e., decrease in average estimation error as a function of the number of data units communicated throughout the network.}
	\label{fig:k1}
\end{figure}
Next, we compare DSA with the distributed methods of DPGD and SeqDistPM to demonstrate its communication efficiency. For that purpose, we generate synthetic data with different eigengaps $\Delta_K \in \{0.6, 0.8\}$. We simulate the distributed setup for Erdos-Renyi ($p = 0.5$), star and cycle graph topologies with $M = 10$ nodes. The data is generated so that each node has 1,000 i.i.d samples ($N_i = 1000$) drawn from a multivariate Gaussian distribution for $d = 20$, i.e., the total samples generated are 10,000. The dimension of the subspace to be estimated is taken to be $K \in \{1, 5\}$. We use $T_c = 50$ as the number of consensus iterations per power iteration for SeqDistPM throughout out experiments. The results reported are an average of 10 Monte-Carlo trials. Figure~\ref{fig:k1} shows the performance of different algorithms for the estimation of the most dominant eigenvector for different network topologies. It is clear that for $K=1$ SeqDistPM outperforms both DSA and DPGD in terms of communications efficiency because it is basically distributed power method, which is shown in~\cite{cksvd,cksvd.allerton.2013} to have good performance for $K=1$. Even though DSA and DPGD have the same performance in terms of communications cost, it is important to remember that DPGD requires an additional QR normalization step per communications round. Next, Figure~\ref{fig:k5} shows a comparison between the three algorithms when the top-5 eigenvectors are estimated i.e., $K=5$. It is clear that while estimating higher-order eigenvectors, DSA slightly outperforms DPGD without performing explicit QR normalization and it also has much better communications efficiency than SeqDistPM. The error for SeqDistPM is significantly high in the beginning because of the sequential estimation, which means that when the first (higher-order) eigenvector(s) is (are) being estimated, the lower-order estimates are still at their initial values and hence those contribute significant error even when the first or higher order terms have low error. After a sufficiently large number of communications rounds, SeqDistPM eventually does reach a lower final error compared to DSA. But this comes at the expense of slower convergence as a function of communications costs. It should also be noted that SeqDistPM lacks a formal convergence analysis and has two time scales that need to be adjusted as both contribute to the final error. Finally, the benefits of DSA over DPGD are twofold. First, DSA reaches similar or better error floor without explicit QR normalization, thus saving $\cO(K^2d)$ computations per iteration; and second, the convergence guarantees for gradient descent-based algorithms for non-convex problems like the PCA have limitations. The guarantees usually exist for convergence to a stationary solution with a sub-linear rate.

\begin{figure}[t]
	\centering
	\begin{subfigure}{.3\textwidth}
		\centering
		\includegraphics[width=\linewidth]{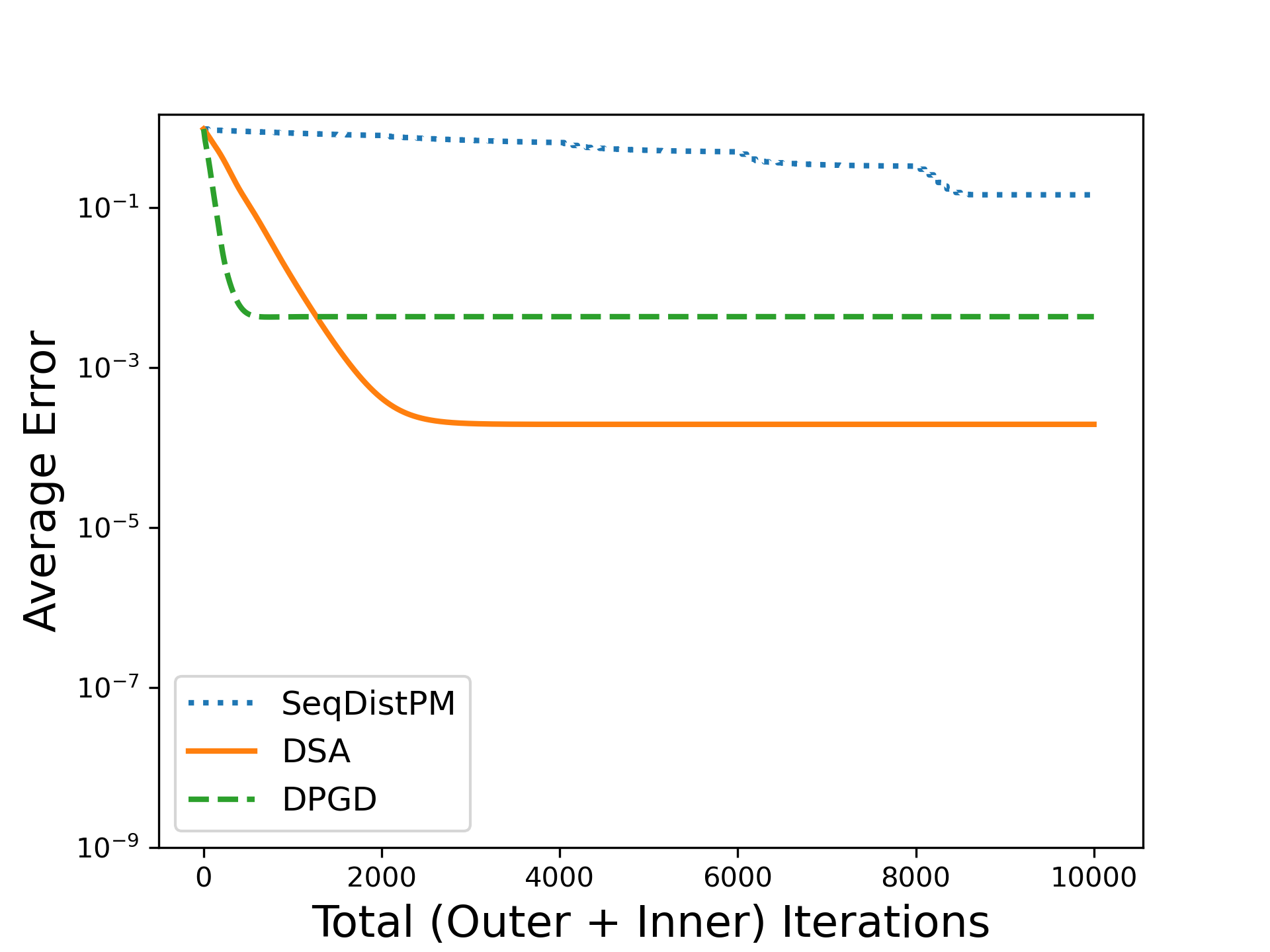}
		\caption{Erdos--Renyi network, $\Delta_K = 0.6$}
		\label{fig:a}
	\end{subfigure}
	\hfil
	\begin{subfigure}{.3\textwidth}
		\centering
		\includegraphics[width=\linewidth]{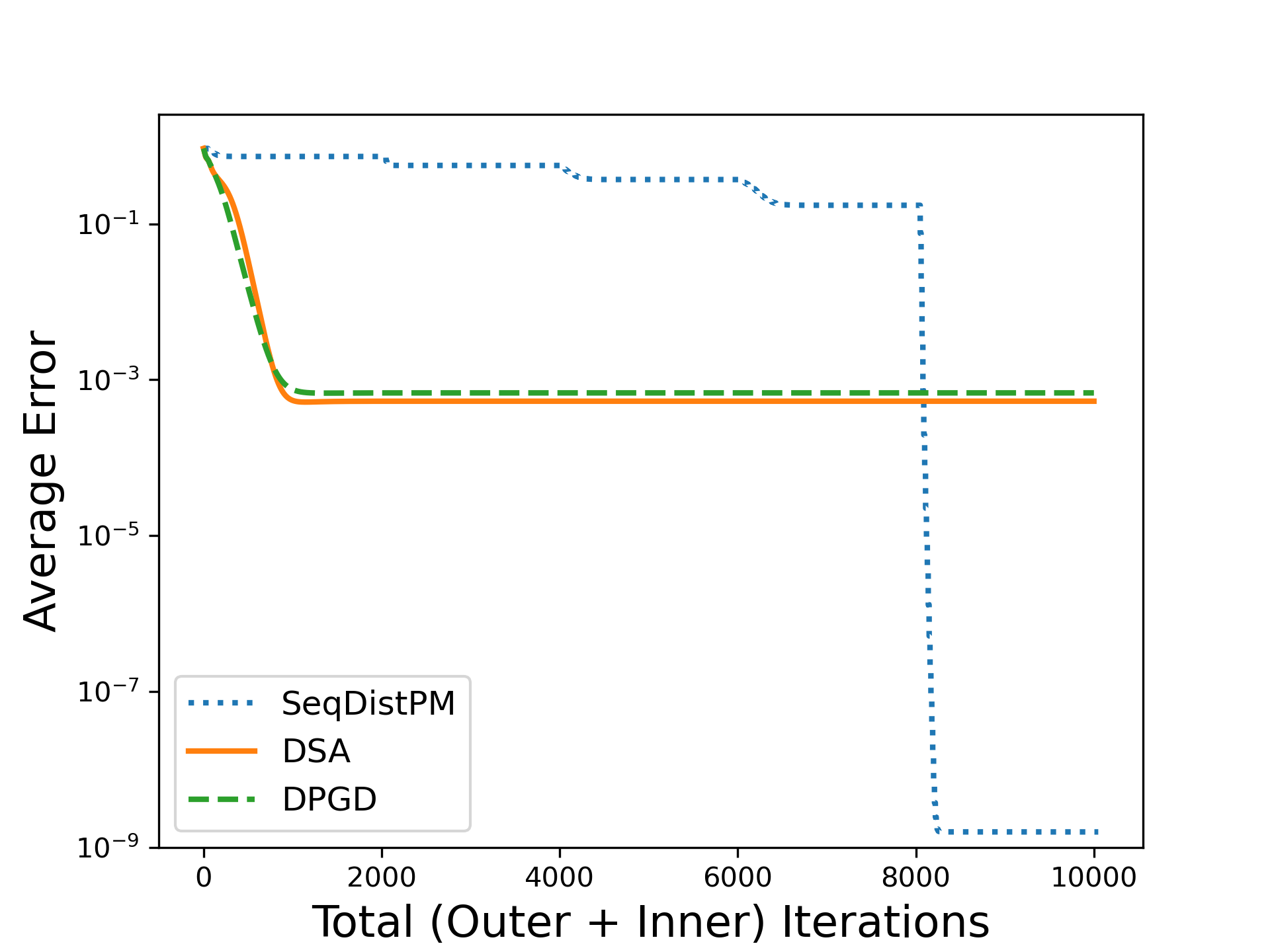}
		\caption{Cyclic network, $\Delta_K = 0.6$}
		\label{fig:b}
	\end{subfigure}
	\hfil
	\begin{subfigure}{.3\textwidth}
		\centering
		\includegraphics[width=\linewidth]{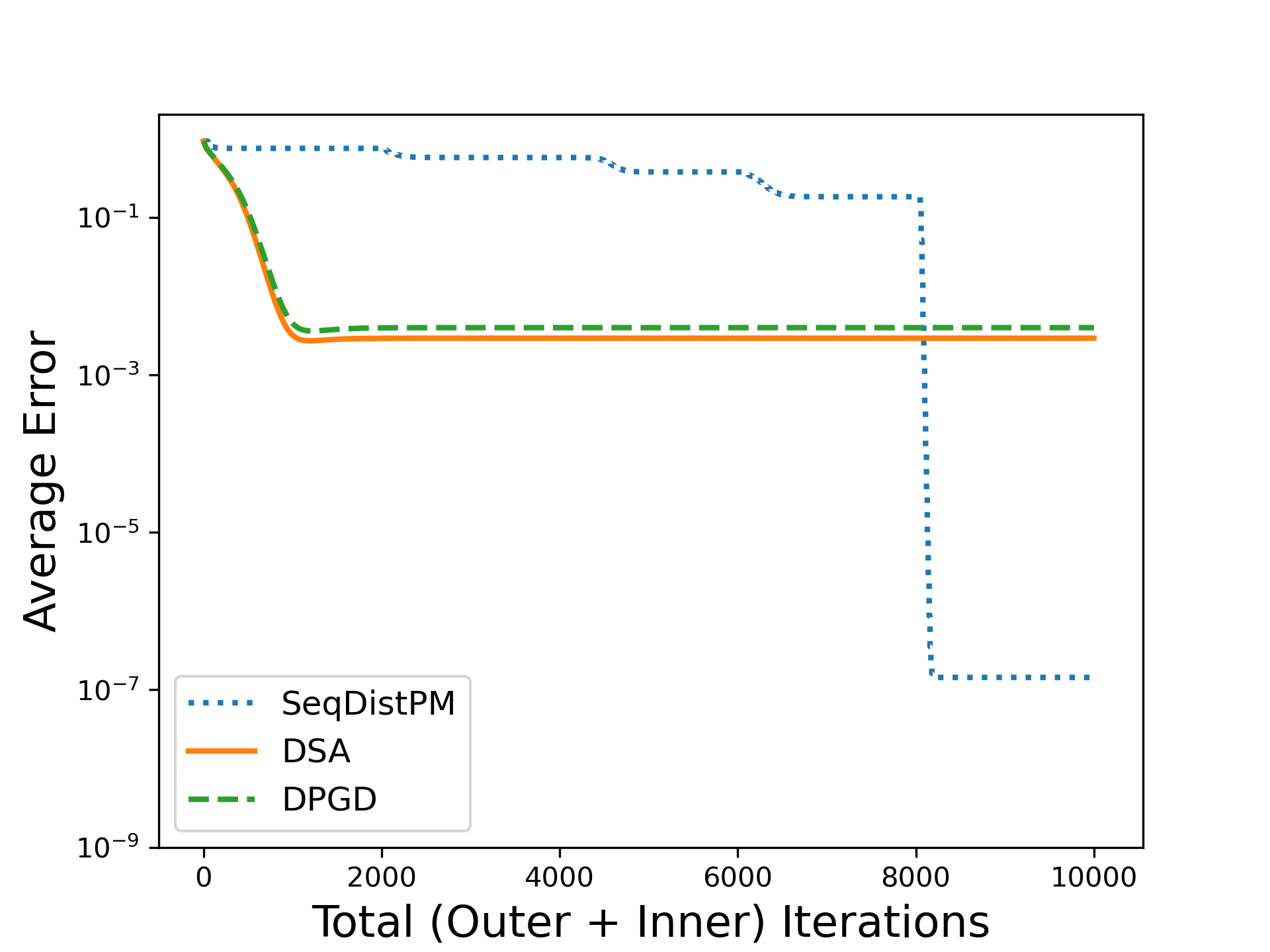}
		\caption{Star network, $\Delta_K = 0.6$}
		\label{fig:c}
	\end{subfigure}
	\hfil
	\begin{subfigure}{.3\textwidth}
		\centering
		\includegraphics[width=\linewidth]{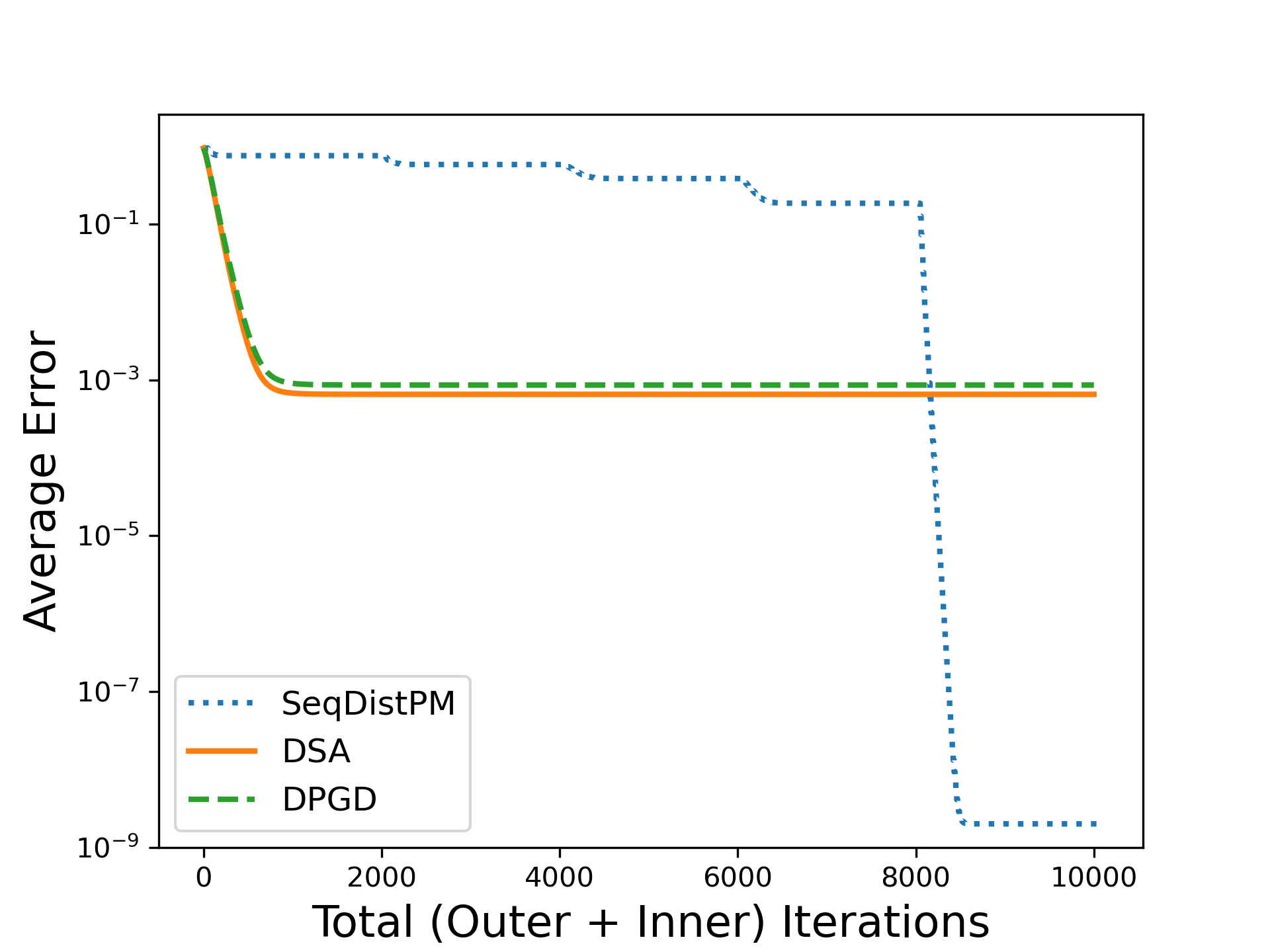}
		\caption{Erdos--Renyi network, $\Delta_K = 0.8$}
		\label{fig:d}
	\end{subfigure}
	\hfil
	\begin{subfigure}{.3\textwidth}
		\centering
		\includegraphics[width=\linewidth]{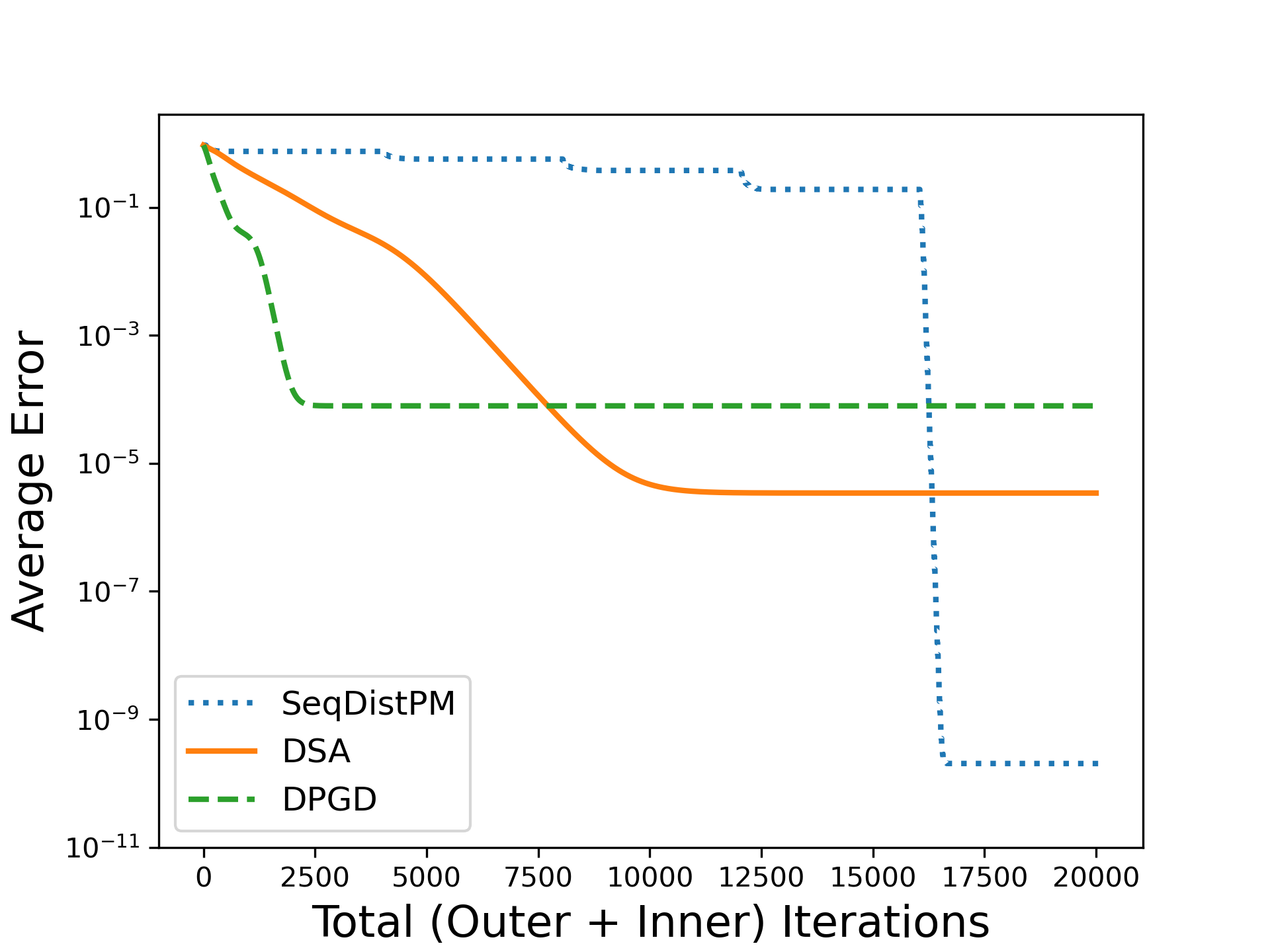}
		\caption{Cyclic network, $\Delta_K = 0.8$}
		\label{fig:e}
	\end{subfigure}
	\hfil
	\begin{subfigure}{.3\textwidth}
		\centering
		\includegraphics[width=\linewidth]{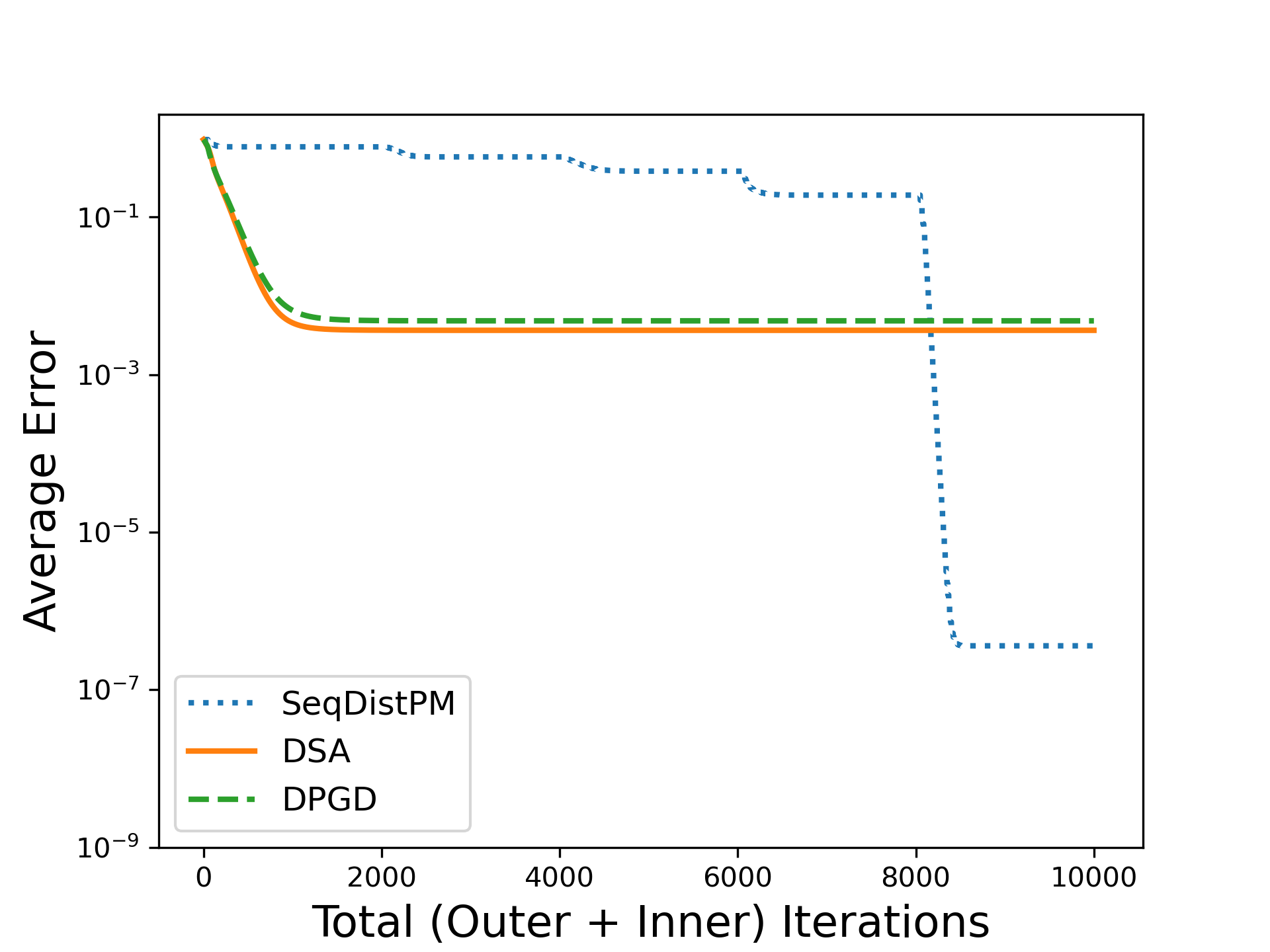}
		\caption{Star network, $\Delta_K = 0.8$}
		\label{fig:f}
	\end{subfigure}
	\caption{Comparison between DSA, DPGD, and SeqDistPM for $K=5$ in terms of communications efficiency.}
	\label{fig:k5}
\end{figure}

\subsection{Real-World Data}
Along with the synthetic data experiments, we provide some experiments with real-world datasets of MNIST~\cite{mnist.2010} and CIFAR-10~\cite{cifar.2009}. For the distributed setup in this case, we use an Erdos--Renyi graph with $M = 20$ nodes and $p=0.5$. Both the datasets have 60,000 samples, thereby making the number of samples per node to be $N_i = 3000$. The data dimension for MNIST is $d = 784$ and a constant step size of $\alpha = 0.1$ was used. The plots in Figure~\ref{fig:mnista} and Figure~\ref{fig:mnistb} show the results for $K \in \{10, 40\}$ for MNIST. Similar plots are shown for CIFAR-10 in Figure~\ref{fig:cifara} and Figure~\ref{fig:cifarb}, where the dimension $d$ for CIFAR-10 is 1024, the number of estimated eigenvectors $K \in \{10,20\}$ and a constant step size of $\alpha = 0.7$ is used. For these real-world data sets, we exclude the comparison with SeqDistPM as it is evident this method requires much higher cost of communications for estimating larger number of eigenvectors. 
\begin{figure}[t]
	\centering
	\begin{subfigure}{.45\textwidth}
		\centering
		\includegraphics[width=\linewidth]{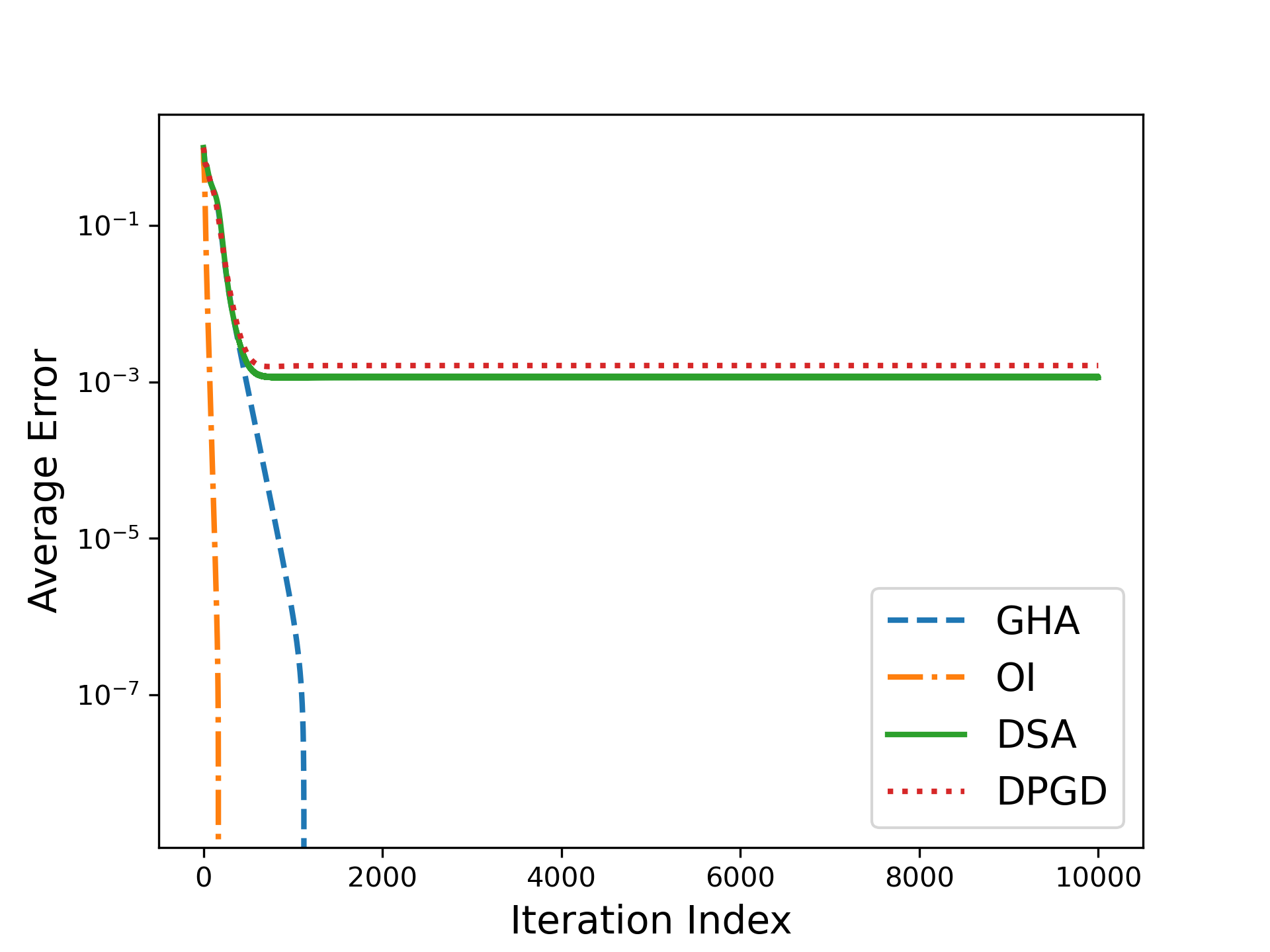}
		\caption{MNIST, $K = 10$}
		\label{fig:mnista}
	\end{subfigure}
	\begin{subfigure}{.45\textwidth}
		\centering
		\includegraphics[width=\linewidth]{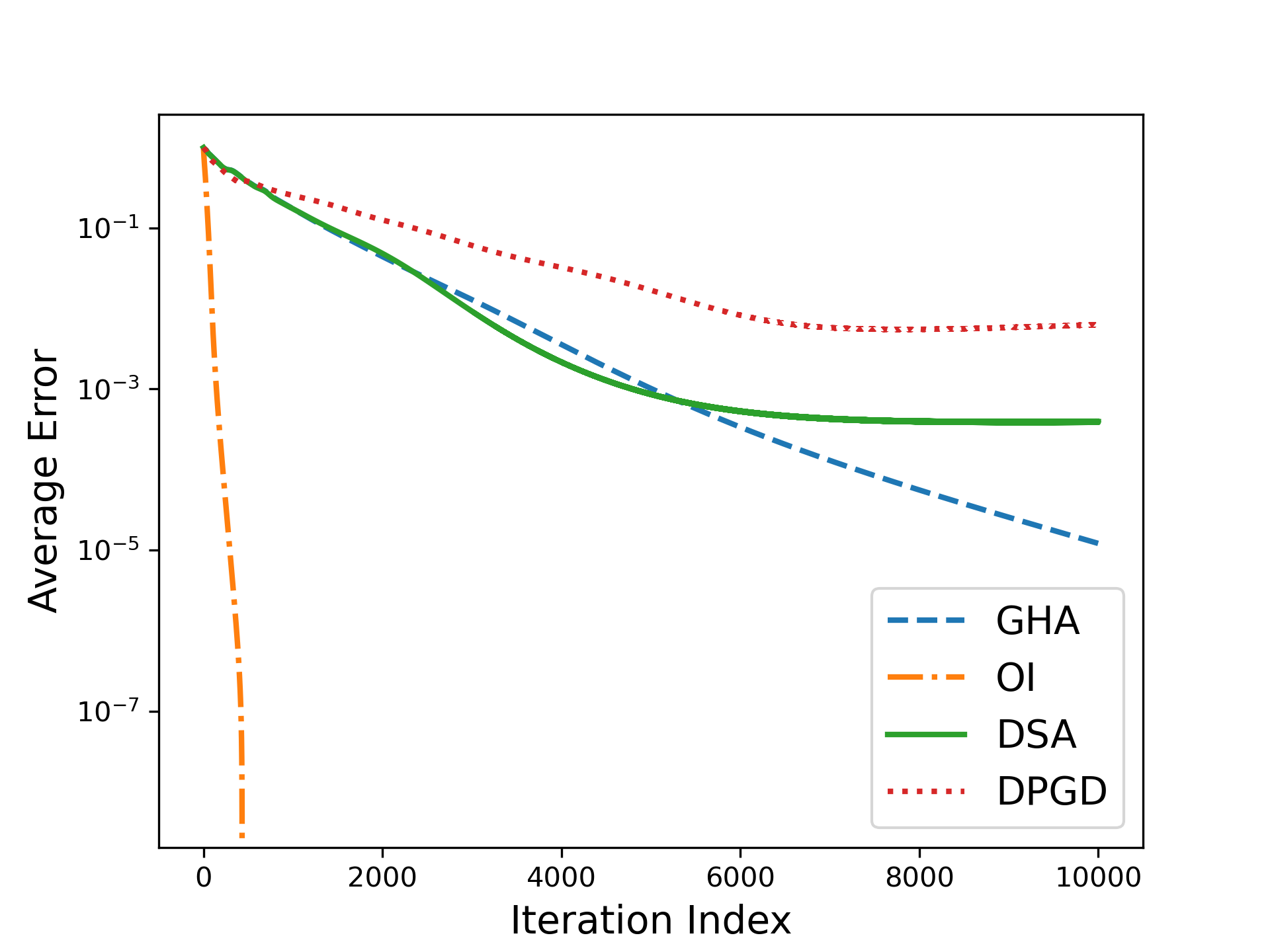}
		\caption{MNIST, $K = 40$}
		\label{fig:mnistb}
	\end{subfigure}
	\caption{Comparison between DSA, OI, GHA, and DPGD for MNIST dataset as a function of the number of algorithmic iterations.}
	\label{fig:mnist}
\end{figure}
	
\begin{figure}[t]
	\centering
	\begin{subfigure}{.45\textwidth}
		\centering
		\includegraphics[width=\linewidth]{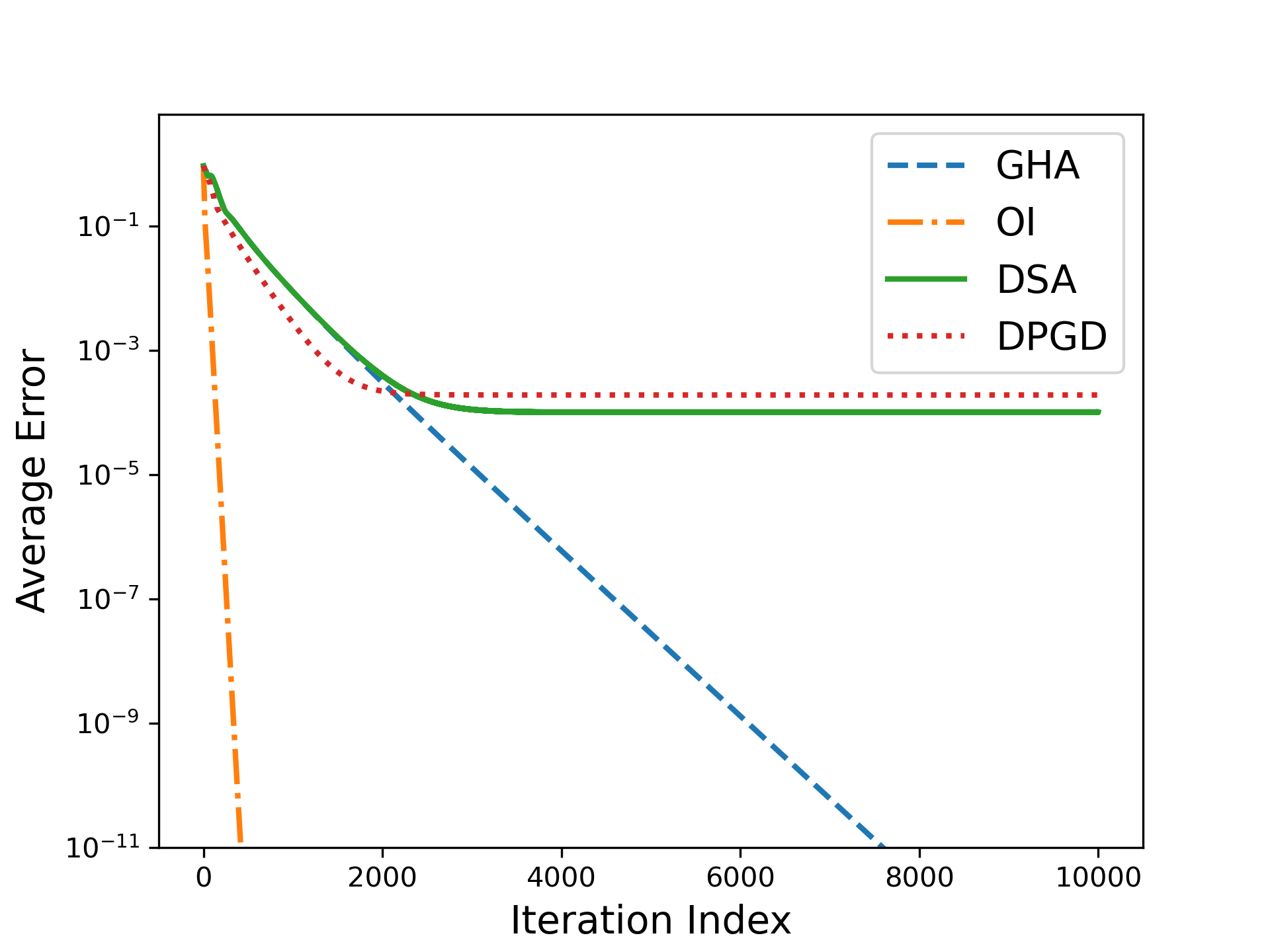}
		\caption{CIFAR-10, $K = 10$}
		\label{fig:cifara}
	\end{subfigure}
	\begin{subfigure}{.45\textwidth}
		\centering
		\includegraphics[width=\linewidth]{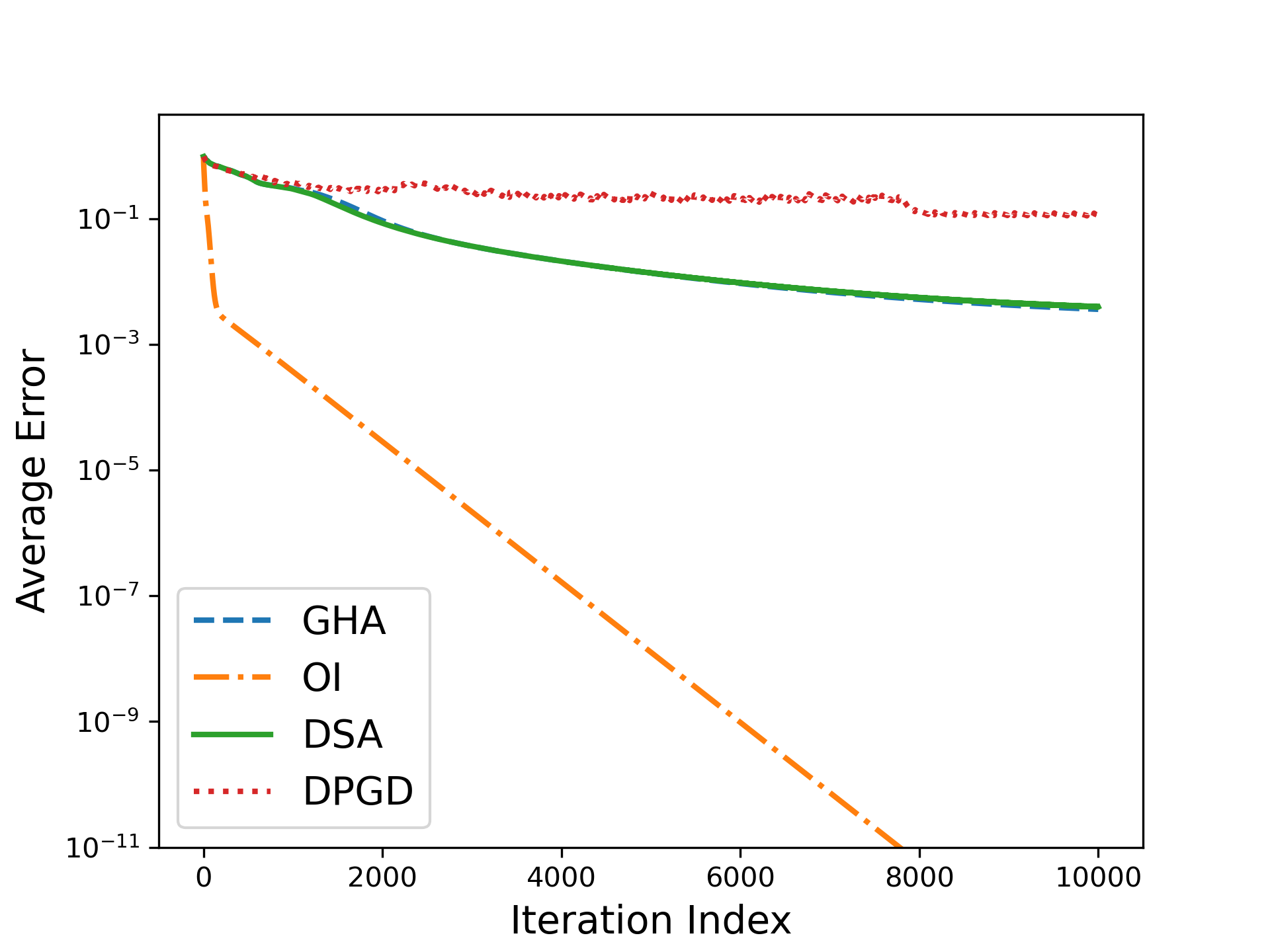}
		\caption{CIFAR-10, $K = 20$}
		\label{fig:cifarb}
	\end{subfigure}
	\caption{Comparison between DSA, OI, GHA, and DPGD for CIFAR-10 dataset as a function of the number of algorithmic iterations.}
	\label{fig:cifar}
\end{figure}